\documentclass{article}
\pdfoutput=1

\usepackage[preprint]{neurips_2024}

\usepackage[utf8]{inputenc} 
\usepackage[T1]{fontenc}    
\usepackage{hyperref}       
\usepackage{url}            
\usepackage{booktabs}       
\usepackage{amsfonts}       
\usepackage{nicefrac}       
\usepackage{microtype}      
\usepackage{xcolor}         

\usepackage{graphicx} 
\usepackage{caption}
\usepackage{bbding}
\usepackage{pifont}     
\usepackage{fontawesome}    
\usepackage{amsmath}
\usepackage{amsthm}
\usepackage{dutchcal}
\newtheorem{theorem}{\bf Theorem}[section]

\newtheorem{myDef}{\bf Definition}[section]
\newtheorem{lemma}{Lemma}[section]
\newtheorem{assumption}{Assumption}[section]
\newtheorem{fact}{Fact}[section]
\newtheorem{remark}{Remark}[section]
\usepackage{subcaption}
\usepackage{enumitem}
\usepackage{multirow}
\usepackage{algorithm}
\usepackage{algorithmic}
\usepackage{array}

\newcommand{\ourtech}{DC-DPSGD}
\newcommand{\ignore}[1]{}

\hypersetup{
  colorlinks   = true, 
  urlcolor     = black, 
  linkcolor    = blue, 
  citecolor   = blue 
}

\title{Clip Body and Tail Separately: High Probability Guarantees for DPSGD with Heavy Tails}

\author{%
  Haichao Sha \\
  Renmin University of China\\
  \texttt{sha@ruc.edu.cn} \\
  \And
  Yang Cao \\
  Tokyo Institute of Technology \\
  \texttt{cao@c.titech.ac.jp} \\
  \AND
  Yong Liu \\
  Renmin University of China \\
  \texttt{liuyonggsai@ruc.edu.cn} \\
  \And
  Yuncheng Wu \\
  Renmin University of China \\
  \texttt{wuyuncheng@ruc.edu.cn} \\
  \And
  Ruixuan Liu \\
  Emory University \\
  \texttt{ruixuan.liu2@emory.edu} \\
  \And
  Hong Chen$^{*}$ \\
  Renmin University of China \\
  \texttt{chong@ruc.edu.cn} \\
}

\begin{document}

\maketitle

\begin{abstract}
Differentially Private Stochastic Gradient Descent (DPSGD) is widely utilized to preserve training data privacy in deep learning, which first clips the gradients to a predefined norm and then injects calibrated noise into the training procedure. Existing DPSGD works typically assume the gradients follow sub-Gaussian distributions and design various clipping mechanisms to optimize training performance. However, recent studies have shown that the gradients in deep learning exhibit a heavy-tail phenomenon, that is, the tails of the gradient have infinite variance, which may lead to excessive clipping loss to the gradients with existing DPSGD mechanisms. To address this problem, we propose a novel approach, Discriminative Clipping~(DC)-DPSGD, with two key designs. First, we introduce a subspace identification technique to distinguish between body and tail gradients. Second, we present a discriminative clipping mechanism that applies different clipping thresholds for body and tail gradients to reduce the clipping loss. Under the non-convex condition, \ourtech{} reduces the empirical gradient norm from {\small ${\mathbb{O}\left(\log^{\max(0,\theta-1)}(T/\delta)\log^{2\theta}(\sqrt{T})\right)}$} to {\small ${\mathbb{O}\left(\log(\sqrt{T})\right)}$} with heavy-tailed index $\theta\geq 1/2$, iterations $T$, and arbitrary probability $\delta$. Extensive experiments on four real-world datasets demonstrate that our approach outperforms three baselines by up to 9.72\% in terms of accuracy. 
\end{abstract}


\section{Introduction}
DPSGD~\cite{abadi2016deep}, as a mainstream paradigm of privacy-preserving deep learning, has wide applications in areas such as privacy-preserving recommender systems~\cite{liu2023privaterec,yi2021efficient}, 
face recognition~\cite{ghalebikesabi2023differentially,harder2022pre,tang2024differentially}, and medical diagnosis~\cite{adnan2022federated,ji2022privacy,meng2021improving,ziller2021medical}. 
Essentially, in each iteration of model training, DPSGD clips per-sample gradient under the $L_2$-norm constraint to obtain the maximum divergence between gradient distributions that differ by only one training data point and adds random noise within rigorous privacy bounds to the gradient
for unbiased gradient estimation. 

Most of existing DPSGD works~\cite{bu2024automatic,yang2022normalized,xia2023differentially,fang2022improved,zhang2023differentially,sha2023pcdp,zhu2023improving,koloskova2023revisiting} rely on the assumption that the gradient noise follows a sub-Gaussian distribution to devise effective clipping strategies. 
However, recent studies~\cite{zhang2020adaptive,gurbuzbalaban2021heavy,simsekli2019tail,zhu2018anisotropic,simsekli2020fractional,camuto2021asymmetric,panigrahi2019non,barsbey2021heavy} have shown that SGD gradient noise in deep learning often exhibit heavy-tailed distributions instead of light-tailed distributions (e.g., sub-Gaussian). This occurs even when the dataset originates from a light-tailed distribution, the gradients still diverge to a heavy-tailed distribution with infinite variance~\cite{gurbuzbalaban2021heavy}, which may slow down the convergence rate and impair training performance~\cite{li2022high,li2023high,madden2020high,gorbunov2020stochastic,davis2020high,li2020high}.
To cope with heavy-tailed dilemma in SGD, \cite{wang2021convergence,li2023high,gorbunov2020stochastic} suggest employing larger clipping thresholds to get rid of the oscillations caused by heavy-tailed gradients on the training trajectory. 
%
Nevertheless, the clipping operation in DPSGD is closely tied to the magnitude of DP noise added to the gradients. Setting the clipping threshold too large can lead to a high-dimensional noise catastrophe~\cite{zhou2020bypassing}, which negatively impacts model performance and potentially disrupts the convergence of DPSGD algorithms. 
Therefore, practitioners need to carefully strike a balance between injected noise and clipping loss, as illustrated in Figure~\ref{fig:1}. 
The left sub-figure shows the trade-off under the light-tailed assumption. As the clipping threshold increases (i.e., when the red dotted line moves to the right), the clipping loss decreases, but the maximum divergence between neighboring distributions increases, leading to more DP noise being added. In the right sub-figure, under the same clipping threshold,
the slower decay rate of the heavy tail distribution (blue line) introduces extra clipping loss, while it simultaneously reduces the maximum divergence compared to the light-tailed distribution. Consequently, the required DP noise magnitude is lower. 
Therefore, we aim to investigate the following key question in this paper: \textit{how to design an effective clipping mechanism under the heavy-tailed assumption to balance the trade-off between clipping loss and DP noise magnitude in DPSGD?}
\begin{figure}[t]
\vspace{-0.3cm}
	\centering
	\begin{subfigure}{0.95\linewidth}
		\centering
            \includegraphics[width=0.75\linewidth]{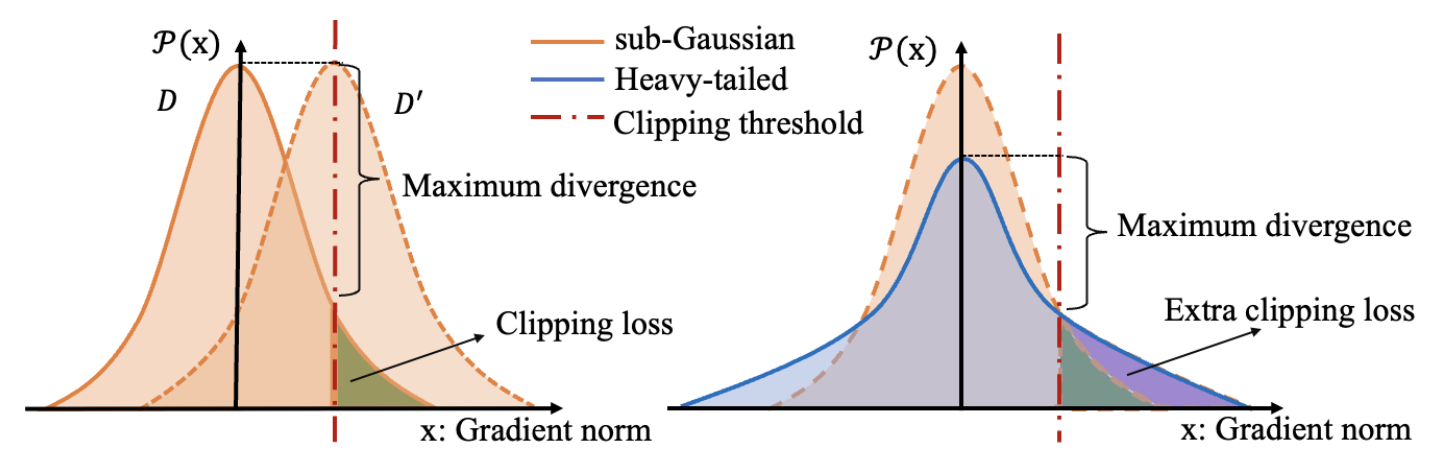}
	\end{subfigure}
        \caption{The trade-off between clipping loss and noise magnitude under heavy-tailed distributions.}
        \label{fig:1}
\vspace{-0.5cm}
\end{figure}
Although a set of DPSGD clipping mechanisms~\cite{bu2024automatic,yang2022normalized,xia2023differentially} have been proposed under the light-tailed assumption, none of them can be adapted to our problem. 
Specifically, \cite{bu2024automatic,yang2022normalized,xia2023differentially} focus on small-norm gradients (i.e., those near the center of the distribution) and normalize them to be around 1. These approaches reduce the maximum divergence, thereby requiring less noise to be injected. However, they do not account for heavy-tailed gradients and thus cannot optimize the clipping loss.
%
Another line of work directly estimates the actual norm of the per-sample gradient and utilizes it as the clipping threshold to reduce the clipping loss. For instance, Andrew et al.~\cite{andrew2021differentially} estimate the true gradient trajectory by collecting the norms of historical gradients. However, this approach requires knowing the upper bound of historical norms for adding noise, which is highly uneconomical under heavy-tailed distributions, as the upper bound for moment generating function~(MGF)~\cite{vladimirova2020sub} can be immeasurable, making the scale of DP noise unbearable.  


In this paper, we propose a novel approach, named \textbf{Discriminative Clipping (DC)-DPSGD}, to effectively balance the trade-off between clipping loss and required DP noise under the heavy-tailed assumption. 
The key idea is to utilize different clipping thresholds for the body gradients and tail gradients respectively, retaining more information from tail gradients that can withstand more severe DP noise.
We introduce two techniques in \ourtech{} to achieve this goal. 
First, we design a subspace identification technique to identify potential heavy-tailed gradients with high probability guarantees. 
We note that the body of heavy-tailed distributions exhibits characteristics similar to those of light-tailed distributions, and the main difference lies in the decay rate at the tails. 
Therefore, we extract orthogonal random vectors from heavy-tailed distributions (e.g., sub-Weibull distribution) to construct a random projection subspace, and compute the trace of the second-moment matrix between gradients and this subspace to distinguish heavy-tailed gradients. 
Second, we present a discriminative clipping mechanism, which applies a large clipping threshold for the identified heavy-tailed gradients and a smaller one for the remaining light-tailed gradients. 
We theoretically analyze the choice of the two clipping thresholds and the convergence of \ourtech{} with a tighter bound under the high probability theory.
Our contributions can be summarized as follows.
\begin{itemize}[leftmargin=*]
    \item We propose \ourtech{} with a subspace identification technique and a discriminative clipping mechanism to optimize DPSGD under the heavy-tailed assumption. To our knowledge, this is the first work to rigorously address heavy tails in DPSGD with a high probability theory guarantee. 
    \item We theoretically analyze the convergence of \ourtech{} and show that \ourtech{} reduces the empirical gradient norm from {\small $\mathbb{O}\left(\log^{\max(0,\theta-1)}(T/\delta)\log^{2\theta}(\sqrt{T})\right)$} to {\small $\mathbb{O}\left(\log(\sqrt{T})\right)$} with heavy-tailed index $\theta\geq 1/2$, iterations $T$, and arbitrary probability $\delta$, under the non-convex condition.
    \item We conduct extensive experiments on four real-world datasets, where \ourtech{} consistently outperforms four baselines with up to 9.72\% accuracy improvements, demonstrating the effectiveness of our proposed approach. 
\end{itemize}

\section{Related Work}
\textbf{Heavy-tailed noise and high probability bounds.}
Recently, from the perspective of escaping from stationary points and Langevin dynamics, the noise in neural networks is more inclined to anisotropic and non-Gaussian properties~\cite{gurbuzbalaban2021heavy,simsekli2019tail,zhu2018anisotropic,panigrahi2019non,gorbunov2020stochastic,zhang2020adaptive}, with specific heavy-tailed phenomena discovered and defined in gradient descent in deep neural networks. 
Current research has primarily focused on heavy-tailed convex optimization in privacy-preserving deep learning~\cite{lowy2023private,wang2020differentially,kamath2022improved}. 
Building upon~\cite{wang2020differentially}, Kamath et al.~\cite{kamath2022improved} relax the assumption of Lipschitz condition and sub-Exponential distribution to a more general $\alpha$-th moment bounded condition. However, no work has been done investigating the convergence characteristics of heavy-tailed DPSGD under non-convex settings. Meanwhile, high probability bounds are more frequently discussed in optimization properties such as convex and non-convex learning with SGD. Specifically, \cite{li2022high} considers gradient noise from heavy-tailed sub-Weibull distribution to present high probability bounds at fast rates, revealing trade-offs between optimization and generalization performance under broader assumptions. With bounded $a$-th moments assumption, \cite{li2023high} provides a high probability theoretical analysis for variants like clipped SGD with momentum and adaptive step sizes. Nevertheless, most work in DPSGD still utilizes expectation bounds, which is not suitable for heavy-tailed assumptions. 

\textbf{Projection subspace in DPSGD.}
DPSGD has gained wide concerns for its detrimental impact on model accuracy. A series of works leverage projection techniques to improve performance. 
For instance, \cite{zhou2020bypassing,yu2021not,liang2020think,song2022sketching,yu2021large} confine DPSGD training dynamics to more compact and condensed subspaces through projection. While ensuring the fidelity of training data compression, they decouple the irrelevant relationship between ambient features and DP noise, and reduce the optimization error of DPSGD under stringent privacy constraints. However, existing works rely on the assumption that public datasets are available for designing the techniques ~\cite{golatkar2022mixed,zhou2020bypassing,yu2021not,gu2023choosing}, which is a rather strong, especially in sensitive domains. 
In contrast, our work does not rely on any public dataset. 

\textbf{Gradient clipping.}
Gradient clipping has attracted significant attention in both practical implementations and theoretical analyses for DPSGD~\cite{chen2020understanding,zhang2019gradient,zhang2022understanding,pichapati2019adaclip,andrew2021differentially,xiao2023theory,koloskova2023revisiting}. Since the tuning parameters in the classical Abadi's clipping function~\cite{abadi2016deep} are complex, adaptive gradient clipping schemes have been proposed~\cite{bu2024automatic,yang2022normalized}. These schemes scale per-sample gradients based on their norms. In particular, gradients with small norms are amplified infinitely. 
Building upon this, Xia et al.~\cite{xia2023differentially} 
control the amplification of gradients with small norms in a finite manner. Additionally, research on clipping bias has gradually gained importance. Wei et al.~\cite{wei2022dpis} and Koloskova et al.~\cite{koloskova2023revisiting} argue for the connection between sampling noise and clipping bias and mitigate clipping bias through group sampling.
Sha et al.~\cite{sha2023pcdp} study pre-projecting per-sample gradient before clipping to reduce clipping errors in DPSGD. Furthermore, \cite{fang2022improved} has shown that naive gradient clipping can accelerate vanilla SGD convergence under heavy-tailed distributions. 
However, no work has specifically optimized gradient clipping under the heavy-tailed assumption of DPSGD. Due to the scale of noise required to achieve differential privacy, trivial clipping methods and analyses are not applicable.

\section{Preliminaries}
\subsection{Notations}
Let $D$ be a private dataset, which consists of $n$ training data $S=\{z_1,...,z_n\}$ with a sample domain $Z$ drawn i.i.d. from the underlying distribution $\mathcal{P}$. 
Since $\mathcal{P}$ is unknown and inaccessible in practice, we minimize the following empirical risk in a differentially private manner: 
\begin{align}
    L_S(\mathbf{w}) := \frac{1}{n}\sum^{n}_{i=1}\ell(\mathbf{w},z_i),
\end{align}
where the objective function $\ell(\cdot):(\mathbf{w}\subseteq W,Z)\rightarrow \mathbb{R}$ is possible non-convex and $W\subseteq\mathbb{R}^d$ represents the model parameter space. Then, we denote $\nabla \ell$ as the gradient of $\ell$ with respect to
$\mathbf{w}$. 
In addition, we introduce some notations regarding the projection subspace. Let $V_k\in \mathbb{R}^{d\times k}$ denotes $k$-dimensional random projection sampled from heavy-tailed distributions. The empirical second moment of $V^T_k\nabla \ell$ is given by $V^T_k\nabla \ell\nabla \ell^T V_k$. The total variance in the empirical projection subspace is generally measured by the trace of the second moment denoted as $\text{tr}(V^T_k\nabla \ell\nabla \ell^T V_k)$.

DPSGD lies in strict mathematical definitions~\cite{dwork2006calibrating,abadi2016deep} and composition theorems~\cite{kairouz2015composition,mironov2017renyi,dong2019gaussian}. Definition~\ref{myDef:dp} gives a formal definition of differential privacy (DP).
\begin{myDef}[\textbf{Differential Privacy}]
A randomized algorithm $M$ is $(\epsilon,\delta)$-differentially private if for any two neighboring datasets $D$, $D^{\prime}$ differ in exactly one data point and any event $Y$, we have
\begin{align}
    \mathbb{P}(M(D) \in Y) \leq \exp({\epsilon}) \cdot \mathbb{P}(M(D^{\prime}) \in Y) + \delta,
\end{align}
where $\epsilon$ is the privacy budget and $\delta$ is a small probability.
\label{myDef:dp}
\end{myDef}

\subsection{Assumptions}
We focus on the sub-Weibull distribution in this work, which extends the sub-Gaussian and sub-Exponential families to potentially heavier-tailed distributions. Sub-Weibull distributions are characterized by a positive tail index $\theta$, with $\theta=\frac{1}{2}$ represents sub-Gaussian distributions, $\theta=1$ represents sub-Exponential distributions, and $\theta>1$ represents heavier-tailed distributions. 
Typically, sub-Gaussian distributions are light-tailed, whereas heavy-tailed distributions occur when $\theta>\frac{1}{2}$.
\begin{assumption}[\textbf{Sub-Weibull Gradient Noise}]
\label{ass:sub-weibull}
Conditioned on the iterates, we make an assumption that the gradient noise $\nabla \ell(\mathbf{w}_t)-\nabla L(\mathbf{w}_t)$ satisfies $\mathbb{E}[\nabla \ell(\mathbf{w}_t)-\nabla L(\mathbf{w}_t)] = 0$ and $\Vert \nabla \ell(\mathbf{w}_t)-\nabla L(\mathbf{w}_t)\Vert_2\sim \text{subWeibull}(\theta,K)$ for some positive K, such that $\theta>\frac{1}{2}$, and have 
\vspace{-0.1cm}
\begin{align}
    \mathbb{E}_t[\exp((\Vert \nabla \ell(\mathbf{w}_t)-\nabla L(\mathbf{w}_t)\Vert_2/K)^{\frac{1}{\theta}})] \leq 2 . \nonumber
\end{align}
\end{assumption}
Assumption~\ref{ass:sub-weibull} is a relaxed version of gradient noise following sub-Gaussian distributions, that is $\mathbb{E}_t[\exp((\Vert \nabla \ell(\mathbf{w}_t)-\nabla L(\mathbf{w}_t)\Vert_2/K)^{2})] \leq 2$, which means that finding upper bounds for MGF under Assumption~\ref{ass:sub-weibull} is impracticable by standard tools~\cite{vladimirova2020sub}. Thus, the truncated tail theory~\cite{bakhshizadeh2023sharp} and martingale difference inequality~\cite{madden2020high} play a crucial role in our analysis.
\begin{assumption}[\textbf{$\beta$-Smoothness}]
 \label{ass:smooth}
    The loss function $\ell$ is $\beta$-smooth, for any $\mathbf{w}_t, \mathbf{w}^{\prime}_t\in \mathbb{R}^d$, we have
 \begin{align}
     \Vert \nabla\ell(\mathbf{w}_t)-\nabla\ell(\mathbf{w}^{\prime}_t)\Vert_2 \leq \beta \Vert \mathbf{w}_t-\mathbf{w}^{\prime}_t \Vert_2. \nonumber
 \end{align}
\end{assumption}
\begin{assumption}[\textbf{G-Bounded}]
\label{ass:G-bounded}
For any $\mathbf{w}\in \mathbb{R}^d$ and per-sample $z$, there exists positive real numbers $G>0$, and the expectation gradient satisfies
\begin{align}
    \Vert \nabla L(\mathbf{w}_t) \Vert^{2}_{2} \leq G. \nonumber
\end{align}
\end{assumption}
Assumption~\ref{ass:smooth} is widely employed in optimization literature~\cite{foster2018uniform,zhou2020bypassing,li2022high} and is essential for ensuring the convergence of gradients to zero~\cite{li2020high}. Compared to the bounded stochastic gradient assumption, i.e., $\Vert \nabla \ell(\mathbf{w}_t,z_i) \Vert^{2}_{2} \leq G$, Assumption~\ref{ass:G-bounded} is mild~\cite{zhou2020bypassing,li2022high,li2023high}. 

\section{Heavy-tailed DPSGD with High Probability Bounds}
Before presenting our approach, we first analyze the high probability bound of classical DPSGD under the heavy-tailed assumption to better motivate our idea. 
We note that previous works rely on the assumption of light-tailed gradients or stronger assumptions
to prove the convergence properties of DPSGD, which cannot be adapted to DPSGD under heavy-tailed distributions. 
Moreover, prior works mainly focus on the expectation bounds of DPSGD. 
However, the operations in DPSGD are constrained by a finite privacy budget, making it difficult to support unlimited algorithm runs. Therefore, we theoretically analyze the high probability bound for classical DPSGD under the heavy-tailed sub-Weibull stochastic gradient noise assumption, as presented in the following theorem.
\begin{table}[h]
\caption{Summary of results under non-convex conditions.}
\label{tab:summary}
\large
\resizebox{\textwidth}{!}{
\begin{tabular}{c|c|c|c|c|c}
\toprule
Measure                      & Proposal        & DPSGD & SGD & Assumption & Clipping \\ \midrule
\multirow{10}{*}{\rotatebox{90}{Expectation}} & Clipped SGD~\cite{zhang2019gradient}           &\ding{53}        &$\displaystyle\mathbb{O}\left(\frac{K^2}{\sqrt{T}}+ \frac{K^{3/2}}{\sqrt[4]{T}}\right)$     &\begin{tabular}[c]{@{}c@{}}$K$-bounded\\variance\end{tabular}            &\checkmark            \\ \cline{2-6} 
                            & NSGD~\cite{yang2022normalized}           &$\displaystyle\mathbb{O}\left(\frac{\sqrt[4]{d\log(1/\delta)}}{(n\epsilon)^{\frac{1}{2}}}\right)$        &\ding{53}     &\begin{tabular}[c]{@{}c@{}}generalized\\smooth\end{tabular}          &\checkmark            \\ \cline{2-6} 
                             &Chen et al.~\cite{chen2020understanding}          &$\displaystyle\mathbb{O}\left(\frac{\sqrt{d}}{n\epsilon}\right)$        &\ding{53}      & symmetry   &\checkmark            \\ \cline{2-6} 
                             &Auto-S~\cite{bu2024automatic}        &$\displaystyle\mathbb{O}\left(\frac{\sqrt{d}}{n\epsilon}\right)$       &$\displaystyle\mathbb{O}\left(\frac{d}{\sqrt[4]{T}}\right)$     & symmetry           &\checkmark            \\ \cline{2-6} 
                             &PDP-SGD~\cite{zhou2020bypassing} &$\displaystyle\mathbb{O}\left(\frac{k}{n\epsilon}\right)$        &\ding{53}    & public data           &\ding{53}          \\ \midrule           
\multirow{13}{*}{\rotatebox{90}{High probability}} &$\begin{tabular}[c]{@{}c@{}}\\ Madden et al.~\cite{madden2020high}  \\ \quad\end{tabular}$               &\ding{53}        &$\displaystyle\mathbb{O}\left(\frac{\sqrt{\log(T)}\log^{\theta}(1/\delta)}{\sqrt{T}}+\frac{\hat{\log}(T/\delta)\log(1/\delta)}{\sqrt{T}}\right)$   &heavy tails            &\ding{53}             \\ \cline{2-6}
                                        &$\begin{tabular}[c]{@{}c@{}}\\ Li et al.~\cite{li2022high}  \\ \quad\end{tabular}$                   &\ding{53}         &$\displaystyle\mathbb{O}\left(\frac{\log^{2\theta}(1/\delta)\log(T)}{\sqrt{T}}+\frac{\hat{\log}(T/\delta)\log(1/\delta)}{\sqrt{T}}\right)$  &heavy tails            &\ding{53}             \\ \cline{2-6}
                                        &$\begin{tabular}[c]{@{}c@{}}\\ Li et al.~\cite{li2022high}  \\ \quad\end{tabular}$                &\ding{53}       &$\displaystyle\mathbb{O}\left(\frac{\log^{\theta}(T/\delta)\log(T)}{\sqrt{T}}+\frac{\log^{2\theta+1}(T)\log(T/\delta)}{\sqrt{T}}\right)$     &heavy tails            & \checkmark           \\ \cline{2-6}
                                        & $\begin{tabular}[c]{@{}c@{}}\text{Our}\\ \text{DPSGD}  \end{tabular}$                  &\multicolumn{2}{c|}{$\displaystyle\mathbb{O}\left(d^\frac{1}{4}\log^{\frac{5}{4}}(T/\delta)\cdot\frac{\hat{\log}(T/\delta)\log^{2\theta}(\sqrt{T})}{(n\epsilon)^{\frac{1}{2}}}\right)$}            &heavy tails            & \checkmark           \\ \cline{2-6}
                             & $\begin{tabular}[c]{@{}c@{}}\text{Our}\\ \text{\ourtech}  \end{tabular}$           &\multicolumn{2}{c|}{$\displaystyle\mathbb{O}\left(d^\frac{1}{4}\log^{\frac{5}{4}}(T/\delta)\cdot\big (p\frac{\hat{\log}(T/\delta)\log^{2\theta}(\sqrt{T})}{(n\epsilon)^{\frac{1}{2}}}+(1-p)(\frac{\log(\sqrt{T})}{(n\epsilon)^{\frac{1}{2}}})\big) \right)$}          &heavy tails    &\checkmark            \\ \bottomrule
\end{tabular}
}
\vspace{-0.5cm}
\end{table}
\begin{theorem}[\textbf{Convergence of Heavy-tailed DPSGD}]
\label{thm:DPSGD}
Under Assumptions~{\ref{ass:sub-weibull}} and~{\ref{ass:smooth}}, let $\mathbf{w}_{t}$ be the iterate produced by DPSGD and $\eta_t = \frac{1}{\sqrt{T}}$. Suppose that $T= \max{\big(m_2eB^2\log(1/\delta), \frac{n\epsilon}{\sqrt{d\log(1/\delta)}}\big)}$ and $c=\max{\big(4K\log^{\theta}({\sqrt{T}}),39K\log^{\theta}(2/\delta)\big)}$, where $m_2$ is a constant that will be introduced later and $d$ is the number of model parameters. For any $\delta \in (0,1)$, with probability $1-\delta$, we have:
\begin{align}
     \frac{1}{T}\sum^T_{t=1} \min\big\{\Vert\nabla L_S(\mathbf{w}_{t})\Vert_2, \Vert\nabla L_S(\mathbf{w}_{t})\Vert^2_2\big\} \leq \mathbb{O}\left(\frac{d^{\frac{1}{4}} \log^{\frac{5}{4}}(T/\delta)\hat{\log}(T/\delta)\log^{2\theta}(\sqrt{T}) }{(n\epsilon)^\frac{1}{2}}\right) , \nonumber
\end{align}
where $\hat{\log}(T/\delta) = \log^{\max(0,\theta-1)}(T/\delta)$. 
\end{theorem}
\begin{remark}
    \normalfont{From Theorem~\ref{thm:DPSGD}, we can derive that as $\theta$ increases, the optimization performance of DPSGD gradually deteriorates. If $\theta = \frac{1}{2}$, the convergence bound will become $\mathbb{O}(d^{\frac{1}{4}} \log^{\frac{5}{4}}(T/\delta)\log(\sqrt{T}) /(n\epsilon)^\frac{1}{2})$, which matches the current optimal expectation bounds of DPSGD variants, i.e., $\mathbb{O}(\sqrt[4]{d\log(1/\delta)}/(n\epsilon)^{\frac{1}{2}})$ in~\cite{yang2022normalized} except for an extra high probability term $\log(T/\delta)\log(\sqrt{T})$.
    This is consistent with the optimization analysis of SGD, where the expectation bound and high probability bound of SGD also differ by such a probability term. 
    When $\theta>1$, the upper bound will increase as the $\hat{\log}(T/\delta)$ term and $\log^{2\theta}(\sqrt{T})$ term increase. In addition, the dependency on the confidence parameter $1/\delta$ is logarithmic, similar to the 
    high probability bounds of SGD~\cite{li2022high,li2023high,madden2020high} summarized in Table~\ref{tab:summary}. To our knowledge, we are the first to use probability bounds as a measure to prove the optimization performance in DPSGD. 
    Besides, suppose $\sqrt{T}=(n\epsilon)^{\frac{1}{2}}/\sqrt[4]{d\log(1/\delta)}$, we can transform the result of DPSGD to $\mathbb{O}(\log(T/\delta)\hat{\log}(T/\delta)\log^{2\theta}(\sqrt{T}) /\sqrt{T})$, which can match the results of clipped SGD~\cite{li2022high} with an improvement $\sqrt{T}$ in the logarithm term $\log^{2\theta}(\sqrt{T})$. }
\end{remark}
\begin{remark}
    \normalfont{From the perspective of the clipping threshold, we can see that the value of $c$ is positively correlated to $\theta$. The ideal clipping threshold should scale up with the increase of the heavy-tailed factor $\theta$. 
    Intuitively speaking, if we utilize the existing guidance for clipping threshold values under the light-tailed assumption, it will cause higher clipping losses for the tailed gradients with larger $L_2$ norms, damaging the effectiveness of DPSGD. 
    }
\end{remark}
Motivated by the above analysis, we now present our approach \ourtech{} that effectively handles the heavy-tailed gradients.
Figure~\ref{fig:2} gives an overview of this approach.
The rationale is to divide gradients following a heavy-tailed sub-Weibull distribution into two parts: light body and heavy tail, and utilize different clipping thresholds for the two parts respectively. 
Then, we adopt a small clipper threshold for light body and 
a larger clipping threshold for heavy tail, so as to mitigate the extra clipping loss introduced by heavy-tailed gradients. 
Specifically, \ourtech{} consists of two steps. 
\vspace{-0.3cm}
\section{Discriminative Clipping DPSGD with Subspace Identification}
\begin{figure}[ht]
	\centering
	\begin{subfigure}{0.95\linewidth}
		\centering
            \includegraphics[width=0.99\linewidth]{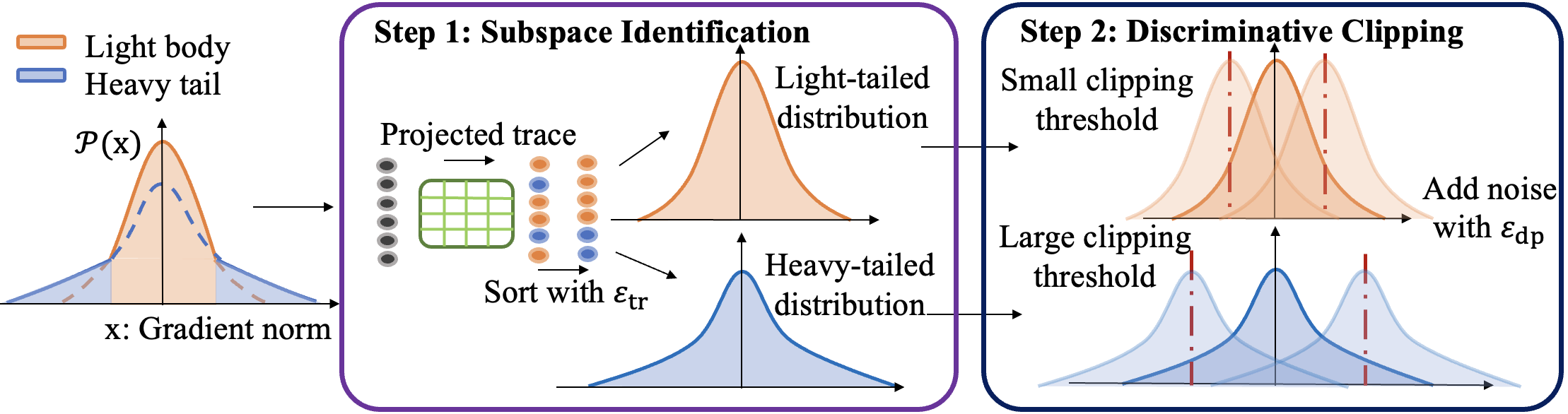}
	\end{subfigure}
        \caption{Overview of~\ourtech.}
        \label{fig:2}
        \vspace{-0.5cm}
\end{figure}
In the first step, we present a subspace identification technique to distinguish gradients. Given that the normalized gradients still retain directional information, which can be amplified when projected onto the subspace consistent with its underlying distribution, we can bypass the unbounded norm of heavy-tailed gradients and capture different responses of private gradients in the heavy-tailed subspace. 
In our approach, we construct projection matrix composed of random vectors following heavy-tailed sub-Weibull distributions~($\theta>\frac{1}{2}$). We then divide the gradients into the light body and heavy tail according to the projected traces $\lambda_{\mathrm{tr}}$, which are calculated from the sample second moment matrix. To satisfy differential privacy,  noise with scale $\sigma_{\mathrm{tr}}$ is added. 

In the second step, we utilize different clipping thresholds for the two parts and add DP noise with scale $\sigma_{\mathrm{dp}}$ based on the discriminative clipping thresholds for privacy preservation. 
For fairness, the total privacy budget allocated by~\ourtech{} to traces $\epsilon_{\mathrm{tr}}$ and gradients $\epsilon_{\mathrm{dp}}$ must be equal to the privacy budget $\epsilon$ in DPSGD variants, that is, $\epsilon = \epsilon_{\mathrm{tr}} + \epsilon_{\mathrm{dp}}$. Algorithm~\ref{alg:algorithm1} presents the detailed steps of \ourtech{} and Theorem~\ref{thm:DCDPSGD} gives its privacy guarantee. 
\begin{theorem}[\textbf{Privacy Guarantee}]
There exist constants $m_1$ and $m_2$ such that for any $\epsilon_{\mathrm{tr}} \leq m_1q^2T$, $\epsilon_{\mathrm{dp}} \leq m_1q^2T$ and $\delta > 0$, the noise multiplier $\sigma^2_{\mathrm{tr}} = \frac{m_2Tq^2\ln{\frac{1}{\delta}}}{\epsilon^2_{\mathrm{tr}}}$ and $\sigma^2_{\mathrm{dp}} = \frac{m_2Tq^2\ln{\frac{1}{\delta}}}{\epsilon^2_{\mathrm{dp}}}$ over the $T$ iterations, where $q=\frac{B}{n}$, and \ourtech~is $(\epsilon_{\mathrm{tr}} + \epsilon_{\mathrm{dp}},\delta)$-differentially private.
\end{theorem}
\vspace{-10pt}
\subsection{Subspace Closeness for Identification}
As introduced above, we use subspace as an auxiliary tool to indirectly identify heavy-tailed gradients and reduce clipping loss.
We construct the subspace $V_kV^T_k$ that is composed of $k$ random orthogonal unit vectors and we need to bound the gap between the empirical second moment and the population second moment, i.e. $\Vert V_kV^T_k-\mathbb{E}[V_kV^T_k]\Vert_2$. 
It is worth noting that we add extra noise in line 9 of Algorithm~\ref{alg:algorithm1}, as the publicly available traces need to be sorted to confirm the top-$p$ heavy-tailed gradients, which may expose intrinsic preferences of the samples. 
According to Ahlswede-Winter Inequality~\cite{wainwright2019high}, 
we analyze the error of subspace skewing in a high probability form.
\begin{theorem}[\textbf{Subspace Skewing for Identification}]
\label{thm:subspace}
Assume that the second moment matrix $M := V_kV^T_k$ with $V^T_kV_k=\mathbb{I}$ approximates the population second moment matrix $\hat{M} := \hat{V}_k\hat{V}^T_k = \mathbb{E}[V_kV^T_k]$, $\lambda_{\mathrm{tr}} := \mathrm{tr}(V^T_k uu^T V_k)$ and $\hat{\lambda}_{\mathrm{tr}}:=\mathrm{tr}(\hat{V}^T_k uu^T\hat{V}_k) $, for any vector $u$ that satisfies $\Vert u \Vert_2 = 1$, $\zeta_{\mathrm{tr}} \sim \mathbb{N}(0,\sigma^2_{\mathrm{tr}}\mathbb{I})$ and $\delta \in (0,1)$, with probability $1-\delta_m - \delta$, we have
\begin{align}
    |\lambda_{\mathrm{tr}} - \hat{\lambda}_{\mathrm{tr}}+\zeta_{\mathrm{tr}}| \leq \frac{4\log{(2d/\delta_m)}}{k} + \sigma_{\mathrm{tr}}\log^{\frac{1}{2}}(2/\delta) . \nonumber
\end{align}
\end{theorem}
\begin{remark}
    \normalfont{Because we have normalized per-sample gradient in advance, the upper bound of the trace for per-sample gradient is limited to 1. So, the sensitivity in differential privacy can be regarded as 1. In addition, since $\lambda_{\mathrm{tr}}$ is a constant, the scale $\sigma_{\mathrm{tr}}$ of noise added is small compared to the noise scale $\sigma^2_{\mathrm{dp}}$ added to gradients. In this case, the probability term $ \log{(2d/\delta_m)}/k$ dominates this boundary and decreases as $k$ increases, so the error is negligible when $k$ is large. 
    Since $\lambda_{\mathrm{tr}}$ represents the total variance of the gradient in the $k$-dimensional subspace, we know from Theorem~\ref{thm:subspace} that the upper bound of the error between this value and the variance of the actual distribution is finite. 
    In other words, Theorem~\ref{thm:subspace} indicates that we can accurately identify and classify gradients with a high probability $1-\delta^{\prime}_m$, where $\delta^{\prime}_m = \delta + \delta_m$.}
\end{remark}

\begin{algorithm}[t]
\caption{Discriminative Clipping DPSGD with Subspace Identification}
\label{alg:algorithm1}
\textbf{Input}: Private batch size $B$, heavy-tailed ratio $p$, heavy-tailed clipping threshold $c_1$, light-tailed clipping threshold $c_2$, learning rate $\eta_t$ and subspace dimension $k$.
\begin{algorithmic}[1] 
\STATE Initialize $\mathbf{w}_0$ randomly.
\FOR{$e \in E$}
\STATE Initialize $V_{t,k}$ to None.
\FOR{$t \in T$}
\STATE Take a random batch $B$ with sampling ratio $B/n$ and $g_t(z_i) = \nabla\ell(\mathbf{w}_t,z_i)$.
\STATE Extract orthogonal vectors $[v_1,...,v_k]$ from sub-Weibull distributions and construct projection subspace with $V_{t,k}V^T_{t,k} = \frac{1}{k}\sum^k_{i=1}v_iv^T_i$.
\STATE Normalize per-sample gradient $\hat{g}_t(x_i) = g_t(x_i)/\Vert g_t(x_i) \Vert$.
\STATE Calculate the trace $\lambda_i$ of the projected second moment $V^T_{t,k}\hat{g}_t(x_i)\hat{g}^T_t(x_i)V_{t,k}$.
\STATE Perturb traces with noise $\widetilde{\lambda}_i = \lambda_i + \mathbb{N}(0,\sigma^2_{\mathrm{tr}}\mathbb{I})$ and identify top-$p$ based on sorted $\widetilde{\lambda}_i$.
\STATE Clip per-sample gradient and add noise. 
\\For heavy tail:~ $\overline{g}_t(z^{\mathrm{tail}}_i) = g_t(z^{\mathrm{tail}}_i) / \mathrm{max}(1, \frac{\Vert g_t(z^{\mathrm{tail}}_i)\Vert_2}{c_1}) + \mathbb{N}(0,c_1^2\sigma^2_{\mathrm{dp}}\mathbb{I})$ 
\\For light body:~ $\overline{g}_t(z^{\mathrm{body}}_i) = g_t(z^{\mathrm{body}}_i) / \mathrm{max}(1, \frac{\Vert g_t(z^{\mathrm{body}}_i)\Vert_2}{c_2}) + \mathbb{N}(0,c_2^2\sigma^2_{\mathrm{dp}}\mathbb{I})$ 
\STATE Weighted average $\widetilde{g}_t = \frac{1}{B}\left(\sum^{pB}_{i=1}\overline{g}_t(z^{\mathrm{tail}}_i) + \sum^{(1-p)B}_{i=1}\overline{g}_t(z^{\mathrm{body}}_i) \right)$.
\STATE Update $\mathbf{w}_{t+1} = \mathbf{w}_t - \eta_t\widetilde{g}_t$.
\ENDFOR
\ENDFOR
\end{algorithmic}
\end{algorithm}
\vspace{-0.4cm}
\subsection{Convergence of Discriminative Clipping DPSGD}
Next, we delve into the convergence analysis of \ourtech{} based on the aforementioned clipping mechanism. 
Typically, the tail probability $\mathbb{P}(|x|>t)=\exp{(-I(t))}~\forall{t>0}$ of the sub-Weibull variables $x\sim subW(\theta,K)$ exhibits two different behaviors: 1) For small $t$ values, the tail rate capturing function $I(t)$ decays like a sub-Gaussian tail. 2) For $t$ greater than the normal convergence region, i.e., $t \geq t_{\mathrm{max}}$ is a large deviation region, its decay is slower than that of the normal distribution. 
Existing literature has studied the first region in the optimization analysis for DPSGD~\cite{zhou2020bypassing,bu2024automatic,yang2022normalized,xia2023differentially,cheng2022differentially,xiao2023theory,sha2023pcdp}, but they overlook the heavy-tailed behavior for the second region. 
In our work, 
we not only investigate the optimization performance of one specific region, but also combine the two tailed-rate regions with our proposed discrimination clipping mechanism. To construct a comprehensive optimization framework under heavy-tailed assumptions, we generalize the sharp heavy-tailed concentration~\cite{bakhshizadeh2023sharp} and sub-Weibull Freedman inequality~\cite{madden2020high} to truncated versions. Consequently, we have the following theorem:
\begin{theorem}[\textbf{Convergence of Discriminative Clipping DPSGD}]
\label{thm:DCDPSGD}
Under Assumptions~{\ref{ass:sub-weibull}},~{\ref{ass:smooth}} and~{\ref{ass:G-bounded}}, let $\mathbf{w}_{t}$ be the iterate produced by~\ourtech~and $\eta_t = \frac{1}{\sqrt{T}}$. Define $\hat{\log}(T/\delta) = \log^{\max(0,\theta-1)}(T/\delta)$, $\lambda_{\mathrm{max}} = \frac{\mu I(\lambda)}{\lambda}aK^2$, $a = 2$ if $\theta = \frac{1}{2}$, $a = (4\theta)^{2\theta}e^2$ if $\theta \in (\frac{1}{2},1]$, and $a = (2^{2\theta+1}+2)\Gamma(2\theta+1) + \frac{2^{3\theta}\Gamma(3\theta+1)}{3}$ if $\theta > 1$, for any $\delta \in (0,1)$, then we have:
\begin{enumerate}[leftmargin=*]
\item[(i).]
For the case $0\leq \lambda_{\mathrm{tr}} \leq \lambda_{\mathrm{max}}$, suppose that $T= \max{\big(m_2eB^2\log(1/\delta), \frac{n\epsilon}{\sqrt{d\log(1/\delta)}}\big)}$ and $c=\max{\big(2\sqrt{2a}K\log^{\frac{1}{2}}({\sqrt{T}}),33\sqrt{2a}K\log^{\frac{1}{2}}(2/\delta)\big)}$, with probability $1-\delta$,
\begin{align}
     \frac{1}{T}\sum^T_{t=1} \min\big\{\Vert\nabla L_S(\mathbf{w}_{t})\Vert_2, \Vert\nabla L_S(\mathbf{w}_{t})\Vert^2_2\big\} \leq \mathbb{O}\left(\frac{d^{\frac{1}{4}} \log^{\frac{5}{4}}(T/\delta)\log(\sqrt{T}) }{(n\epsilon)^\frac{1}{2}}\right). \nonumber
\end{align}
\item[(ii).]
For the case $\lambda_{\mathrm{tr}} \geq \lambda_{\mathrm{max}}$, suppose that $T= \max{\big(m_2eB^2\log(1/\delta), \frac{n\epsilon}{\sqrt{d\log(1/\delta)}}\big)}$ and $c=\max{\big(4^{\theta}2K\log^{\theta}({\sqrt{T}}),4^{\theta}33K\log^{\theta}(2/\delta)\big)}$, with probability $1-\delta$,
\vspace{-0.2cm}
\begin{align}
     \frac{1}{T}\sum^T_{t=1} \min\big\{\Vert\nabla L_S(\mathbf{w}_{t})\Vert_2, \Vert\nabla L_S(\mathbf{w}_{t})\Vert^2_2\big\} \leq \mathbb{O}\left(\frac{d^{\frac{1}{4}} \log^{\frac{5}{4}}(T/\delta)\hat{\log}(T/\delta)\log^{2\theta}(\sqrt{T}) }{(n\epsilon)^\frac{1}{2}}\right). \nonumber
\end{align}
\end{enumerate}
\end{theorem}
\begin{remark}
    \normalfont{From Theorem~\ref{thm:DCDPSGD}, we can infer that when gradients fall into the first light body region, i.e. $0\leq\lambda\leq \lambda_{\mathrm{max}}$, our results no longer contain the heavy-tailed index $\theta$. This implies that in this region, the optimization performance of the algorithm is not directly affected by the heavy-tailed assumption and always converges according to the light-tailed sub-Gaussian rate. However, when gradients are classified into the second heavy-tailed region, i.e. $\lambda \geq \lambda_{\mathrm{max}}$, the behavior of convergence in this part will remain the same as that of classic DPSGD, becoming deteriorated with the increase of $\theta$. Specifically, the optimization performance in the first region is actually a transformation of the second region when $\theta=\frac{1}{2}$. In the first light body part, our guidance for the clipping threshold depends on the logarithmic factor $\log^{1/2}$, but in the second heavy-tailed region, our theoretical clipping threshold increases with the heavy-tailed index $\log^{\theta}$. For $\lambda_{\mathrm{max}}$, it is correlated with the population variance of the underlying distribution in Assumption~\ref{ass:sub-weibull}~\cite{bakhshizadeh2023sharp}, and we empirically use the trace of the second moment to approximate the total variance of the gradients in the subspace, where the approximation error has been bounded in Theorem~\ref{thm:subspace}. In summary, unlike existing optimization results on the heavy-tailed assumption that entirely rely on the heavy-tailed index~\cite{li2022high,madden2020high}, our \ourtech{} bounds are partially free from the dependence on heavy tails and can provide theoretical guidance on large clipping thresholds.}
\end{remark}
\vspace{-5pt}
\subsection{Uniform Bound for Heavy-tailed~\ourtech}
According to Algorithm~\ref{alg:algorithm1} and Theorem~\ref{thm:subspace}, we note that the premise of discriminative clipping relies on the classification of gradients by the subspace. 
However, in practice, this step incurs errors and losses, leading to a misalignment between Theorem~\ref{thm:DCDPSGD} and the algorithm. Considering that the accuracy of subspace identification holds with high probability at $1-\delta^{\prime}_m$, we need to re-analyze the convergence associated with partitioning regions in~\ourtech. Therefore, in this section, we will merge Theorems~\ref{thm:subspace} and~~\ref{thm:DCDPSGD} to derive the final bound for Algorithm~\ref{alg:algorithm1}.
\begin{theorem}[\textbf{Uniform Bound for \ourtech{}}] 
\label{thm:UniB}
Under Assumptions~\ref{ass:sub-weibull},~\ref{ass:smooth}~and~\ref{ass:G-bounded}, combining Theorem~\ref{thm:subspace} and Theorem~\ref{thm:DCDPSGD}, for any $\delta^{\prime}\in(0,1)$, with probability $1-\delta^{\prime}$ and $\mathcal{C}_{\mathrm{u}} := \sum^T_{t=1}\min\{\Vert\nabla \hat{L}_S(\mathbf{w}_t) \Vert^2_2, \Vert\nabla \hat{L}_S(\mathbf{w}_t) \Vert_2\}$, we have
\begin{align}
    \mathcal{C}_{\mathrm{u}} &\leq p*\mathbb{O}\left(\frac{d^{\frac{1}{4}} \log^{\frac{5}{4}}(T/\delta)\hat{\log}(T/\delta)\log^{2\theta}(\sqrt{T}) }{(n\epsilon)^\frac{1}{2}}\right) + (1-p)*\mathbb{O}\left(\frac{d^{\frac{1}{4}} \log^{\frac{5}{4}}(T/\delta)\log(\sqrt{T}) }{(n\epsilon)^\frac{1}{2}}\right), \nonumber
\end{align}
\end{theorem}
where $\delta^{\prime}=\delta^{\prime}_m + \delta$, $\hat{\log}(T/\delta) = \log^{\max(0,\theta-1)}(T/\delta)$ and $p$ is ratio of heavy-tailed gradients.
\begin{remark}
    \normalfont{Theorem~\ref{thm:UniB} states that when the proportion of the tail region is $p$, the optimization performance of \ourtech{} with subspace identification is composed of $p$-weighted average bounds, where the heavy-tailed convergence rate merely accounts for a portion of $p$, with the rest made up of the light rate. Therefore, our bound minimizes the dependency on $\theta$ from $\hat{\log}(T/\delta)\log^{2\theta}(\sqrt{T})$ to $\log(\sqrt{T})$ with percentage $(1-p)*(1-\delta^{\prime})$, which is tighter than DPSGD. According to the statistical properties~\cite{wainwright2019high,vershynin2018high}, around 5\% -10\% data points will fall to the tail, that is, $p\in[0.05, 0.1]$. The probability term $\delta^{\prime}$ includes both $\delta^{\prime}_m$ and $\delta$, with $\delta^{\prime}_m$ being the error of subspace identification and $\delta$ being the convergence probability of~\ourtech.}
\end{remark}

\vspace{-0.2cm}
\section{Experiments}
\vspace{-0.075cm}
\subsection{Experimental Setup}
\vspace{-0.075cm}
We use four real-world datasets in the experiments, including MNIST, FMNIST, CIFAR10, and ImageNette~(a subset of ImageNet~\cite{deng2009imagenet}). We further utilize two heavy-tailed versions: namely CIFAR10-HT~\cite{cao2019learning}~(a heavy-tailed version of CIFAR10) and ImageNette-HT~(modified on \cite{park2021influence}), to evaluate the performance under heavy tail assumption. 
For MNIST and FMNIST, we use a two-layer CNN with batch size $B=128$. For CIFAR10 and CIFAR10-HT, we take $B=256$ and fine-tune on model SimCLRv2 pre-trained by unlabeled ImageNet and ResNeXt-29 pre-trained by CIFAR100~\cite{tramer2020differentially} with a linear classifier, respectively. For ImageNette and ImageNette-HT, we adopt the same settings~\cite{bu2024automatic} and ResNet9 without pre-train. 

We compare \ourtech{} with three differentially private baselines: DPSGD with Abadi's clipping~\cite{abadi2016deep}, Auto-S/NSGD~\cite{bu2024automatic,yang2022normalized}, DP-PSAC~\cite{xia2023differentially}, and a non-private baseline: non-DP~($\epsilon=\infty$).
In addition, we set $c_2=0.1$ and $\eta=0.1$ for MNIST and FMNIST. For CIFAR10 tasks, we set $c_2=0.01$ and $\eta=10$, and let $c_2=0.15$, $\eta=0.0001$ and $B=1000$ on ImageNette tasks. The large clipping threshold is set as $c_1=10*c_2$. We implement per-sample clipping in private SGD by BackPACK~\cite{dangel2020backpack} and allocate privacy budget fairly according to $\epsilon = \epsilon_{\mathrm{dp}} + \epsilon_{\mathrm{tr}}$. 

\subsection{Effectiveness Evaluation}
Table~\ref{Table:test_acc} summarizes the test accuracy of \ourtech{} and baselines. 
We can observe that, on normal datasets, \ourtech{} outperforms DPSGD, Auto-S, and DP-PSAC by up to 4.57\%, 5.42\%, and 4.99\%, respectively. While on heavy-tailed datasets, the corresponding improvements are 8.34\%, 9.72\%, and 9.55\%.
The reason is that 
our approach places greater emphasis on the clipping weight of heavy-tailed gradients, thereby preserving more information about heavy-tailed gradients and improving accuracy. 
Moreover, we demonstrate the trajectories of training accuracy in Figure~\ref{fig:3}, indicating that the optimization performance of~\ourtech~is superior to existing clipping mechanisms.

\renewcommand{\arraystretch}{1.1}
\begin{table}[htbp]
\vspace{-15pt}
\caption{Test accuracy of baselines and~\ourtech.}
\label{Table:test_acc}
\resizebox{\textwidth}{!}{
\begin{tabular}{c|c|ccccc}
\toprule
\multirow{2}{*}{Dataset} & \multirow{2}{*}{\begin{tabular}[c]{@{}c@{}}DP\\ ($\epsilon, \delta$)\end{tabular} }  & \multicolumn{5}{c}{Accuracy \%}                                                                                               \\ \cline{3-7} 
                         &                              & \multicolumn{1}{c|}{DPSGD~\cite{abadi2016deep}} & \multicolumn{1}{c|}{Auto-S~\cite{bu2024automatic,yang2022normalized}} & \multicolumn{1}{c|}{DP-PSAC~\cite{xia2023differentially}} & \multicolumn{1}{c|}{Ours} & non-DP \\ \midrule    
MNIST                    & (8,$1e^{-5}$)                              & \multicolumn{1}{c|}{97.65$\pm$0.09}      & \multicolumn{1}{c|}{97.55$\pm$0.16}            & \multicolumn{1}{c|}{97.67$\pm$0.06}     & \multicolumn{1}{c|}{\textbf{98.72$\pm$0.02}}       &\multicolumn{1}{c}{99.10$\pm$0.02}       \\ 
FMNIST                   & (8,$1e^{-5}$)                              & \multicolumn{1}{c|}{83.23$\pm$0.10}      & \multicolumn{1}{c|}{82.38$\pm$0.15}            & \multicolumn{1}{c|}{82.81$\pm$0.18}     & \multicolumn{1}{c|}{\textbf{87.80$\pm$0.47}}       &\multicolumn{1}{c}{89.95$\pm$0.32}     \\ 
CIFAR10                  & (8,$1e^{-5}$)                                & \multicolumn{1}{c|}{93.31$\pm$0.01}      & \multicolumn{1}{c|}{93.28$\pm$0.06}            & \multicolumn{1}{c|}{93.30$\pm$0.03}     & \multicolumn{1}{c|}{\textbf{94.05$\pm$0.11}}       &\multicolumn{1}{c}{94.62$\pm$0.03}      \\ 
CIFAR10                  & (4,$1e^{-5}$)                                & \multicolumn{1}{c|}{93.06$\pm$0.09}      & \multicolumn{1}{c|}{93.08$\pm$0.06}            & \multicolumn{1}{c|}{93.11$\pm$0.08}     & \multicolumn{1}{c|}{\textbf{93.42$\pm$0.14}}       &\multicolumn{1}{c}{94.62$\pm$0.03}      \\ 
ImageNette               & (8,$1e^{-4}$)                              & \multicolumn{1}{c|}{66.81$\pm$0.42}      & \multicolumn{1}{c|}{65.57$\pm$0.85}            & \multicolumn{1}{c|}{65.68$\pm$1.71}     & \multicolumn{1}{c|}{\textbf{69.29$\pm$0.19}}       &\multicolumn{1}{c}{71.67$\pm$0.49}       \\ 
CIFAR10-HT                               & (8,$1e^{-5}$)               & \multicolumn{1}{c|}{57.98$\pm$0.59}      & \multicolumn{1}{c|}{58.30$\pm$0.61}            & \multicolumn{1}{c|}{57.99$\pm$0.58}     & \multicolumn{1}{c|}{\textbf{62.57$\pm$1.03}}       &\multicolumn{1}{c}{71.74$\pm$0.65}       \\ 
ImageNette-HT                &(8,$1e^{-4}$)                              & \multicolumn{1}{c|}{25.36$\pm$1.71}      & \multicolumn{1}{c|}{23.98$\pm$2.00}            & \multicolumn{1}{c|}{24.15$\pm$1.99}     & \multicolumn{1}{c|}{\textbf{33.70$\pm$0.91}}       &\multicolumn{1}{c}{39.91$\pm$1.46}       \\ \bottomrule
\end{tabular}
}
\vspace{-5pt}
\end{table}

We then evaluate the effects of three parameters on test accuracy, including the subspace-$k$, the allocation of privacy budget $\epsilon$, and the heavy tail index sub-Weibull-$\theta$. The results are shown in Table~\ref{Table:ablation_study}.
We can see that the test accuracy increases with the value of $k$, which aligns with the theory that the trace error is related to $\mathbb{O}(1/k)$ and has a small impact on the results. 
For the allocation of privacy budget between subspace identification and privacy oracle, we find that allocation biased towards moderate or $\epsilon_{\mathrm{tr}}$ is better due to the high dimensionality of gradients. 
For subspace distribution, since the `HT' dataset is extracted through sub-Exponential distributions, the gradient exhibits a heavier tail phenomenon in networks. Therefore, the accuracy increases as $\theta$ becomes larger.

\renewcommand{\arraystretch}{1.1}
\begin{table}[htbp]
\vspace{-15pt}
\caption{Effects of parameters on test accuracy.}
\label{Table:ablation_study}
\resizebox{\textwidth}{!}{
\begin{tabular}{c|cccc|ccc|ccc}
\toprule
\multirow{2}{*}{Dataset} & \multicolumn{4}{c|}{Subspace-$k$}                                                               & \multicolumn{3}{c|}{$\epsilon_{\mathrm{tr}}$ + $\epsilon_{\mathrm{dp}}$}                                  & \multicolumn{3}{c}{sub-Weibull-$\theta$}                      \\ \cline{2-11} 
                         & \multicolumn{1}{c}{None} & \multicolumn{1}{c}{100} & \multicolumn{1}{c}{150} & 200 & \multicolumn{1}{c}{2+6} & \multicolumn{1}{c}{4+4} & \multicolumn{1}{c|}{6+2} & \multicolumn{1}{c}{1/2} & \multicolumn{1}{c}{1} & 2 \\ \midrule
CIFAR10                  & \multicolumn{1}{c}{93.07}            & \multicolumn{1}{c}{93.82}    & \multicolumn{1}{c}{93.96}    &94.05     & \multicolumn{1}{c}{93.92}  & \multicolumn{1}{c}{94.05}  & \multicolumn{1}{c|}{93.37}     & \multicolumn{1}{c}{93.88}    & \multicolumn{1}{c}{93.99}  &94.05   \\ 
CIFAR10-HT               & \multicolumn{1}{c}{57.27}            & \multicolumn{1}{c}{61.60}    & \multicolumn{1}{c}{62.48 }    &62.57    & \multicolumn{1}{c}{62.54}  & \multicolumn{1}{c}{62.57}  & \multicolumn{1}{c|}{60.07}      & \multicolumn{1}{c}{61.58}    & \multicolumn{1}{c}{62.28}  &62.57   \\ \bottomrule
\end{tabular}
}
\end{table}

\begin{figure}[h!]
    \vspace{-0.3cm}
    \centering
    \begin{minipage}[]{0.33\textwidth}
        \includegraphics[width=0.99\textwidth]{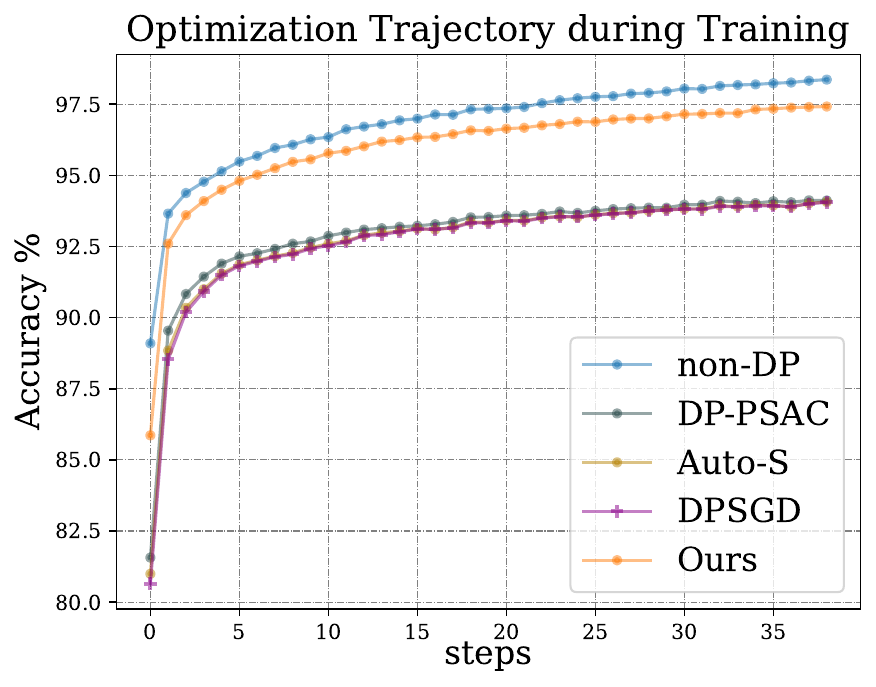}
        \caption{Optimization performance during CIFAR10 Training.}
    \label{fig:3}
    \end{minipage}
    \begin{minipage}[]{0.66\textwidth}
        \includegraphics[width=0.5\textwidth]{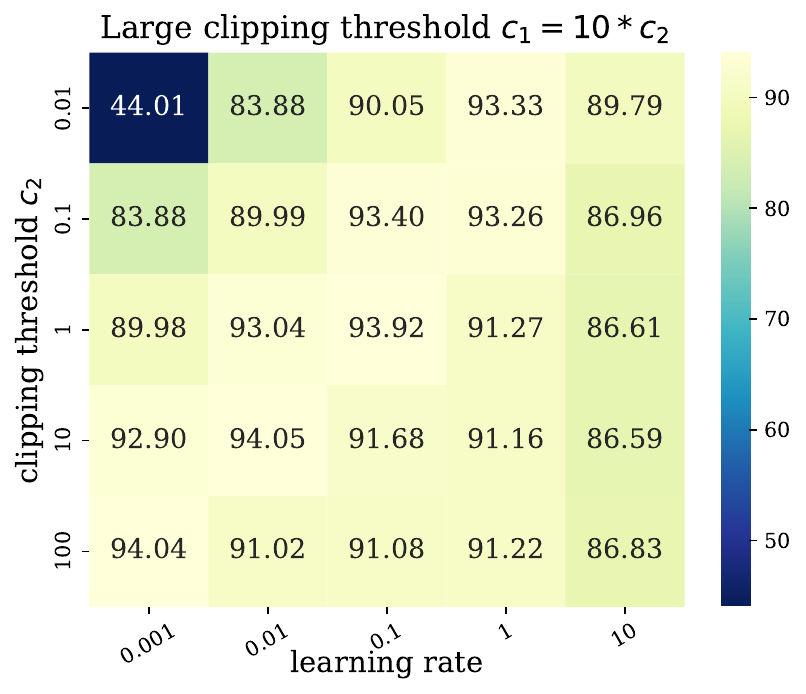}
        \includegraphics[width=0.5\textwidth]{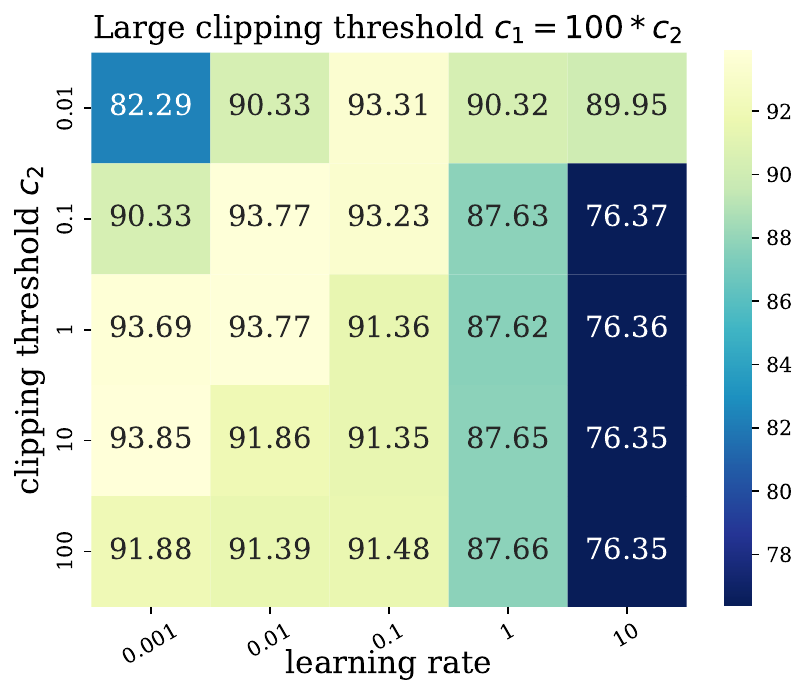}
        \caption{Test accuracy heatmap on CIFAR10 with $c_1$, $c_2$ and $\eta$.}
        \label{fig:4}
    \end{minipage}
\end{figure}

\vspace{-10pt}
\subsection{Guidance for Large Clipping Threshold}
Based on the theoretical analysis in~Theorem~\ref{thm:DCDPSGD} and experimental results in Figure~\ref{fig:4}, we can provide a recommended interval of clipping threshold for~\ourtech. Taking CIFAR-10 as an example, where $\delta = 1e^{-5}$ and $\eta/B=0.04$, we combine the empirical $\theta\approx 2$ with the theoretical guidance in~\cite{gurbuzbalaban2021heavy}. Consequently, we obtain $c_1=\mathbb{O}(\log^{\theta}(1/\delta))$ is
$\sqrt{125}$ times larger than $c_2=\mathbb{O}(\log^{1/2}(1/\delta))$, that is, $c_1 =\log^{3/2}(1/\delta)c_2 $ and then $c_1 \approx 10c_2$.


\section{Conclusion}
In this paper, we propose a novel approach \ourtech{} under the heavy-tailed assumption, 
which effectively reduces extra clipping loss in the heavy-tailed region. 
We rigorously analyze the high-probability bound of the classic heavy-tailed DPSGD under non-convex conditions and obtain results matching the expectation bounds. Furthermore, 
we characterize the weighted average optimization performance of~\ourtech{}. Extensive experiments on four real-world datasets validate that~\ourtech{} outperforms three state-of-the-art clipping mechanisms for heavy-tailed gradients. 


\clearpage

{
    \small
    \bibliographystyle{plain}
    \bibliography{neurips_2024}

\begin{thebibliography}{10}

\bibitem{abadi2016deep}
Martin Abadi, Andy Chu, Ian Goodfellow, H~Brendan McMahan, Ilya Mironov, Kunal Talwar, and Li~Zhang.
\newblock Deep learning with differential privacy.
\newblock In {\em SIGSAC}, pages 308--318, 2016.

\bibitem{adnan2022federated}
Mohammed Adnan, Shivam Kalra, Jesse~C Cresswell, Graham~W Taylor, and Hamid~R Tizhoosh.
\newblock Federated learning and differential privacy for medical image analysis.
\newblock {\em Scientific reports}, 12(1):1953, 2022.

\bibitem{andrew2021differentially}
Galen Andrew, Om~Thakkar, Brendan McMahan, and Swaroop Ramaswamy.
\newblock Differentially private learning with adaptive clipping.
\newblock {\em NeurIPS}, 34:17455--17466, 2021.

\bibitem{bakhshizadeh2023sharp}
Milad Bakhshizadeh, Arian Maleki, and Victor~H De~La~Pena.
\newblock Sharp concentration results for heavy-tailed distributions.
\newblock {\em Information and Inference: A Journal of the IMA}, 12(3):1655--1685, 2023.

\bibitem{barsbey2021heavy}
Melih Barsbey, Milad Sefidgaran, Murat~A Erdogdu, Gael Richard, and Umut Simsekli.
\newblock Heavy tails in sgd and compressibility of overparametrized neural networks.
\newblock {\em NeurIPS}, 34:29364--29378, 2021.

\bibitem{bu2024automatic}
Zhiqi Bu, Yu-Xiang Wang, Sheng Zha, and George Karypis.
\newblock Automatic clipping: Differentially private deep learning made easier and stronger.
\newblock {\em NeurIPS}, 36, 2024.

\bibitem{camuto2021asymmetric}
Alexander Camuto, Xiaoyu Wang, Lingjiong Zhu, Chris Holmes, Mert Gurbuzbalaban, and Umut Simsekli.
\newblock Asymmetric heavy tails and implicit bias in gaussian noise injections.
\newblock In {\em ICML}, pages 1249--1260. PMLR, 2021.

\bibitem{cao2019learning}
Kaidi Cao, Colin Wei, Adrien Gaidon, Nikos Arechiga, and Tengyu Ma.
\newblock Learning imbalanced datasets with label-distribution-aware margin loss.
\newblock {\em NeurIPS}, 32, 2019.

\bibitem{chen2020understanding}
Xiangyi Chen, Steven~Z Wu, and Mingyi Hong.
\newblock Understanding gradient clipping in private sgd: A geometric perspective.
\newblock {\em NeurIPS}, 33:13773--13782, 2020.

\bibitem{cheng2022differentially}
Anda Cheng, Peisong Wang, Xi~Sheryl Zhang, and Jian Cheng.
\newblock Differentially private federated learning with local regularization and sparsification.
\newblock In {\em CVPR}, pages 10122--10131, 2022.

\bibitem{cutkosky2020momentum}
Ashok Cutkosky and Harsh Mehta.
\newblock Momentum improves normalized sgd.
\newblock In {\em ICML}, pages 2260--2268. PMLR, 2020.

\bibitem{dangel2020backpack}
Felix Dangel, Frederik Kunstner, and Philipp Hennig.
\newblock Back{PACK}: Packing more into backprop.
\newblock In {\em ICLR}, 2020.

\bibitem{davis2020high}
Damek Davis and Dmitriy Drusvyatskiy.
\newblock High probability guarantees for stochastic convex optimization.
\newblock In {\em COLT}, pages 1411--1427. PMLR, 2020.

\bibitem{deng2009imagenet}
Jia Deng, Wei Dong, Richard Socher, Li-Jia Li, Kai Li, and Li~Fei-Fei.
\newblock Imagenet: A large-scale hierarchical image database.
\newblock In {\em CVPR}, pages 248--255. Ieee, 2009.

\bibitem{dong2019gaussian}
Jinshuo Dong, Aaron Roth, and Weijie~J Su.
\newblock Gaussian differential privacy.
\newblock {\em arXiv preprint arXiv:1905.02383}, 2019.

\bibitem{dwork2006calibrating}
Cynthia Dwork, Frank McSherry, Kobbi Nissim, and Adam Smith.
\newblock Calibrating noise to sensitivity in private data analysis.
\newblock In {\em Theory of Cryptography: Third Theory of Cryptography Conference, TCC 2006, New York, NY, USA, March 4-7, 2006. Proceedings 3}, pages 265--284. Springer, 2006.

\bibitem{fang2022improved}
Huang Fang, Xiaoyun Li, Chenglin Fan, and Ping Li.
\newblock Improved convergence of differential private sgd with gradient clipping.
\newblock In {\em ICLR}, 2022.

\bibitem{foster2018uniform}
Dylan~J Foster, Ayush Sekhari, and Karthik Sridharan.
\newblock Uniform convergence of gradients for non-convex learning and optimization.
\newblock {\em NeurIPS}, 31, 2018.

\bibitem{ghalebikesabi2023differentially}
Sahra Ghalebikesabi, Leonard Berrada, Sven Gowal, Ira Ktena, Robert Stanforth, Jamie Hayes, Soham De, Samuel~L Smith, Olivia Wiles, and Borja Balle.
\newblock Differentially private diffusion models generate useful synthetic images.
\newblock {\em arXiv preprint arXiv:2302.13861}, 2023.

\bibitem{golatkar2022mixed}
Aditya Golatkar, Alessandro Achille, Yu-Xiang Wang, Aaron Roth, Michael Kearns, and Stefano Soatto.
\newblock Mixed differential privacy in computer vision.
\newblock In {\em CVPR}, pages 8376--8386, 2022.

\bibitem{gorbunov2020stochastic}
Eduard Gorbunov, Marina Danilova, and Alexander Gasnikov.
\newblock Stochastic optimization with heavy-tailed noise via accelerated gradient clipping.
\newblock {\em NeurIPS}, 33:15042--15053, 2020.

\bibitem{gu2023choosing}
Xin Gu, Gautam Kamath, and Zhiwei~Steven Wu.
\newblock Choosing public datasets for private machine learning via gradient subspace distance.
\newblock {\em arXiv preprint arXiv:2303.01256}, 2023.

\bibitem{gurbuzbalaban2021heavy}
Mert Gurbuzbalaban, Umut Simsekli, and Lingjiong Zhu.
\newblock The heavy-tail phenomenon in sgd.
\newblock In {\em ICML}, pages 3964--3975. PMLR, 2021.

\bibitem{harder2022pre}
Fredrik Harder, Milad~Jalali Asadabadi, Danica~J Sutherland, and Mijung Park.
\newblock Pre-trained perceptual features improve differentially private image generation.
\newblock {\em arXiv preprint arXiv:2205.12900}, 2022.

\bibitem{he2016deep}
Kaiming He, Xiangyu Zhang, Shaoqing Ren, and Jian Sun.
\newblock Deep residual learning for image recognition.
\newblock In {\em CVPR}, pages 770--778, 2016.

\bibitem{ji2022privacy}
Jiazhen Ji, Huan Wang, Yuge Huang, Jiaxiang Wu, Xingkun Xu, Shouhong Ding, ShengChuan Zhang, Liujuan Cao, and Rongrong Ji.
\newblock Privacy-preserving face recognition with learnable privacy budgets in frequency domain.
\newblock In {\em ECCV}, pages 475--491. Springer, 2022.

\bibitem{kairouz2015composition}
Peter Kairouz, Sewoong Oh, and Pramod Viswanath.
\newblock The composition theorem for differential privacy.
\newblock In {\em ICML}, pages 1376--1385. PMLR, 2015.

\bibitem{kamath2022improved}
Gautam Kamath, Xingtu Liu, and Huanyu Zhang.
\newblock Improved rates for differentially private stochastic convex optimization with heavy-tailed data.
\newblock In {\em ICML}, pages 10633--10660. PMLR, 2022.

\bibitem{koloskova2023revisiting}
Anastasia Koloskova, Hadrien Hendrikx, and Sebastian~U Stich.
\newblock Revisiting gradient clipping: Stochastic bias and tight convergence guarantees.
\newblock In {\em ICML}, pages 17343--17363. PMLR, 2023.

\bibitem{li2022high}
Shaojie Li and Yong Liu.
\newblock High probability guarantees for nonconvex stochastic gradient descent with heavy tails.
\newblock In {\em ICML}, pages 12931--12963. PMLR, 2022.

\bibitem{li2023high}
Shaojie Li and Yong Liu.
\newblock High probability analysis for non-convex stochastic optimization with clipping.
\newblock {\em arXiv preprint arXiv:2307.13680}, 2023.

\bibitem{li2020high}
Xiaoyu Li and Francesco Orabona.
\newblock A high probability analysis of adaptive sgd with momentum.
\newblock {\em arXiv preprint arXiv:2007.14294}, 2020.

\bibitem{liang2020think}
Paul~Pu Liang, Terrance Liu, Liu Ziyin, Nicholas~B Allen, Randy~P Auerbach, David Brent, Ruslan Salakhutdinov, and Louis-Philippe Morency.
\newblock Think locally, act globally: Federated learning with local and global representations.
\newblock {\em arXiv preprint arXiv:2001.01523}, 2020.

\bibitem{liu2023privaterec}
Ruixuan Liu, Yang Cao, Yanlin Wang, Lingjuan Lyu, Yun Chen, and Hong Chen.
\newblock Privaterec: Differentially private model training and online serving for federated news recommendation.
\newblock In {\em SIGKDD}, pages 4539--4548, 2023.

\bibitem{lowy2023private}
Andrew Lowy and Meisam Razaviyayn.
\newblock Private stochastic optimization with large worst-case lipschitz parameter: Optimal rates for (non-smooth) convex losses and extension to non-convex losses.
\newblock In {\em International Conference on Algorithmic Learning Theory}, pages 986--1054. PMLR, 2023.

\bibitem{madden2020high}
Liam Madden, Emiliano Dall'Anese, and Stephen Becker.
\newblock High-probability convergence bounds for non-convex stochastic gradient descent.
\newblock {\em arXiv preprint arXiv:2006.05610}, 2020.

\bibitem{meng2021improving}
Qiang Meng, Feng Zhou, Hainan Ren, Tianshu Feng, Guochao Liu, and Yuanqing Lin.
\newblock Improving federated learning face recognition via privacy-agnostic clusters.
\newblock In {\em ICLR}, 2021.

\bibitem{mironov2017renyi}
Ilya Mironov.
\newblock R{\'e}nyi differential privacy.
\newblock In {\em CSF}, pages 263--275. IEEE, 2017.

\bibitem{panigrahi2019non}
Abhishek Panigrahi, Raghav Somani, Navin Goyal, and Praneeth Netrapalli.
\newblock Non-gaussianity of stochastic gradient noise.
\newblock {\em arXiv preprint arXiv:1910.09626}, 2019.

\bibitem{park2021influence}
Seulki Park, Jongin Lim, Younghan Jeon, and Jin~Young Choi.
\newblock Influence-balanced loss for imbalanced visual classification.
\newblock In {\em ICCV}, pages 735--744, 2021.

\bibitem{pichapati2019adaclip}
Venkatadheeraj Pichapati, Ananda~Theertha Suresh, Felix~X Yu, Sashank~J Reddi, and Sanjiv Kumar.
\newblock Adaclip: Adaptive clipping for private sgd.
\newblock {\em arXiv preprint arXiv:1908.07643}, 2019.

\bibitem{sha2023pcdp}
Haichao Sha, Ruixuan Liu, Yixuan Liu, and Hong Chen.
\newblock Pcdp-sgd: Improving the convergence of differentially private sgd via projection in advance.
\newblock {\em arXiv preprint arXiv:2312.03792}, 2023.

\bibitem{simsekli2019tail}
Umut Simsekli, Levent Sagun, and Mert Gurbuzbalaban.
\newblock A tail-index analysis of stochastic gradient noise in deep neural networks.
\newblock In {\em ICML}, pages 5827--5837. PMLR, 2019.

\bibitem{simsekli2020fractional}
Umut Simsekli, Lingjiong Zhu, Yee~Whye Teh, and Mert Gurbuzbalaban.
\newblock Fractional underdamped langevin dynamics: Retargeting sgd with momentum under heavy-tailed gradient noise.
\newblock In {\em International conference on machine learning}, pages 8970--8980. PMLR, 2020.

\bibitem{song2022sketching}
Zhao Song, Yitan Wang, Zheng Yu, and Lichen Zhang.
\newblock Sketching for first order method: Efficient algorithm for low-bandwidth channel and vulnerability.
\newblock {\em arXiv preprint arXiv:2210.08371}, 2022.

\bibitem{tang2024differentially}
Xinyu Tang, Ashwinee Panda, Vikash Sehwag, and Prateek Mittal.
\newblock Differentially private image classification by learning priors from random processes.
\newblock {\em NeurIPS}, 36, 2024.

\bibitem{tramer2020differentially}
Florian Tramer and Dan Boneh.
\newblock Differentially private learning needs better features (or much more data).
\newblock {\em arXiv preprint arXiv:2011.11660}, 2020.

\bibitem{vershynin2018high}
Roman Vershynin.
\newblock {\em High-dimensional probability: An introduction with applications in data science}, volume~47.
\newblock Cambridge university press, 2018.

\bibitem{vladimirova2020sub}
Mariia Vladimirova, St{\'e}phane Girard, Hien Nguyen, and Julyan Arbel.
\newblock Sub-weibull distributions: Generalizing sub-gaussian and sub-exponential properties to heavier tailed distributions.
\newblock {\em Stat}, 9(1):e318, 2020.

\bibitem{wainwright2019high}
Martin~J Wainwright.
\newblock {\em High-dimensional statistics: A non-asymptotic viewpoint}, volume~48.
\newblock Cambridge university press, 2019.

\bibitem{wang2020differentially}
Di~Wang, Hanshen Xiao, Srinivas Devadas, and Jinhui Xu.
\newblock On differentially private stochastic convex optimization with heavy-tailed data.
\newblock In {\em ICML}, pages 10081--10091. PMLR, 2020.

\bibitem{wang2021convergence}
Hongjian Wang, Mert Gurbuzbalaban, Lingjiong Zhu, Umut Simsekli, and Murat~A Erdogdu.
\newblock Convergence rates of stochastic gradient descent under infinite noise variance.
\newblock {\em NeurIPS}, 34:18866--18877, 2021.

\bibitem{wei2022dpis}
Jianxin Wei, Ergute Bao, Xiaokui Xiao, and Yin Yang.
\newblock Dpis: An enhanced mechanism for differentially private sgd with importance sampling.
\newblock In {\em SIGSAC}, pages 2885--2899, 2022.

\bibitem{xia2023differentially}
Tianyu Xia, Shuheng Shen, Su~Yao, Xinyi Fu, Ke~Xu, Xiaolong Xu, and Xing Fu.
\newblock Differentially private learning with per-sample adaptive clipping.
\newblock In {\em AAAI}, volume~37, pages 10444--10452, 2023.

\bibitem{xiao2023theory}
Hanshen Xiao, Zihang Xiang, Di~Wang, and Srinivas Devadas.
\newblock A theory to instruct differentially-private learning via clipping bias reduction.
\newblock In {\em SP}, pages 2170--2189. IEEE, 2023.

\bibitem{xie2017aggregated}
Saining Xie, Ross Girshick, Piotr Doll{\'a}r, Zhuowen Tu, and Kaiming He.
\newblock Aggregated residual transformations for deep neural networks.
\newblock In {\em CVPR}, pages 1492--1500, 2017.

\bibitem{yang2022normalized}
Xiaodong Yang, Huishuai Zhang, Wei Chen, and Tie-Yan Liu.
\newblock Normalized/clipped sgd with perturbation for differentially private non-convex optimization.
\newblock {\em arXiv preprint arXiv:2206.13033}, 2022.

\bibitem{yi2021efficient}
Jingwei Yi, Fangzhao Wu, Chuhan Wu, Ruixuan Liu, Guangzhong Sun, and Xing Xie.
\newblock Efficient-fedrec: Efficient federated learning framework for privacy-preserving news recommendation.
\newblock {\em arXiv preprint arXiv:2109.05446}, 2021.

\bibitem{yu2021not}
Da~Yu, Huishuai Zhang, Wei Chen, and Tie-Yan Liu.
\newblock Do not let privacy overbill utility: Gradient embedding perturbation for private learning.
\newblock {\em arXiv preprint arXiv:2102.12677}, 2021.

\bibitem{yu2021large}
Da~Yu, Huishuai Zhang, Wei Chen, Jian Yin, and Tie-Yan Liu.
\newblock Large scale private learning via low-rank reparametrization.
\newblock In {\em ICML}, pages 12208--12218. PMLR, 2021.

\bibitem{zhang2019gradient}
Jingzhao Zhang, Tianxing He, Suvrit Sra, and Ali Jadbabaie.
\newblock Why gradient clipping accelerates training: A theoretical justification for adaptivity.
\newblock {\em ICLR}, 2020.

\bibitem{zhang2020adaptive}
Jingzhao Zhang, Sai~Praneeth Karimireddy, Andreas Veit, Seungyeon Kim, Sashank Reddi, Sanjiv Kumar, and Suvrit Sra.
\newblock Why are adaptive methods good for attention models?
\newblock {\em NeurIPS}, 33:15383--15393, 2020.

\bibitem{zhang2005data}
Tong Zhang.
\newblock Data dependent concentration bounds for sequential prediction algorithms.
\newblock In {\em COLT}, pages 173--187. Springer, 2005.

\bibitem{zhang2023differentially}
Xinwei Zhang, Zhiqi Bu, Steven Wu, and Mingyi Hong.
\newblock Differentially private sgd without clipping bias: An error-feedback approach.
\newblock In {\em ICLR}, 2023.

\bibitem{zhang2022understanding}
Xinwei Zhang, Xiangyi Chen, Mingyi Hong, Zhiwei~Steven Wu, and Jinfeng Yi.
\newblock Understanding clipping for federated learning: Convergence and client-level differential privacy.
\newblock In {\em ICML}, 2022.

\bibitem{zhou2020bypassing}
Yingxue Zhou, Steven Wu, and Arindam Banerjee.
\newblock Bypassing the ambient dimension: Private sgd with gradient subspace identification.
\newblock In {\em ICLR}, 2021.

\bibitem{zhu2023improving}
Junyi Zhu and Matthew~B Blaschko.
\newblock Improving differentially private sgd via randomly sparsified gradients.
\newblock {\em Transactions on Machine Learning Research}, 2023.

\bibitem{zhu2018anisotropic}
Zhanxing Zhu, Jingfeng Wu, Bing Yu, Lei Wu, and Jinwen Ma.
\newblock The anisotropic noise in stochastic gradient descent: Its behavior of escaping from minima and regularization effects.
\newblock 2018.

\bibitem{ziller2021medical}
Alexander Ziller, Dmitrii Usynin, Rickmer Braren, Marcus Makowski, Daniel Rueckert, and Georgios Kaissis.
\newblock Medical imaging deep learning with differential privacy.
\newblock {\em Scientific Reports}, 11(1):13524, 2021.

\end{thebibliography}
}


\clearpage
\setcounter{page}{1}
\setcounter{theorem}{0}
\setcounter{theorem}{0}
\appendix

\section*{Appendix}
\section{Preliminaries}
A random variable $X$ called a sub-Weibull random variable with tail parameter $\theta$ and scale factor $K$, which is denoted by $X \sim subW(\theta,K)$. We next introduce the equivalent properties and theoretical tools of sub-Weibull distributions.
\subsection{Properties}
\begin{myDef}[Sub-Weibull Equivalent Properties~\cite{vladimirova2020sub}] Let $X$ be a random variable and $\theta \geq 0$, and there exists some constant $K_1,K_2,K_3,K_4$ depending on $\theta$. Then the following characterizations are equivalent:
\begin{enumerate}
    \item The tails of $X$ satisfy 
    \begin{align}
        \exists K_1>0~\text{such that}~\mathbb{P}(|X|>t)\leq 2\mathrm{exp}(-(t/K_1)^{\frac{1}{\theta}} ), \forall t>0. \nonumber
    \end{align}
    \item The moments of $X$ satisfy
    \begin{align}
        \exists K_2>0~\text{such that}~\Vert X\Vert_p \leq K_2p^{\theta}, \forall k \geq 1. \nonumber
    \end{align}
    \item The moment generating function~(MGF) of $|X|^{\frac{1}{\theta}}$ satisfies
    \begin{align}
        \exists K_3>0~\text{such that}~\mathbb{E}[\mathrm{exp}((\lambda|X|)^{\frac{1}{\theta}})] \leq \mathrm{exp}((\lambda K_3)^{\frac{1}{\theta}}), \forall \lambda\in(0,1/K_3). \nonumber
    \end{align}
    \item The MGF of $|X|^{\frac{1}{\theta}}$ is bounded at some point,
    \begin{align}
        \exists K_4>0~\text{such that}~\mathbb{E}[\mathrm{exp}((|X|/K_4)^{\frac{1}{\theta}})] \leq 2. \nonumber 
    \end{align}
\end{enumerate}
\label{myDef:1}
\end{myDef}

\begin{fact}
For any $V_k \in \mathbb{R}^{d\times k}$, $\mathrm{tr}(V^T_k\nabla \ell\nabla \ell^T V_k)= \Vert V^T_k \nabla \ell\Vert^2_2$. Moreover, if the condition $V^T_kV_k = \mathbb{I}$ holds, then $\Vert V^T_k \nabla \ell\Vert^2_2 = \Vert V_kV^T_k \nabla \ell\Vert^2_2$.
\label{fact:1}
\end{fact}

\subsection{Theoretical tools}
Based on the properties of sub-Weibull variables, we have the following high probability bounds and concentration inequalities for heavier tails as theoretical tools. Besides, We define $l_p$ norm as $\Vert \Vert_p$, for any $p\geq 1$.
\begin{lemma}
\label{lemma:1} 
    Let a variable $X \sim subW(\theta,K)$, for any $\delta \in (0,1)$, then with probability $(1-\delta)$ we have 
    \begin{align}
        |X| \leq K\log^{\theta}{(2/\delta)}. \nonumber
    \end{align}   
\end{lemma}
\begin{proof}
    Let $K_1 = K$ in Definition~\ref{myDef:1}, and take $t=K\log^{\theta}{(2/\delta)}$, then the inequality holds with probability $1-\delta$.
\end{proof}
\begin{lemma}[\cite{vladimirova2020sub,madden2020high}]
    Let $X_1,...,X_n$ are $subW(\theta,K_i)$ random variables with scale parameters $K_1,...K_n$. $\forall t \geq 0$, we have
    \begin{align}
        \mathbb{P}(|\sum^n_{i=1}X_i|\geq t) \leq 2\mathrm{exp}(-(\frac{t}{g(\theta)\sum^n_{i=1}K_i} )^{\frac{1}{\theta}} ) \nonumber
    \end{align}
    where $g(\theta) = (4e)^{\theta}$ for $\theta \leq 1$ and $g(\theta) = 2(2e\theta)^{\theta}$ for $\theta \geq 1$.
\label{lemma:2}
\end{lemma}

\begin{lemma}[Sub-Weibull Freedman Inequality~\cite{madden2020high}] Let $(\Omega,\mathcal{F},(\mathcal{F}_i),\mathbb{P})$ be a filtered probability space. Let $(\xi_i)$ and $(K_i)$ be adapted to $(\mathcal{F}_i)$. Let $n \in \mathbb{N}$, then $\forall i \in [n]$, assume $K_{i-1} \geq 0$, $\mathbb{E}[\xi_i|\mathcal{F}_{i-1}]=0$, and $\mathbb{E}[\mathrm{exp}((|\xi_i|/K_{i-1} )^{\frac{1}{\theta}} )|\mathcal{F}_{i-1}] \leq 2$ where $\theta \geq 1/2.$ If $\theta > 1/2$, assume there exists $(m_i)$ such that $K_{i-1}\leq m_i$.

if $\theta = 1/2$, let $a=2$, then $\forall x,\beta \geq 0$, $\alpha > 0$, and $\lambda\in [0,\frac{1}{2\alpha}]$, 
\begin{align}
    \mathbb{P}\left( \bigcup_{k\in[n]}\Big\{\sum^k_{i=1}\xi_i\geq x~\mathrm{and}~\sum^k_{i=1}aK^2_{i-1}\leq \alpha\sum^k_{i=1}\xi_i+\beta \Big\} \right) \leq \mathrm{exp}(-\lambda x + 2\lambda^2\beta),
\end{align}
and $\forall x,\beta,\lambda\geq0$,
\begin{align}
    \mathbb{P}\left( \bigcup_{k\in[n]}\Big\{\sum^k_{i=1}\xi_i\geq x~\mathrm{and}~\sum^k_{i=1}aK^2_{i-1}\leq \beta \Big\} \right) \leq \mathrm{exp}(-\lambda x + \frac{\lambda^2}{2}\beta).
\end{align}
If $\theta\in(\frac{1}{2},1]$, let $a=(4\theta)^{2\theta}e^2$ and $b=(4\theta)^{\theta}e$. $\forall x,\beta\geq 0$, and $\alpha\geq b\mathrm{max}_{i\in[n]}m_i$, and $\lambda\in[0,\frac{1}{2\alpha}]$,
\begin{align}
    \mathbb{P}\left( \bigcup_{k\in[n]}\Big\{\sum^k_{i=1}\xi_i\geq x~\mathrm{and}~\sum^k_{i=1}aK^2_{i-1}\leq \alpha\sum^k_{i=1}\xi_i+\beta \Big\} \right) \leq \mathrm{exp}(-\lambda x + 2\lambda^2\beta),
\end{align}
and $\forall x,\beta\geq 0$, and $\lambda\in[0,\frac{1}{b\mathrm{max}_{i\in[n]}m_i}]$,
\begin{align}
    \mathbb{P}\left( \bigcup_{k\in[n]}\Big\{\sum^k_{i=1}\xi_i\geq x~\mathrm{and}~\sum^k_{i=1}aK^2_{i-1}\leq \beta \Big\} \right) \leq \mathrm{exp}(-\lambda x + \frac{\lambda^2}{2}\beta).
\end{align}
 If $\theta > 1$, let $\delta\in(0,1)$. Let $a=(2^{2\theta+1}+2 )\Gamma(2\theta+1) + 2^{3\theta}\Gamma(3\theta+1)/3$ and $b=2\log{n/\delta}^{\theta-1}$, where $\Gamma(x) = \int^{\infty}_0t^{x-1}e^{-t}dt$. $\forall x,\beta\geq0$, $\alpha\geq b\mathrm{max}_{i\in[n]}m_i$, and $\lambda\in[0,\frac{1}{2\alpha}]$,
 \begin{align}
     \mathbb{P}\left( \bigcup_{k\in[n]}\Big\{\sum^k_{i=1}\xi_i\geq x~\mathrm{and}~\sum^k_{i=1}aK^2_{i-1}\leq \alpha\sum^k_{i=1}\xi_i+\beta \Big\} \right) \leq \mathrm{exp}(-\lambda x + 2\lambda^2\beta)+2\delta,
 \end{align}
 and $\forall x,\beta\geq0$, and $\lambda\in[0,\frac{1}{b\mathrm{max}_{i\in[n]}m_i}]$,
 \begin{align}
    \mathbb{P}\left( \bigcup_{k\in[n]}\Big\{\sum^k_{i=1}\xi_i\geq x~\mathrm{and}~\sum^k_{i=1}aK^2_{i-1}\leq \beta \Big\} \right) \leq \mathrm{exp}(-\lambda x + \frac{\lambda^2}{2}\beta)+2\delta.
\end{align}
\label{lemma:3}
\end{lemma}

\begin{lemma}[\cite{zhang2005data}]
Let $z_1,...,z_n$ be a sequence of randoms variables such that $z_k$ may depend the previous variables $z_1,...,z_{k-1}$ for all $k=1,...,n$. Consider a sequence of functionals $\xi_k(z_1,...,z_k)$, $k=1,...,n$. Let $\sigma^2_n=\sum^n_{k=1}\mathbb{E}_{z_k}[(\xi_k-\mathbb{E}_{z_k}[\xi_k])^2]$ be the conditional variance. Assume $|\xi_k - \mathbb{E}_{z_k}[\xi_k]|\leq b$ for each $k$. Let $\rho \in (0,1]$ and $\delta\in(0,1)$. With probability at least $1-\delta$ we have
\begin{align}
    \sum^n_{k=1}\xi_k-\sum^n_{k=1}\mathbb{E}_{z_k}[\xi_k] \leq \frac{\rho\sigma^2_n}{b} + \frac{b\log{\frac{1}{\delta}}}{\rho}.
\end{align}
\label{lemma:4}
\end{lemma}

\begin{lemma}[\cite{cutkosky2020momentum}]
For any vector $\mathbf{g}\in\mathbb{R}^d$, $\langle \mathbf{g}/\Vert \mathbf{g}\Vert_2,\nabla L_S(\mathbf{w})\rangle \geq \frac{\Vert\nabla L_S(\mathbf{w})\Vert_2}{3} - \frac{8\Vert \mathbf{g}-L_S(\mathbf{w})\Vert_2}{3}$.
\label{lemma:5}
\end{lemma}

\begin{lemma}[\cite{madden2020high}]
If $X \sim subW(\theta,K)$, then $\mathbb{E}[|X^p |] \leq 2\Gamma(p\theta +1)K^p~\forall p>0$. In particular, $\mathbb{E}[X^2]\leq 2\Gamma(2\theta +1)K^2$.
\label{lemma:6}
\end{lemma}

\begin{lemma}[\cite{bakhshizadeh2023sharp}]
Suppose $X_1,...,X_m\overset{d}{=} X$ are independent and identically distributed random variables whose right tails are captured by an increasing and continuous function $I:\mathbb{R}\rightarrow\mathbb{R}^{\geq0}$ with the property $I(t) = \mathbb{O}(t)$ as $t\rightarrow \infty$. Let $X^L = X\mathbb{I}(X\leq L)$, $S_m = \sum^m_{i=1}X_i$ and $Z^L := X^L - \mathbb{E}[X]$. Define $t_{\mathrm{max}}(\mu):=\mathrm{sup}\{t\geq0: t\leq\mu v(mt,\mu)\frac{I(mt)}{mt}\}$, then \\
\begin{align}
    \mathbb{P}(S_m-\mathbb{E}[S_m] > mt) \leq \left\{ \begin{array}{ll} \mathrm{exp}(-c_t\mu I(mt)) + m\mathrm{exp}(-I(mt)) , & \text{if } t\geq t_{\mathrm{max}}(\mu), \\ \\ \displaystyle\mathrm{exp}(-\frac{mt^2}{2v(mt_{\mathrm{max}}(\mu),\mu)}) + m\mathrm{exp}(-\frac{mt^2_{\mathrm{max}}(\mu)}{\mu v(mt_{\mathrm{max}}(\mu),\mu)}), & \text{if } 0\leq t\leq t_{\mathrm{max}}(\mu), \end{array} \right.
\end{align}
where $c_t = 1- \frac{\mu v(mt,\mu)I(mt)}{2mt^2} $ and $v(L,\mu)=\mathbb{E}\big[(Z^L)^2\mathbb{I}(Z^L\leq0) + (Z^L)^2~\mathrm{exp}(\mu\frac{I(L)}{L}Z^L)\mathbb{I}(Z^L>0)\big], \forall\beta\in(0,1]$.
\label{lemma:7}
\end{lemma}

\begin{lemma}[\cite{bakhshizadeh2023sharp}]
Consider the same settings as the ones in Lemma~\ref{lemma:7}. Assume $\mathbb{E}[X_i]=0$, then $\forall t\geq0$ we have
\begin{align}
    \mathbb{P}(S_m > mt) \leq \mathrm{exp}(-\frac{mt^2}{2v(mt,\mu)}) + \mathrm{exp}(-\mu \max\{c_t,\frac{1}{2}\} I(mt)) + m\mathrm{exp}(-I(mt)).
\end{align}  
\label{lemma:8}
\end{lemma}

\begin{lemma}[\textbf{Ahlswede-Winter Inequality}~\cite{wainwright2019high}]
\label{thm:Ahlswede-Winter Inequality}
Let $Y$ be a random, symmetric, positive semi-definite $d × d$ matrix such that $\Vert \mathbb{E}[Y] \Vert_2 \leq 1$. Suppose $\Vert Y \Vert_2 \leq R $ for some fixed scalar $R \geq 1$. Let $Y_1, . . ., Y_m$ be independent copies of $Y$ (i.e.,
independently sampled matrix with the same distribution as $Y$). For any $\mu\in(0, 1)$, we have
\begin{align}
\mathbb{P}(\Vert \frac{1}{m}\sum_{i=1}^mY_i - \mathbb{E}[Y_i] \Vert_{2} > \mu) \leq 2d \cdot \text{exp}(-m\mu^2/4R). \nonumber
\label{lemma:9}
\end{align}
\end{lemma}

\clearpage
\section{Convergence of Heavy-tailed DPSGD}
\begin{algorithm}[h!]
\caption{Outline of DPSGD~\cite{abadi2016deep}}
\label{alg:algorithm2}
\textbf{Input}: Samples $n$, Private batch size $B$, clipping threshold $c$, learning rate $\eta_t$ and noise scale $\sigma$.
\begin{algorithmic}[1] 
\STATE Initialize $\mathbf{w}_0$ randomly.
\FOR{$e \in E$}
\FOR{$t \in T$}
\STATE Take a random batch $B$ with sampling ratio $B/n$ and $g_t(z_i) = \nabla\ell(\mathbf{w}_t,z_i)$.
\STATE Clip per-sample gradient. 
\\$\overline{g}_t(z_i) = g_t(z_i) / \mathrm{max}(1, \frac{\Vert g_t(z_i)\Vert_2}{c})$ . 
\STATE Add noise and average.\\
$\widetilde{g}_t = \frac{1}{B}(\sum^B_{i=1}\overline{g}_t(z_i)  + \mathbb{N}(0,c^2\sigma^2\mathbb{I})) $.
\STATE Update $\mathbf{w}_{t+1} = \mathbf{w}_t - \eta_t\widetilde{g}_t$.
\ENDFOR
\ENDFOR
\end{algorithmic}
\end{algorithm}

\begin{theorem}[\textbf{Convergence of Heavy-tailed DPSGD}]
Under Assumption A.1, let $\mathbf{w}_{t}$ be the iterate produced by Algorithm~\ref{alg:algorithm2}-DPSGD and $\eta_t = \frac{1}{\sqrt{T}}$. If $\theta = \frac{1}{2}$ and $19K\log^{\theta}(2/\delta) \leq 12\sqrt{e}\sigma_{\mathrm{dp}}\log^{\frac{1}{2}}(2/\delta)$, then $T= \max{\big(m_2eB^2\log(1/\delta), \frac{n\epsilon}{\sqrt{d\log(1/\delta)}}\big)}$ and $c=\max{\big(4K\log^{\theta}({\sqrt{T}}),27\sqrt{e}\sigma_{\mathrm{dp}}\log^{\frac{1}{2}}(2/\delta)\big)}$. If $\theta = \frac{1}{2}$ and $19K\log^{\theta}(2/\delta) \geq 12\sqrt{e}\sigma_{\mathrm{dp}}\log^{\frac{1}{2}}(2/\delta)$, then $T= \frac{n\epsilon}{\sqrt{d\log(1/\delta)}}$ and $c=\max{\big(4K\log^{\theta}({\sqrt{T}}),39K\log^{\frac{1}{2}}(2/\delta)\big)}$. If $\theta > \frac{1}{2}$, then  $T= \frac{n\epsilon}{\sqrt{d\log(1/\delta)}}$ and $c=\max{\big(4K\log^{\theta}({\sqrt{T}}),20K\log^{\theta}(2/\delta)\big)}$. For any $\delta \in (0,1)$, with probability $1-\delta$, we have
\begin{align}
     \frac{1}{T}\sum^T_{t=1} \min\big\{\Vert\nabla L_S(\mathbf{w}_{t})\Vert_2, \Vert\nabla L_S(\mathbf{w}_{t})\Vert^2_2\big\} \leq \mathbb{O}\left(\frac{d^{\frac{1}{4}} \log^{\frac{5}{4}}(T/\delta)\hat{\log}(T/\delta)\log^{2\theta}(\sqrt{T}) }{(n\epsilon)^\frac{1}{2}}\right) , \nonumber
\end{align}
where $\hat{\log}(T/\delta) = \log^{\max(0,\theta-1)}(T/\delta)$. 
\end{theorem}

\begin{proof}
We consider two cases: $L_S(\mathbf{w}_{t})\leq c/2$ and $L_S(\mathbf{w}_{t})\geq c/2$. 

We first consider the case $\nabla L_S(\mathbf{w}_{t}) \leq c/2$ with Assumption~\ref{ass:smooth}.
\begin{align}
    &L_S(\mathbf{w}_{t+1}) - L_S(\mathbf{w}_{t}) \leq \langle\mathbf{w}_{t+1}-\mathbf{w}_{t}, \nabla L_S(\mathbf{w}_{t})\rangle + \frac{1}{2}\beta\Vert\mathbf{w}_{t+1}-\mathbf{w}_{t}\Vert^2  \\
    &\leq -\eta_{t}\langle \overline{\mathbf{g}}_t + \mathbf{\zeta}_{t}, \nabla L_S(\mathbf{w}_{t})\rangle + \frac{1}{2}\beta\eta^2_{t}\Vert \overline{\mathbf{g}}_t + \mathbf{\zeta}_{t}\Vert^2 \nonumber \\
    &= -\eta_{t}\langle \overline{\mathbf{g}}_t - \mathbb{E}_t[\overline{\mathbf{g}}_t] + \mathbb{E}_t[\overline{\mathbf{g}}_t] - \nabla L_S(\mathbf{w}_{t}), \nabla L_S(\mathbf{w}_{t}) \rangle -\eta_t\langle\mathbf{\zeta}_{t},\nabla L_S(\mathbf{w}_{t})\rangle \nonumber \\
    &- \eta_t\Vert \nabla L_S(\mathbf{w}_{t})\Vert^2 + \frac{1}{2}\beta\eta^2_{t}\Vert \overline{\mathbf{g}}_t\Vert^2 + \frac{1}{2}\beta\eta^2_{t}\Vert \mathbf{\zeta}_{t}\Vert^2 + \beta\eta^2_{t}\langle \overline{\mathbf{g}}_t, \mathbf{\zeta}_{t}\rangle \nonumber \\
    &= -\eta_{t}\langle \overline{\mathbf{g}}_t - \mathbb{E}_t[\overline{\mathbf{g}}_t], \nabla L_S(\mathbf{w}_{t}) \rangle -\eta_{t}\langle \mathbb{E}_t[\overline{\mathbf{g}}_t] - \nabla L_S(\mathbf{w}_{t}), \nabla L_S(\mathbf{w}_{t}) \rangle - \eta_t\langle\mathbf{\zeta}_{t},\nabla L_S(\mathbf{w}_{t})\rangle \nonumber \\
    &- \eta_t\Vert \nabla L_S(\mathbf{w}_{t})\Vert^2 + \frac{1}{2}\beta\eta^2_{t}\Vert \overline{\mathbf{g}}_t\Vert^2 + \frac{1}{2}\beta\eta^2_{t}\Vert \mathbf{\zeta}_{t}\Vert^2 + \beta\eta^2_{t}\langle \overline{\mathbf{g}}_t, \mathbf{\zeta}_{t}\rangle \nonumber
\end{align}
Considering all $T$ iterations, we get
\begin{align}
    &\sum^T_{t=1}\eta_{t}\Vert\nabla L_S(\mathbf{w}_{t}) \Vert^2 \leq L_S(\mathbf{w}_{1}) - L_S(\mathbf{w}_{S}) + \sum^T_{t=1}\frac{1}{2}\beta\eta^2_{t}c^2 + \underbrace{\sum^T_{t=1}\frac{1}{2}\beta\eta^2_{t}\Vert\mathbf{\zeta}_{t} \Vert^2}_{\mathrm{E}.1} + \underbrace{\sum^T_{t=1}\beta\eta^2_{t}\langle \overline{\mathbf{g}}_t, \mathbf{\zeta}_{t}\rangle}_{\mathrm{E}.2} \nonumber \\
    &-\underbrace{\sum^T_{t=1}\eta_t\langle\mathbf{\zeta}_{t},\nabla L_S(\mathbf{w}_{t})\rangle}_{\mathrm{E}.3} -\underbrace{\sum^T_{t=1}\eta_{t}\langle \overline{\mathbf{g}}_t - \mathbb{E}_t[\overline{\mathbf{g}}_t], \nabla L_S(\mathbf{w}_{t}) \rangle}_{\mathrm{E}.4} -\underbrace{\sum^T_{t=1}\eta_{t}\langle \mathbb{E}_t[\overline{\mathbf{g}}_t] - \nabla L_S(\mathbf{w}_{t}), \nabla L_S(\mathbf{w}_{t}) \rangle}_{\mathrm{E}.5}
\end{align}
For E.1, E.2 and E.3, since $\zeta_t \sim \mathbb{N}(0,c^2\sigma^2_{\mathrm{dp}}\mathbb{I}_d)$, we set $\sigma^2_{\mathrm{dp}}=m_2\frac{TdB^2\log(1/\delta)}{n^2\epsilon^2}$ for simplicity, with sub-Gaussian properties~\ref{myDef:1} and Lemma~\ref{lemma:2}, with probability at least $1-\delta$, and we have
\begin{align}
    \sum^T_{t=1}\frac{1}{2}\beta\eta^2_{t}\Vert\mathbf{\zeta}_{t} \Vert^2 &\leq 2\beta K^2e\log(2/\delta)\sum^T_{t=1}\eta^2_{t} \nonumber \\
    &\leq 2\beta m_2 e d\frac{Tc^2B^2\log^2(2/\delta)}{n^2\epsilon^2}\sum^T_{t=1}\eta^2_{t} .
\end{align}
Also, with probability at least $1-\delta$, we get
\begin{align}
    \sum^T_{t=1}\beta\eta^2_{t}\langle \overline{\mathbf{g}}_t, \mathbf{\zeta}_{t}\rangle &\leq \sum^T_{t=1}\beta\eta^2_{t}\Vert \overline{\mathbf{g}}_t\Vert\Vert\mathbf{\zeta}_{t}\Vert \nonumber\\
    &\leq \sum^T_{t=1}2\beta c K\sqrt{e}\log^{\frac{1}{2}}(2/\delta)\eta^2_{t} \nonumber \\
    &\leq 2\beta \sqrt{em_2Td}\frac{c^2B\log(2/\delta)}{n\epsilon}\sum^T_{t=1}\eta^2_{t}.
    \label{formula:14}
\end{align}
Due to $\nabla L_S(\mathbf{w}_{t}) \leq c/2$, for the term $-\sum^T_{t=1}\eta_t\langle\mathbf{\zeta}_{t},\nabla L_S(\mathbf{w}_{t})\rangle$, with probability at least $1-\delta$, we have
\begin{align}
  -\sum^T_{t=1}\eta_t\langle\mathbf{\zeta}_{t},\nabla L_S(\mathbf{w}_{t})\rangle &\leq \sum^T_{t=1}\eta_t\Vert \zeta_t\Vert\Vert \nabla L_S(\mathbf{w}_{t})\Vert \nonumber \\
  &\leq \sum^T_{t=1}2 c K\sqrt{e}\log^{\frac{1}{2}}(2/\delta)\eta_t \nonumber \\
  &\leq 2\sqrt{em_2Td}\frac{c^2B\log(2/\delta)}{n\epsilon}\sum^T_{t=1}\eta_t.
\end{align}
Since $\mathbb{E}_t[-\eta_t\langle \overline{\mathbf{g}}_t - \mathbb{E}_t[\overline{\mathbf{g}}_t], \nabla L_S(\mathbf{w}_{t}) \rangle]=0$, the sequence $(-\eta_t\langle \overline{\mathbf{g}}_t - \mathbb{E}_t[\overline{\mathbf{g}}_t], \nabla L_S(\mathbf{w}_{t}) \rangle, t\in \mathbb{N})$ is a martingale difference sequence. Applying Lemma~\ref{lemma:4}, we define $\xi_t = -\eta_t\langle \overline{\mathbf{g}}_t - \mathbb{E}_t[\overline{\mathbf{g}}_t], \nabla L_S(\mathbf{w}_{t}) \rangle$ and have
\begin{align}
    |\xi_t|\leq \eta_t(\Vert\overline{\mathbf{g}} \Vert_2 + \Vert\mathbb{E}_t[\overline{\mathbf{g}}] \Vert_2) \Vert \nabla L_S(\mathbf{w}_{t})\Vert_2 \leq \eta_t c^2.
\end{align}
Applying $\mathbb{E}_t[(\xi_t-\mathbb{E}_t\xi_t)^2]\leq\mathbb{E}_t[\xi^2_t]$, we have
\begin{align}
    \sum^T_{t=1}\mathbb{E}_t[(\xi_t-\mathbb{E}_t\xi_t)^2]&\leq \sum^T_{t=1}\eta^2_t\mathbb{E}_t[\Vert\overline{\mathbf{g}} - \mathbb{E}_t[\overline{\mathbf{g}}] \Vert^2_2 \Vert \nabla L_S(\mathbf{w}_{t})\Vert^2_2] \nonumber \\
    &\leq 4c^2\sum^T_{t=1}\eta^2_t\Vert \nabla L_S(\mathbf{w}_{t})\Vert^2_2.
\end{align}
Then, with probability $1-\delta$, we obtain
\begin{align}
    \sum^T_{t=1}\xi_t\leq \frac{\rho4c^2\sum^T_{t=1}\eta^2_t\Vert \nabla L_S(\mathbf{w}_{t})\Vert^2_2}{\eta_tc^2} + \frac{\eta_tc^2\log{(1/\delta)}}{\rho}.
\end{align}
Next, to bound term E.5, we have
\begin{align}
    \sum^T_{t=1}\eta_{t}\langle \mathbb{E}_t[\overline{\mathbf{g}}_t] - \nabla L_S(\mathbf{w}_{t}), \nabla L_S(\mathbf{w}_{t}) \rangle \leq \frac{1}{2}\sum^T_{t=1}\eta_{t}\Vert \mathbb{E}_t[\overline{\mathbf{g}}_t]-\nabla L_S(\mathbf{w}_{t})\Vert^2_2 + \frac{1}{2}\sum^T_{t=1}\eta_{t}\Vert \nabla L_S(\mathbf{w}_{t})\Vert^2_2. \nonumber 
\end{align}
Setting $a_t=\mathbb{I}_{\Vert\mathbf{g}_t \Vert_2>c}$ and $b_t=\mathbb{I}_{\Vert\mathbf{g}_t -\nabla L_S(\mathbf{w}_{t})\Vert_2>\frac{c}{2}}$, for term $\Vert \mathbb{E}_t[\overline{\mathbf{g}}_t]-\nabla L_S(\mathbf{w}_{t})\Vert_2$, we have
\begin{align}
    \Vert \mathbb{E}_t[\overline{\mathbf{g}}_t]-\nabla L_S(\mathbf{w}_{t})\Vert_2 &= \Vert \mathbb{E}_t[(\overline{\mathbf{g}}_t - \mathbf{g}_t)a_t]\Vert_2 \nonumber \\
    &=  \Vert \mathbb{E}_t[(\mathbf{g}_t (\frac{c}{\Vert\mathbf{g}_t \Vert_2}-1)a_t]\Vert_2 \nonumber \\
    &\leq  \mathbb{E}_t[\Vert(\mathbf{g}_t (\frac{c}{\Vert\mathbf{g}_t \Vert_2}-1)a_t\Vert_2] \nonumber \\
    &\leq  \mathbb{E}_t[|\Vert\mathbf{g}_t\Vert_2 - c|a_t] \nonumber \\
    &\leq  \mathbb{E}_t[|\Vert\mathbf{g}_t\Vert_2 - \Vert\nabla L_S(\mathbf{w}_{t})\Vert_2|a_t] \nonumber \\
    &\leq  \mathbb{E}_t[|\Vert\mathbf{g}_t - \nabla L_S(\mathbf{w}_{t})\Vert_2|a_t] \nonumber \\
    &\leq  \mathbb{E}_t[|\Vert\mathbf{g}_t - \nabla L_S(\mathbf{w}_{t})\Vert_2|b_t] \nonumber \\
    &\leq  \sqrt{\mathbb{E}_t[\Vert\mathbf{g}_t - \nabla L_S(\mathbf{w}_{t})\Vert^2_2]\mathbb{E}_tb^2_t} .
\end{align}
Applying Lemma~\ref{lemma:6}, we get $\mathbb{E}_t[\Vert\mathbf{g}_t - \nabla L_S(\mathbf{w}_{t})\Vert^2_2] \leq 2K^2\Gamma(2\theta+1)$.
Then, for term $\mathbb{E}_tb^2_t$, with sub-Weibull properties and probability $1-\delta$ we have
\begin{align}
    \mathbb{E}_tb^2_t = \mathbb{P}(\Vert\mathbf{g}_t -\nabla L_S(\mathbf{w}_{t})\Vert_2>\frac{c}{2})\leq 2\mathrm{exp}(-(\frac{c}{4K})^{\frac{1}{\theta}} )
\end{align}
So, we get formula.(18) as
\begin{align}
    \sqrt{\mathbb{E}_t[\Vert\mathbf{g}_t - \nabla L_S(\mathbf{w}_{t})\Vert^2_2]\mathbb{E}_tb^2_t} \leq 2\sqrt{K^2\Gamma(2\theta+1)\mathrm{exp}(-(\frac{c}{4K})^{\frac{1}{\theta}} ) }.
\end{align}
Thus, for E.5, with probability $1-T\delta$ we finally obtain
\begin{align}
    &\sum^T_{t=1}\eta_{t}\langle \mathbb{E}_t[\overline{\mathbf{g}}_t] - \nabla L_S(\mathbf{w}_{t}), \nabla L_S(\mathbf{w}_{t}) \rangle \nonumber \\
    &\leq 2K^2\Gamma(2\theta+1)\sum^T_{t=1}\eta_t\mathrm{exp}(-(\frac{c}{4K})^{\frac{1}{\theta}} ) + \frac{1}{2}\sum^T_{t=1}\eta_{t}\Vert \nabla L_S(\mathbf{w}_{t})\Vert^2_2.
\end{align}
Combining E.1-5 with the inequality~(10), with probability $1-4\delta-T\delta$, we have
\begin{align}
     &\sum^T_{t=1}\eta_{t}\Vert\nabla L_S(\mathbf{w}_{t}) \Vert^2_2 \leq L_S(\mathbf{w}_{1}) - L_S(\mathbf{w}_{S}) + \sum^T_{t=1}\frac{1}{2}\beta\eta^2_{t}c^2 +  2\beta m_2 e d\frac{Tc^2B^2\log^2(2/\delta)}{n^2\epsilon^2}\sum^T_{t=1}\eta^2_{t} \nonumber \\
     &+2\beta \sqrt{em_2Td}\frac{c^2B\log(2/\delta)}{n\epsilon}\sum^T_{t=1}\eta^2_{t} + 2\sqrt{em_2Td}\frac{c^2B\log(2/\delta)}{n\epsilon}\sum^T_{t=1}\eta_t + \frac{\eta_tc^2\log{(1/\delta)}}{\rho} \nonumber \\
     &+ \frac{4\rho c^2\sum^T_{t=1}\eta^2_t\Vert \nabla L_S(\mathbf{w}_{t})\Vert^2_2}{\eta_tc^2} + 2K^2\Gamma(2\theta+1)\mathrm{exp}(-(\frac{c}{4K})^{\frac{1}{\theta}} )\sum^T_{t=1}\eta_t + \frac{1}{2}\sum^T_{t=1}\eta_{t}\Vert \nabla L_S(\mathbf{w}_{t})\Vert^2_2.
\end{align}
Setting $\rho=\frac{1}{16}$, $T= \frac{n\epsilon}{\sqrt{d\log(1/\delta)}}$ and $\eta_t = \frac{1}{\sqrt{T}}$, we have
\begin{align}
    &\frac{1}{4}\sum^T_{t=1}\eta_{t}\Vert\nabla L_S(\mathbf{w}_{t}) \Vert^2_2 \leq L_S(\mathbf{w}_{1}) - L_S(\mathbf{w}_{S}) + \frac{1}{2}\beta c^2  +  2\beta m_2 e\frac{d^{\frac{1}{2}}c^2B^2\log^{\frac{3}{2}}(2/\delta)}{n\epsilon} \nonumber \\
     &+2\beta \sqrt{em_2}\frac{d^{\frac{1}{4}}c^2B\log^{\frac{1}{2}}(2/\delta)}{\sqrt{n\epsilon}} + 2\sqrt{em_2}c^2B\log^{\frac{1}{2}}(2/\delta) + \frac{16d^{\frac{1}{4}}c^2\log^{\frac{5}{4}}{(1/\delta)}}{\sqrt{n\epsilon}} \nonumber \\
     &+ \underbrace{2K^2\Gamma(2\theta+1)\mathrm{exp}(-(\frac{c}{4K})^{\frac{1}{\theta}} )\sqrt{T}}_{\mathrm{E}.6} .
\end{align}
Then, we pay attention to term E.6. If $c\rightarrow0$, then $\mathrm{exp}(-(\frac{c}{4K})^{\frac{1}{\theta}} )\rightarrow1$ and $\sqrt{T}$ will dominate term E.6. We know that in classical DPSGD, a small $c$ is regarded as the clipping threshold guide, which will cause the variance term E.6 to dominate the entire bound. For this, we will provide guidance on the clipping values of DPSGD under the heavy-tailed assumption.

Let $\mathrm{exp}(-(\frac{c}{4K})^{\frac{1}{\theta}} ) \leq \frac{1}{\sqrt{T}}$, then we have $c\geq 4K\log^{\theta}({\sqrt{T}})$. So, we obtain
\begin{align}
    &\sum^T_{t=1}\eta_{t}\Vert\nabla L_S(\mathbf{w}_{t}) \Vert^2_2 \leq 4(L_S(\mathbf{w}_{1}) - L_S(\mathbf{w}_{S}) ) + 2\beta c^2  +  8\beta m_2 e \frac{d^{\frac{1}{2}}c^2B^2\log^{\frac{3}{2}}(2/\delta)}{n\epsilon} \nonumber \\
     &+8\beta \sqrt{em_2}\frac{d^{\frac{1}{4}}c^2B\log^{\frac{1}{2}}(2/\delta)}{\sqrt{n\epsilon}} + 8\sqrt{em_2}c^2B\log^{\frac{1}{2}}(2/\delta) + \frac{64d^{\frac{1}{4}}c^2\log^{\frac{5}{4}}{(1/\delta)}}{\sqrt{n\epsilon}} + 8K^2\Gamma(2\theta+1) .
\end{align}
Multiplying $\frac{1}{\sqrt{T}}$ on both sides, we get
\begin{align}
    &\frac{1}{\sqrt{T}}\sum^T_{t=1}\eta_{t}\Vert\nabla L_S(\mathbf{w}_{t}) \Vert^2_2 \leq \frac{1}{\sqrt{T}} \left( 4(L_S(\mathbf{w}_{1}) - L_S(\mathbf{w}_{S}) ) + 2\beta c^2  +  8\beta m_2 e \frac{d^{\frac{1}{2}}c^2B^2\log^{\frac{3}{2}}(2/\delta)}{n\epsilon} \right. \nonumber \\
     &\left. +8\beta \sqrt{em_2}\frac{d^{\frac{1}{4}}c^2B\log^{\frac{1}{2}}(2/\delta)}{\sqrt{n\epsilon}} + 8\sqrt{em_2}c^2B\log^{\frac{1}{2}}(2/\delta) + \frac{64d^{\frac{1}{4}}c^2\log^{\frac{5}{4}}{(1/\delta)}}{\sqrt{n\epsilon}} + 8K^2\Gamma(2\theta+1)\right ). 
\end{align}
Taking $c=4K\log^{\theta}(\sqrt{T})$, due to $T\geq 1$, we achieve
\begin{align}
    \frac{1}{\sqrt{T}}\sum^T_{t=1}\eta_{t}\Vert\nabla L_S(\mathbf{w}_{t}) \Vert^2_2 &\leq \frac{4(L_S(\mathbf{w}_{1}) - L_S(\mathbf{w}_{S}) )}{\sqrt{T}} + \frac{8K^2\Gamma(2\theta+1)}{\sqrt{T}} \nonumber \\
    &+ \frac{16K^2\log^{2\theta}(\sqrt{T})\log(2/\delta)}{\sqrt{T}}\left( 2\beta + 8\beta m_2 e \frac{d^{\frac{1}{2}}B^2\log^{\frac{1}{2}}(2/\delta)}{n\epsilon} \right. \nonumber \\
    &\left. + 8\beta \sqrt{em_2}\frac{d^{\frac{1}{4}}B\log^{-\frac{1}{2}}(2/\delta)}{\sqrt{n\epsilon}}  + 8\sqrt{em_2}B\log^{-\frac{1}{2}}(2/\delta) + \frac{64d^{\frac{1}{4}}\log^{\frac{1}{4}}{(1/\delta)}}{\sqrt{n\epsilon}} \right) \nonumber \\
    &\leq \mathbb{O}\left(\frac{\log^{2\theta}(\sqrt{T})\log(1/\delta)}{\sqrt{T}}\cdot \frac{d^{\frac{1}{4}}\log^{\frac{1}{4}}{(1/\delta)}}{\sqrt{n\epsilon}} \right) \nonumber \\
    &\leq \mathbb{O}\left(\frac{\log^{2\theta}(\sqrt{T})\log(1/\delta)d^{\frac{1}{4}}\log^{\frac{1}{4}}{(1/\delta)}}{\sqrt{n\epsilon}}\right).
\end{align}
Due to $\frac{1}{T}\sum^T_{t=1}\Vert\nabla L_S(\mathbf{w}_{t}) \Vert^2_2 \leq \frac{1}{\sqrt{T}}\sum^T_{t=1}\eta_{t}\Vert\nabla L_S(\mathbf{w}_{t}) \Vert^2_2 $, we have
\begin{align}
    \frac{1}{T}\sum^T_{t=1}\Vert\nabla L_S(\mathbf{w}_{t}) \Vert^2_2 \leq \mathbb{O}\left(\frac{d^{\frac{1}{4}}\log^{2\theta}(\sqrt{T})\log^{\frac{5}{4}}(1/\delta) }{(n\epsilon)^\frac{1}{2}}\right),
\end{align}
with probability $1-T\delta -4\delta$.

By substitution, with probability $1-\delta$, we get
\begin{align}
    \frac{1}{T}\sum^T_{t=1}\Vert\nabla L_S(\mathbf{w}_{t}) \Vert^2_2 \leq \mathbb{O}\left(\frac{d^{\frac{1}{4}}\log^{2\theta}(\sqrt{T})\log^{\frac{5}{4}}(T/\delta) }{(n\epsilon)^\frac{1}{2}}\right).
\end{align}

Secondly, we consider the case $\nabla L_S(\mathbf{w}_{t}) \geq c/2$.
\begin{align}
    L_S(\mathbf{w}_{t+1}) - L_S(\mathbf{w}_{t}) &\leq \langle\mathbf{w}_{t+1}-\mathbf{w}_{t}, \nabla L_S(\mathbf{w}_{t})\rangle + \frac{1}{2}\beta\Vert\mathbf{w}_{t+1}-\mathbf{w}_{t}\Vert^2_2 \nonumber \\
    &\leq \underbrace{-\eta_t\langle \overline{\mathbf{g}}_t + \mathbf{\zeta}_{t}, \nabla L_S(\mathbf{w}_{t}) \rangle}_{\mathrm{E}.7} + \underbrace{\frac{1}{2}\beta\eta^2_t\Vert \overline{\mathbf{g}}_t + \mathbf{\zeta}_{t} \Vert^2_2}_{\mathrm{E}.8} 
\end{align}
We have discussed term E.8 in the above case, so we focus on E.7 here. Setting $s^{+}_t = \mathbb{I}_{\Vert\mathbf{g}_t\Vert_2\geq c}$ and $s^{-}_t = \mathbb{I}_{\Vert\mathbf{g}_t\Vert_2\leq c}$.
\begin{align}
    &-\eta_t\langle \overline{\mathbf{g}}_t + \mathbf{\zeta}_{t}, \nabla L_S(\mathbf{w}_{t}) \rangle \nonumber \\
    &=-\eta_t\langle \frac{c\mathbf{g}_t}{\Vert\mathbf{g}_t \Vert_2}s^{+}_t + \mathbf{g}_ts^{-}_t, \nabla L_S(\mathbf{w}_{t}) \rangle -\eta_t\langle \mathbf{\zeta}_{t}, \nabla L_S(\mathbf{w}_{t}) \rangle .
\end{align}
Applying Lemma~\ref{lemma:5} to term $-\eta_t\langle \frac{c\mathbf{g}_t}{\Vert\mathbf{g}_t \Vert_2}s^{+}_t, \nabla L_S(\mathbf{w}_{t}) \rangle$, we have
\begin{align}
    -\eta_t\langle \frac{c\mathbf{g}_t}{\Vert\mathbf{g}_t \Vert_2}s^{+}_t, \nabla L_S(\mathbf{w}_{t}) \rangle &\leq - \frac{c\eta_t s^{+}_t\Vert\nabla L_S(\mathbf{w}_t)\Vert_2}{3} + \frac{8c\eta_t\Vert \mathbf{g}_t-\nabla L_S(\mathbf{w}_t)\Vert_2}{3} \nonumber \\
     &\leq - \frac{c\eta_t(1-s^{-}_t)\Vert\nabla L_S(\mathbf{w}_t)\Vert_2}{3} + \frac{8c\eta_t\Vert \mathbf{g}_t-\nabla L_S(\mathbf{w}_t)\Vert_2}{3} .
\end{align}
For term $-\eta_t\langle \mathbf{g}_ts^{-}_t, \nabla L_S(\mathbf{w}_{t}) \rangle$, we obtain
\begin{align}
    -\eta_t\langle \mathbf{g}_ts^{-}_t, \nabla L_S(\mathbf{w}_{t}) \rangle &= -\eta_t s^{-}_t (\langle \mathbf{g}_t - \nabla L_S(\mathbf{w}_{t}), \nabla L_S(\mathbf{w}_{t}) \rangle + \Vert \nabla L_S(\mathbf{w}_{t}) \Vert^2_2) \nonumber \\
    &\leq -\eta_t s^{-}_t (-\Vert\mathbf{g}_t - \nabla L_S(\mathbf{w}_{t}) \Vert_2\Vert \nabla L_S(\mathbf{w}_{t}) \Vert_2 + \Vert \nabla L_S(\mathbf{w}_{t}) \Vert^2_2) \nonumber \\
    &\leq \eta_t \Vert\mathbf{g}_t - \nabla L_S(\mathbf{w}_{t}) \Vert_2\Vert \nabla L_S(\mathbf{w}_{t}) \Vert_2 - \frac{c}{2}\eta_t s^{-}_t \Vert \nabla L_S(\mathbf{w}_{t}) \Vert_2 \nonumber \\
    &\leq \eta_t \Vert\mathbf{g}_t - \nabla L_S(\mathbf{w}_{t}) \Vert_2\Vert \nabla L_S(\mathbf{w}_{t}) \Vert_2 - \frac{c}{3}\eta_t s^{-}_t \Vert \nabla L_S(\mathbf{w}_{t}) \Vert_2  .
\end{align}
According to Lemma~\ref{lemma:1}, with probability at least $1-\delta$, we have
\begin{align}
    \Vert\mathbf{g}_t - \nabla L_S(\mathbf{w}_{t}) \Vert_2 \leq K\log^{\theta}(2/\delta),
\end{align}
then we get
\begin{align}
    -\eta_t\langle \mathbf{g}_ts^{-}_t, \nabla L_S(\mathbf{w}_{t}) \rangle &\leq K\log^{\theta}(2/\delta)\Vert \nabla L_S(\mathbf{w}_{t}) \Vert_2 - \frac{c}{3}\eta_t s^{-}_t \Vert \nabla L_S(\mathbf{w}_{t}) \Vert_2,
\end{align}
and
\begin{align}
    -\eta_t\langle \frac{c\mathbf{g}_t}{\Vert\mathbf{g}_t \Vert_2}s^{+}_t, \nabla L_S(\mathbf{w}_{t}) \rangle &\leq - \frac{c\eta_t(1-s^{-}_t)\Vert\nabla L_S(\mathbf{w}_t)\Vert_2}{3} + \frac{8c\eta_tK\log^{\theta}(2/\delta)}{3} .
\end{align}

Using Lemma~\ref{lemma:2} to term $-\sum^T_{t=1}\eta_t\langle \mathbf{\zeta}_{t}, \nabla L_S(\mathbf{w}_{t}) \rangle$, with probability at least $1-\delta$, we have
\begin{align}
    -\sum^T_{t=1}\eta_t\langle \mathbf{\zeta}_{t}, \nabla L_S(\mathbf{w}_{t}) \rangle \leq 4\sqrt{em_2Td}\frac{cB\log(2/\delta)}{n\epsilon}\sum^T_{t=1}\eta_t\Vert\nabla L_S(\mathbf{w}_{t})\Vert_2.
\end{align}
So, combining formula.(34), formula.(35) and formula.(36) with term E.7, with probability at least $1-2\delta-T\delta$, we obtain
\begin{align}
    &-\sum^T_{t=1}\eta_t\langle \overline{\mathbf{g}}_t + \mathbf{\zeta}_{t}, \nabla L_S(\mathbf{w}_{t}) \rangle \leq - \sum^T_{t=1}\frac{c\eta_t}{3}\Vert \nabla L_S(\mathbf{w}_{t}) \Vert_2 + \sum^T_{t=1}\frac{8c\eta_tK\log^{\theta}(2/\delta)}{3}  \nonumber \\
    &+  K\log^{\theta}(2/\delta)\sum^T_{t=1}\eta_t \Vert \nabla L_S(\mathbf{w}_{t}) \Vert_2 + 4\sqrt{em_2Td}\frac{cB\log(2/\delta)}{n\epsilon}\sum^T_{t=1}\eta_t\Vert\nabla L_S(\mathbf{w}_{t})\Vert_2 \nonumber \\
    &\leq - \sum^T_{t=1}\frac{c\eta_t}{3}\Vert \nabla L_S(\mathbf{w}_{t}) \Vert_2 + (\frac{19}{3}K\log^{\theta}(2/\delta)+4\sqrt{em_2Td}\frac{cB\log(2/\delta)}{n\epsilon})\sum^T_{t=1}\eta_t \Vert \nabla L_S(\mathbf{w}_{t}) \Vert_2.
\end{align}
Next, considering all $T$ iterations and applying Lemma~\ref{lemma:2} to term E.8 with $\sigma^2_{\mathrm{dp}} = m_2\frac{TdB^2\log(1/\delta)}{n^2\epsilon^2}$, $d$ is dimension the dimension of the gradient and probability $1-4\delta-T\delta$, we have
\begin{align}
    &(\frac{c}{3} - \frac{19}{3}K\log^{\theta}(2/\delta)-4\sqrt{e}\sigma_{\mathrm{dp}}\log^{\frac{1}{2}}(2/\delta)) \sum^T_{t=1}\eta_t\Vert\nabla L_S(\mathbf{w}_{t})\Vert_2 \leq L_S(\mathbf{w}_{1}) - L_S(\mathbf{w}_{S}) \nonumber \\
    & + (2\beta m_2 e d\frac{Tc^2B^2\log^2(2/\delta)}{n^2\epsilon^2} + 2\beta \sqrt{em_2Td}\frac{c^2B\log(2/\delta)}{n\epsilon} + \frac{1}{2}\beta c^2)\sum^T_{t=1}\eta^2_{t} .
\end{align}

If $\theta = \frac{1}{2}$ and $19K\log^{\theta}(2/\delta) > 12\sqrt{e}\sigma_{\mathrm{dp}}\log^{\frac{1}{2}}(2/\delta)$, let $\frac{c}{3}\geq \frac{39}{3}K\log^{\frac{1}{2}}(2/\delta)$, i.e. $c\geq 39K\log^{\frac{1}{2}}(2/\delta) $, taking $c = 39K\log^{\frac{1}{2}}(2/\delta)$, $T= \frac{n\epsilon}{\sqrt{d\log(1/\delta)}}$ and $\eta_t = \frac{1}{\sqrt{T}}$, we have
\begin{align}
    &\sum^T_{t=1}\eta_t\Vert\nabla L_S(\mathbf{w}_{t})\Vert_2 \leq \frac{3}{K\log^{\frac{1}{2}}(2/\delta)} ( L_S(\mathbf{w}_{1}) - L_S(\mathbf{w}_{S}) )  \nonumber \\
    &+ \frac{3\sum^T_{t=1}\eta^2_{t}}{K\log^{\frac{1}{2}}(2/\delta)} \left( 2\beta m_2 e d\frac{Tc^2B^2\log^2(2/\delta)}{n^2\epsilon^2} + 2\beta \sqrt{em_2Td}\frac{c^2B\log(2/\delta)}{n\epsilon} + \frac{1}{2}\beta c^2 \right) \nonumber \\
    &\leq \frac{ L_S(\mathbf{w}_{1}) - L_S(\mathbf{w}_{S}) + 2\beta e\sigma^2_{\mathrm{dp}}\log(2/\delta) + 2\beta c\sqrt{e}\sigma_{\mathrm{dp}}\log^{\frac{1}{2}}(2/\delta) + \frac{39^2}{2}\beta K^2\log(2/\delta)  }{\frac{1}{3}K\log^{\frac{1}{2}}(2/\delta)}  \nonumber \\
    &\leq  \frac{3(L_S(\mathbf{w}_{1}) - L_S(\mathbf{w}_{S}) )}{K\log^{\frac{1}{2}}(2/\delta)} + 6\beta eK\log^{\frac{1}{2}}(2/\delta) + 6\beta\sqrt{e}\log^{\frac{1}{2}}(2/\delta) + 3\beta\frac{(39)^2}{2} K\log^{\frac{1}{2}}(2/\delta) .
\end{align}

Thus, with probability $1-4\delta-T\delta$, we have
\begin{align}
    \frac{1}{T}\sum^T_{t=1}\Vert\nabla L_S(\mathbf{w}_{t})\Vert_2 \leq \frac{1}{\sqrt{T}}\sum^T_{t=1}\eta_t\Vert\nabla L_S(\mathbf{w}_{t})\Vert_2 \leq \mathbb{O}\left(\frac{\log^{\frac{1}{2}}(1/\delta)}{\sqrt{T}}\right) = \mathbb{O}\left(\frac{\log^{\frac{1}{2}}(1/\delta)d^{\frac{1}{4}}\log^{\frac{1}{4}}(1/\delta)}{\sqrt{n\epsilon}}\right), \nonumber
\end{align}
implying that with probability $1-\delta$, we have
\begin{align}
    \frac{1}{T}\sum^T_{t=1}\Vert\nabla L_S(\mathbf{w}_{t})\Vert_2 \leq \mathbb{O}\left(\frac{d^{\frac{1}{4}}\log^{\frac{3}{4}}(T/\delta)}{\sqrt{n\epsilon}}\right).
\end{align}

If $\theta = \frac{1}{2}$ and $19K\log^{\theta}(2/\delta) \leq 12\sqrt{e}\sigma_{\mathrm{dp}}\log^{\frac{1}{2}}(2/\delta)$, that is, there exists $T= \max{(m_2eB^2\log(1/\delta), \frac{n\epsilon}{\sqrt{d\log(1/\delta)}})}$ and $\eta_t = \frac{1}{\sqrt{T}}$ that we obtain
\begin{align}
    &\sum^T_{t=1}\eta_t\Vert\nabla L_S(\mathbf{w}_{t})\Vert_2 \leq \frac{1}{\sqrt{e}\sigma_{\mathrm{dp}}\log^{\frac{1}{2}}(2/\delta)} ( L_S(\mathbf{w}_{1}) - L_S(\mathbf{w}_{S}) )  \nonumber \\
    &+ \frac{\sum^T_{t=1}\eta^2_{t}}{\sqrt{e}\sigma_{\mathrm{dp}}\log^{\frac{1}{2}}(2/\delta)} \left( 2\beta m_2 e d\frac{Tc^2B^2\log^2(2/\delta)}{n^2\epsilon^2} + 2\beta \sqrt{em_2Td}\frac{c^2B\log(2/\delta)}{n\epsilon} + \frac{1}{2}\beta c^2 \right)  \nonumber \\
    &\leq \frac{1}{\sqrt{e}\sigma_{\mathrm{dp}}\log^{\frac{1}{2}}(2/\delta)} ( L_S(\mathbf{w}_{1}) - L_S(\mathbf{w}_{S}) )  \nonumber \\
    &+ \frac{\sum^T_{t=1}\eta^2_{t}}{\sqrt{e}\sigma_{\mathrm{dp}}\log^{\frac{1}{2}}(2/\delta)} \left( 2\beta e\sigma^2_{\mathrm{dp}}\log(2/\delta) + 2\beta \sqrt{e}\sigma_{\mathrm{dp}}\log^{\frac{1}{2}}(2/\delta) + \frac{27^2}{2}\beta e \sigma^2_{\mathrm{dp}}\log(2/\delta) \right) \nonumber \\
    &\leq  \frac{L_S(\mathbf{w}_{1}) - L_S(\mathbf{w}_{S}) }{\sqrt{e}\sigma_{\mathrm{dp}}\log^{\frac{1}{2}}(2/\delta)} + 2\beta\sqrt{e}\sigma_{\mathrm{dp}}\log^{\frac{1}{2}}(2/\delta) + 54\beta \sigma_{\mathrm{dp}}\log^{\frac{1}{2}}(2/\delta) + \beta\frac{(27)^2}{2} \sqrt{e}\sigma_{\mathrm{dp}}\log^{\frac{1}{2}}(2/\delta) ,
\end{align}
with $c=27\sqrt{e}\sigma_{\mathrm{dp}}\log^{\frac{1}{2}}(2/\delta)$ and $\sigma_{\mathrm{dp}} = \frac{\sqrt{m_2Td}cB\log^{\frac{1}{2}}(1/\delta)}{n\epsilon} =\mathbb{O}(1)$.


Therefore, with probability $1-4\delta-T\delta$, we have
\begin{align}
    \frac{1}{T}\sum^T_{t=1}\Vert\nabla L_S(\mathbf{w}_{t})\Vert_2 \leq \mathbb{O}\left(\frac{\log^{\frac{1}{2}}(1/\delta)d^{\frac{1}{4}}\log^{\frac{1}{4}}(1/\delta)}{\sqrt{n\epsilon}}\right) , \nonumber
\end{align}
then, with probability $1-\delta$, we have
\begin{align}
    \frac{1}{T}\sum^T_{t=1}\Vert\nabla L_S(\mathbf{w}_{t})\Vert_2 \leq \mathbb{O}\left(\frac{d^{\frac{1}{4}}\log^{\frac{3}{4}}(T/\delta)}{\sqrt{n\epsilon}}\right) .
\end{align}

If $\theta > \frac{1}{2}$, then term $\log^{\theta}(2/\delta)$ dominates the left-hand inequality, i.e. $ \frac{19}{3}K\log^{\theta}(2/\delta)\geq4\sqrt{e}\sigma_{\mathrm{dp}}\log^{\frac{1}{2}}(2/\delta)$. Let $\frac{c}{3} \geq \frac{20}{3}K\log^{\theta}(2/\delta)$, $T= \frac{n\epsilon}{\sqrt{d\log(1/\delta)}}$ and $\eta_t = \frac{1}{\sqrt{T}}$, we obtain
\begin{align}
     &\sum^T_{t=1}\eta_t\Vert\nabla L_S(\mathbf{w}_{t})\Vert_2 \leq \frac{3}{K\log^{\theta}(2/\delta)} ( L_S(\mathbf{w}_{1}) - L_S(\mathbf{w}_{S}) )  \nonumber \\
    &+ \frac{3\sum^T_{t=1}\eta^2_{t}}{K\log^{\theta}(2/\delta)} \left( 2\beta m_2 e d\frac{Tc^2B^2\log^2(2/\delta)}{n^2\epsilon^2} + 2\beta \sqrt{em_2Td}\frac{c^2B\log(2/\delta)}{n\epsilon} + \frac{1}{2}\beta c^2 \right) \nonumber \\
    &\leq  \frac{3(L_S(\mathbf{w}_{1}) - L_S(\mathbf{w}_{S}) )}{K\log^{\theta}(2/\delta)} + \frac{19^2}{24}\beta K\log^{\theta}(2/\delta) + 190\beta K\log^{\theta}(2/\delta) + 3\beta(20)^2 K\log^{\theta}(2/\delta) .
\end{align}
Consequently, with probability $1-\delta$, we have
\begin{align}
    \frac{1}{T}\sum^T_{t=1}\Vert\nabla L_S(\mathbf{w}_{t})\Vert_2 \leq \mathbb{O}\left(\frac{\log^{\theta}(T/\delta)d^\frac{1}{4}\log^{\frac{1}{4}}(T/\delta)}{\sqrt{n\epsilon}}\right) .
\end{align}

Integrating the above results, when $\nabla L_S(\mathbf{w}_{t}) \geq c/2$ we have 
\begin{align}
    \frac{1}{T}\sum^T_{t=1}\Vert\nabla L_S(\mathbf{w}_{t})\Vert_2 \leq \mathbb{O}\left(\frac{d^\frac{1}{4}\log^{\theta+\frac{1}{4}}(1/\delta)}{\sqrt{n\epsilon}}\right),
\end{align}
with probability $1-\delta$ and $\theta\geq \frac{1}{2}$.

To sum up, covering the two cases, we ultimately come to the conclusion with probability $1-\delta$ and $\eta_t = \frac{1}{\sqrt{T}}$ 
\begin{align}
     \frac{1}{T}\sum^T_{t=1} \min\big\{\Vert\nabla L_S(\mathbf{w}_{t})\Vert_2, \Vert\nabla L_S(\mathbf{w}_{t})\Vert^2_2\big\} &\leq \mathbb{O}\left(\frac{d^{\frac{1}{4}}\log^{\theta+\frac{1}{4}}(T/\delta)}{(n\epsilon)^{\frac{1}{2}}}\right) + \mathbb{O}\left(\frac{d^{\frac{1}{4}}\log^{2\theta}(\sqrt{T})\log^{\frac{5}{4}}(T/\delta) }{(n\epsilon)^\frac{1}{2}}\right) \nonumber \\
     &\leq \mathbb{O}\left(\frac{d^{\frac{1}{4}} \log^{\frac{5}{4}}(T/\delta) \big (\log^{\theta-1}(T/\delta) + \log^{2\theta}(\sqrt{T}) \big ) }{(n\epsilon)^\frac{1}{2}}\right) \nonumber \\
     &\leq \mathbb{O}\left(\frac{d^{\frac{1}{4}} \log^{\frac{5}{4}}(T/\delta)\hat{\log}(T/\delta)\log^{2\theta}(\sqrt{T}) }{(n\epsilon)^\frac{1}{2}}\right) ,
\end{align}
where $\hat{\log}(T/\delta) = \log^{\max(0,\theta-1)}(T/\delta)$. If $\theta = \frac{1}{2}$ and $19K\log^{\theta}(2/\delta) \leq 12\sqrt{e}\sigma_{\mathrm{dp}}\log^{\frac{1}{2}}(2/\delta)$, then $T= \max{\big(m_2eB^2\log(1/\delta), \frac{n\epsilon}{\sqrt{d\log(1/\delta)}}\big)}$ and $c=\max{\big(4K\log^{\theta}({\sqrt{T}}),27\sqrt{e}\sigma_{\mathrm{dp}}\log^{\frac{1}{2}}(2/\delta)\big)}$. If $\theta = \frac{1}{2}$ and $19K\log^{\theta}(2/\delta) \geq 12\sqrt{e}\sigma_{\mathrm{dp}}\log^{\frac{1}{2}}(2/\delta)$, then $T= \frac{n\epsilon}{\sqrt{d\log(1/\delta)}}$ and $c=\max{\big(4K\log^{\theta}({\sqrt{T}}),39K\log^{\frac{1}{2}}(2/\delta)\big)}$. If $\theta > \frac{1}{2}$, then  $T= \frac{n\epsilon}{\sqrt{d\log(1/\delta)}}$ and $c=\max{\big(4K\log^{\theta}({\sqrt{T}}),20K\log^{\theta}(2/\delta)\big)}$.
\end{proof}

The proof is completed.

\clearpage
\section{Subspace Skewing for Identification}
\begin{theorem}[\textbf{Subspace Skewing for Identification}] Assume that the second moment matrix $M := V_kV^T_k$ with $V^T_kV_k=\mathbb{I}$ approximates the population second moment matrix $\hat{M} := \hat{V}_k\hat{V}^T_k = \mathbb{E}[V_kV^T_k]$, $\lambda_{\mathrm{tr}} := \mathrm{tr}(V^T_k uu^T V_k)$ and $\hat{\lambda}_{\mathrm{tr}}:=\mathrm{tr}(\hat{V}^T_k uu^T\hat{V}_k) $, for any vector $u$ that satisfies $\Vert u \Vert_2 = 1$, $\zeta_{\mathrm{tr}} \sim \mathbb{N}(0,\sigma^2_{\mathrm{tr}}\mathbb{I})$ and $\delta \in (0,1)$, with probability $1-\delta_m - \delta$, we have
\begin{align}
    |\lambda_{\mathrm{tr}} - \hat{\lambda}_{\mathrm{tr}}+\zeta_{\mathrm{tr}}| \leq \frac{4\log{(2d/\delta_m)}}{k} + \sigma_{\mathrm{tr}}\log^{\frac{1}{2}}(2/\delta) . \nonumber
\end{align}
\end{theorem}
\begin{proof}
For simplicity, we abbreviate $\hat{g}_t(x_i)$ as $\hat{g}_t$. Due to the Fact~\ref{fact:1}, $V^T_kV_k = \mathbb{I}$ and $\hat{V}^T_k\hat{V}_k = \mathbb{I}$, we omit subscripts of expectation and have
\begin{align}
    |\lambda_{\mathrm{tr}} - \hat{\lambda}_{\mathrm{tr}}| &:= |\text{tr}(V^T_k \hat{g}_t\hat{g}^T_t V_k) - \text{tr}(\hat{V}^T_k \hat{g}_t\hat{g}^T_t\hat{V}_k)| \nonumber \\
    &= |\Vert V^T_k\hat{g}_t\Vert^2_2 - \Vert \hat{V}^T_k\hat{g}_t\Vert^2_2| \nonumber \\
    &= |\Vert V_kV^T_k\hat{g}_t\Vert^2_2 - \Vert \hat{V}_k\hat{V}^T_k\hat{g}_t\Vert^2_2| \nonumber \\
    &\leq \Vert V_kV^T_k\hat{g}_t - \hat{V}_k\hat{V}^T_k\hat{g}_t \Vert^2_2 \nonumber \\
    &\leq \Vert V_kV^T_k - \hat{V}_k\hat{V}^T_k\Vert^2_2 \Vert\hat{g}_t \Vert^2_2
\end{align}
To bound $\mathbb{E}\Vert V_kV^T_k - \hat{V}_k\hat{V}^T_k\Vert^2_2$, we need to bound the gap between the sum of the random positive semidefinite matrix $M := V_kV^T_k = \frac{1}{k}\sum^k_{i=1}v_iv^T_i$ and the expectation $\hat{M} := \hat{V}_k\hat{V}^T_k = \mathbb{E}[V_kV^T_k]$.

Due to $\Vert v_j \Vert_2 = 1$, we can easily get  
\begin{align}
    \Vert M\Vert_2&= \Vert\frac{1}{k}\sum^k_{i=1}v_iv^T_i\Vert_2 \leq \frac{1}{k}\sum^k_{i=1}\Vert v_iv^T_i\Vert_2 \nonumber \\
    &= \text{sup}_{x:\Vert x\Vert_2=1}\frac{1}{k}\sum^k_{i=1} x^Tv_iv^T_ix \nonumber \\
    &= \text{sup}_{x:\Vert x\Vert_2=1}\frac{1}{k}\sum^k_{i=1} \langle x,v_i \rangle\nonumber \\
    &\leq \frac{1}{k}\sum^k_{i=1}\Vert x\Vert_2\Vert v_i\Vert_2 \nonumber \\
    &=1
\end{align}
Thus, $\Vert M\Vert_2 \leq 1$ and  $\Vert \mathbb{E}M\Vert_2 = \Vert M\cdot\mathbb{P}(M) \Vert_2 \leq 1$ because of $\mathbb{P}(M) \leq 1$.

Then, according to Ahlswede-Winter Inequality with $R=1$ and $m=k$, we have for any $\mu \in (0,1)$
\begin{align}
&\mathbb{P}(\Vert M - \hat{M} \Vert_2 > \mu) \leq 2d \cdot \text{exp}(\frac{-k\mu^2}{4}), 
\end{align}
where $d$ is dimension of gradients. The inequality shows that the bounded spectral norm of random matrix $\Vert M \Vert_2$ concentrates around its expectation with high probability $1-2d \cdot \text{exp}(-k\mu^2/4)$.

Since $\Vert M\Vert_2 \in [0,1]$ and $\Vert\mathbb{E} M\Vert_2 \in [0,1]$, $\Vert M - \hat{M} \Vert_2$ is always bounded by 1. Therefore, for $\mu\geq 1$, $\Vert M - \hat{M} \Vert_2 > u$ holds with probability 0. So that for any $\mu > 0$, we have
\begin{align}
&\mathbb{P}(\Vert M - \hat{M} \Vert_{2} > 2\sqrt{\frac{\log{2d}}{k}}\mu) \leq \text{exp}(-\mu^2) .
\end{align}
Based on the inequality above, with probability $1-\delta_m$, we have
\begin{align}
    \Vert M - \hat{M} \Vert_{2} \leq 2\frac{\log^{\frac{1}{2}}{(2d/\delta_m)}}{\sqrt{k}} .
\end{align}
Next, considering that we have implicitly normalized the term $\Vert \hat{g}_t\Vert^2_2$ by the threshold $1$, the upper bound of $\Vert \hat{g}_t\Vert^2_2$ is $1$. As a result, we obtain 
\begin{align}
    |\lambda_{\mathrm{tr}} - \hat{\lambda}_{\mathrm{tr}}| &\leq \Vert V_kV^T_k - \hat{V}_k\hat{V}^T_k\Vert^2_2 \Vert\hat{g}_t \Vert^2_2 \nonumber \\
    &\leq \Vert V_kV^T_k - \hat{V}_k\hat{V}^T_k\Vert^2_2 \nonumber \\
    &\leq \Vert M - \hat{M}\Vert^2_2 \nonumber \\
    &\leq \frac{4\log{(2d/\delta_m)}}{k},
\end{align}
with probability $1-\delta_m$.

Due to the shared random subspace of per-sample gradient, the exposed trace may pose potential privacy risks. Thus, we add the noise that satisfies differential privacy to the trace $\lambda_{\mathrm{tr}}$, i.e. $\lambda_{\mathrm{tr}} + \zeta_{\mathrm{tr}}$. The upper bound of the trace for per-sample gradient is limited to 1, because we normalize per-sample gradient in advance. So, the sensitivity in differential privacy can be regarded as 1, which means $\zeta_{\mathrm{tr}} \sim \mathbb{N}(0,\sigma^2_{\mathrm{tr}}\mathbb{I})$. Then, applying Gaussian properties, with probability $1-\delta_{m}-\delta$, we have
\begin{align}
    |\lambda_{\mathrm{tr}} - \hat{\lambda}_{\mathrm{tr}}+\zeta_{\mathrm{tr}}| &\leq |\lambda_{\mathrm{tr}} - \hat{\lambda}_{\mathrm{tr}}| + |\zeta_{\mathrm{tr}} | \nonumber \\
    &\leq \frac{4\log{(2d/\delta_m)}}{k} + \sigma_{\mathrm{tr}}\log^{\frac{1}{2}}(2/\delta) .
\end{align}

In addition,  since $\lambda_{\mathrm{tr}}$ is a constant, the scale $\sigma_{\mathrm{tr}}$ of noise added is actually small compared to the noise added to gradients. Accordingly, the term $ \frac{4\log{(2d/\delta_m)}}{k}$ will dominate the error of subspace skewing, and we can control this part of the error by adjusting a larger $k$.

In conclusion, for the per-sample trace, there is a high probability $1-\delta^{\prime}_m$ that we can accurately identify heavy-tailed samples within a finite and minor error dependent on the factor $\mathbb{O}(\frac{1}{k})$.

\end{proof}
The proof is completed.

\clearpage
\section{Convergence of Discriminative Clipping DPSGD}
\begin{theorem}[\textbf{Convergence of Discriminative Clipping DPSGD}]
Under Assumption A.1 and A.2, let $\mathbf{w}_{t}$ be the iterate produced by Algorithm Discriminative Clipping DPSGD and $\eta_t = \frac{1}{\sqrt{T}}$. Define $\hat{\log}(T/\delta) = \log^{\max(0,\theta-1)}(T/\delta)$, $x_{\mathrm{max}} = \frac{\mu I(x)}{x}aK^2$, $a = 2$ if $\theta = 1/2$, $a = (4\theta)^{2\theta}e^2$ if $\theta \in (1/2,1]$ and $a = (2^{2\theta+1}+2)\Gamma(2\theta+1) + \frac{2^{3\theta}\Gamma(3\theta+1)}{3}$ if $\theta > 1$, for any $\delta \in (0,1)$, with probability $1-\delta$, then we have:
\begin{enumerate}[leftmargin=*]
\item[a.]
For the case $0\leq x\leq x_{\mathrm{max}}$,
\begin{align}
     \frac{1}{T}\sum^T_{t=1} \min\big\{\Vert\nabla L_S(\mathbf{w}_{t})\Vert_2, \Vert\nabla L_S(\mathbf{w}_{t})\Vert^2_2\big\} \leq \mathbb{O}\left(\frac{d^{\frac{1}{4}} \log^{\frac{5}{4}}(T/\delta)\log(\sqrt{T}) }{(n\epsilon)^\frac{1}{2}}\right). \nonumber
\end{align}
$\mathrm{(1)}$ If $16\sqrt{2a}K\log^{\frac{1}{2}}(2/\delta)\leq 12\sqrt{e}\sigma_{\mathrm{dp}}\log^{\frac{1}{2}}(1/\delta) $, then $T= \max{\big(m_2eB^2\log(1/\delta), \frac{n\epsilon}{\sqrt{d\log(1/\delta)}}\big)}$ and $c=\max{\big(2\sqrt{2a}K\log^{\frac{1}{2}}({\sqrt{T}}),27\sqrt{e}\sigma_{\mathrm{dp}}\log^{\frac{1}{2}}(1/\delta)\big)}$.\\
$\mathrm{(2)}$ If $16\sqrt{2a}K\log^{\frac{1}{2}}(2/\delta)\geq 12\sqrt{e}\sigma_{\mathrm{dp}}\log^{\frac{1}{2}}(1/\delta) $, then $T= \frac{n\epsilon}{\sqrt{d\log(1/\delta)}}$ and $c=\max{\big(2\sqrt{2a}K\log^{\frac{1}{2}}({\sqrt{T}}),33\sqrt{2a}K\log^{\frac{1}{2}}(2/\delta)\big)}$. 
\item[b.]
For the case $x\geq x_{\mathrm{max}}$,
\begin{align}
     \frac{1}{T}\sum^T_{t=1} \min\big\{\Vert\nabla L_S(\mathbf{w}_{t})\Vert_2, \Vert\nabla L_S(\mathbf{w}_{t})\Vert^2_2\big\} \leq \mathbb{O}\left(\frac{d^{\frac{1}{4}} \log^{\frac{5}{4}}(T/\delta)\hat{\log}(T/\delta)\log^{2\theta}(\sqrt{T}) }{(n\epsilon)^\frac{1}{2}}\right). \nonumber
\end{align}
$\mathrm{(1)}$ If $\theta = \frac{1}{2}$ and $16\sqrt{2a}K\log^{\theta}(2/\delta)\leq 12\sqrt{e}\sigma_{\mathrm{dp}}\log^{\frac{1}{2}}(1/\delta) $, then $T= \max{\big(m_2eB^2\log(1/\delta), \frac{n\epsilon}{\sqrt{d\log(1/\delta)}}\big)}$ and $c=\max{\big(4^{\theta}2K\log^{\theta}({\sqrt{T}}),27\sqrt{e}\sigma_{\mathrm{dp}}\log^{\frac{1}{2}}(1/\delta)\big)}$. \\
$\mathrm{(2)}$ If $\theta = \frac{1}{2}$ and $16\sqrt{2a}K\log^{\theta}(2/\delta)\geq 12\sqrt{e}\sigma_{\mathrm{dp}}\log^{\frac{1}{2}}(1/\delta) $, then $T= \frac{n\epsilon}{\sqrt{d\log(1/\delta)}}$ and $c=\max{\big(4^{\theta}2K\log^{\theta}({\sqrt{T}}),33\sqrt{2a}K\log^{\frac{1}{2}}(2/\delta)\big)}$. \\
$\mathrm{(3)}$ If $\theta > \frac{1}{2}$, then  $T= \frac{n\epsilon}{\sqrt{d\log(1/\delta)}}$ and $c=\max{\big(4^{\theta}2K\log^{\theta}({\sqrt{T}}),17K\log^{\theta}(2/\delta)\big)}$.
\end{enumerate}
\end{theorem}

\begin{proof}
We review two cases in Discriminative Clipping DPSGD: $L_S(\mathbf{w}_{t})\leq c/2$ and $L_S(\mathbf{w}_{t})\geq c/2$. 

Firstly, in the case $\nabla L_S(\mathbf{w}_{t}) \leq c/2$:
\begin{align}
    &L_S(\mathbf{w}_{t+1}) - L_S(\mathbf{w}_{t}) \leq \langle\mathbf{w}_{t+1}-\mathbf{w}_{t}, \nabla L_S(\mathbf{w}_{t})\rangle + \frac{1}{2}\beta\Vert\mathbf{w}_{t+1}-\mathbf{w}_{t}\Vert^2  \nonumber \\
    &\leq  -\eta_{t}\langle \overline{\mathbf{g}}_t - \mathbb{E}_t[\overline{\mathbf{g}}_t], \nabla L_S(\mathbf{w}_{t}) \rangle -\eta_{t}\langle \mathbb{E}_t[\overline{\mathbf{g}}_t] - \nabla L_S(\mathbf{w}_{t}), \nabla L_S(\mathbf{w}_{t}) \rangle - \eta_t\langle\mathbf{\zeta}_{t},\nabla L_S(\mathbf{w}_{t})\rangle \nonumber \\
    &- \eta_t\Vert \nabla L_S(\mathbf{w}_{t})\Vert^2 + \frac{1}{2}\beta\eta^2_{t}\Vert \overline{\mathbf{g}}_t\Vert^2 + \frac{1}{2}\beta\eta^2_{t}\Vert \mathbf{\zeta}_{t}\Vert^2 + \beta\eta^2_{t}\langle \overline{\mathbf{g}}_t, \mathbf{\zeta}_{t}\rangle \nonumber
\end{align}
Applying the properties of Gaussian tails and Lemma~\ref{lemma:2} to $\zeta_t$, Lemma~\ref{lemma:4} to term $\sum^T_{t=1}\eta_{t}\langle \overline{\mathbf{g}}_t - \mathbb{E}_t[\overline{\mathbf{g}}_t], \nabla L_S(\mathbf{w}_{t}) \rangle$, with probability $1-4\delta$, we have
\begin{align}
     &\sum^T_{t=1}\eta_{t}\Vert\nabla L_S(\mathbf{w}_{t}) \Vert^2_2 \leq L_S(\mathbf{w}_{1}) - L_S(\mathbf{w}_{S}) + \sum^T_{t=1}\frac{1}{2}\beta\eta^2_{t}c^2 +  2\beta m_2 e d\frac{Tc^2B^2\log^2(2/\delta)}{n^2\epsilon^2}\sum^T_{t=1}\eta^2_{t} \nonumber \\
     &+2\beta \sqrt{em_2Td}\frac{c^2B\log(2/\delta)}{n\epsilon}\sum^T_{t=1}\eta^2_{t} + 2\sqrt{em_2Td}\frac{c^2B\log(2/\delta)}{n\epsilon}\sum^T_{t=1}\eta_t + \frac{\eta_tc^2\log{(1/\delta)}}{\rho} \nonumber \\
     &+ \frac{4\rho c^2\sum^T_{t=1}\eta^2_t\Vert \nabla L_S(\mathbf{w}_{t})\Vert^2_2}{\eta_tc^2} -\underbrace{\sum^T_{t=1}\eta_{t}\langle \mathbb{E}_t[\overline{\mathbf{g}}_t] - \nabla L_S(\mathbf{w}_{t}), \nabla L_S(\mathbf{w}_{t}) \rangle}_{\mathrm{E}.9}.
\end{align}

We will consider a truncated version of term E.9 in the following. Similarly,
\begin{align}
    \sum^T_{t=1}\eta_{t}\langle \mathbb{E}_t[\overline{\mathbf{g}}_t] - \nabla L_S(\mathbf{w}_{t}), \nabla L_S(\mathbf{w}_{t}) \rangle \leq \frac{1}{2}\sum^T_{t=1}\eta_{t}\Vert \mathbb{E}_t[\overline{\mathbf{g}}_t]-\nabla L_S(\mathbf{w}_{t})\Vert^2_2 + \frac{1}{2}\sum^T_{t=1}\eta_{t}\Vert \nabla L_S(\mathbf{w}_{t})\Vert^2_2. \nonumber 
\end{align}
For term $\Vert \mathbb{E}_t[\overline{\mathbf{g}}_t]-\nabla L_S(\mathbf{w}_{t})\Vert_2$, we also define $a_t=\mathbb{I}_{\Vert\mathbf{g}_t \Vert_2>c}$ and $b_t=\mathbb{I}_{\Vert\mathbf{g}_t -\nabla L_S(\mathbf{w}_{t})\Vert_2>\frac{c}{2}}$, and have
\begin{align}
    \Vert \mathbb{E}_t[\overline{\mathbf{g}}_t]-\nabla L_S(\mathbf{w}_{t})\Vert_2 &= \Vert \mathbb{E}_t[(\overline{\mathbf{g}}_t - \mathbf{g}_t)a_t]\Vert_2 \nonumber \\
    &\leq  \mathbb{E}_t[\Vert(\mathbf{g}_t (\frac{c-\Vert\mathbf{g}_t \Vert_2}{\Vert\mathbf{g}_t \Vert_2})a_t\Vert_2] \nonumber \\
    &\leq  \mathbb{E}_t[|\Vert\mathbf{g}_t\Vert_2 - \Vert\nabla L_S(\mathbf{w}_{t})\Vert_2|a_t] \nonumber \\
    &\leq  \mathbb{E}_t[|\Vert\mathbf{g}_t - \nabla L_S(\mathbf{w}_{t})\Vert_2|b_t] \nonumber \\
    &\leq  \sqrt{\mathbb{E}_t[\Vert\mathbf{g}_t - \nabla L_S(\mathbf{w}_{t})\Vert^2_2]\mathbb{E}_tb^2_t} .
\end{align}

Due to $\mathbb{E}[\mathbf{g}_t-\nabla L_S(\mathbf{w}_t)]=0$, applying Lemma~\ref{lemma:7} and~\ref{lemma:8}
with 
\begin{align}
    &m = 1 \nonumber \\
    &\sup_{\eta\in(0,1]}\{v(L,\mu)\} = aK^2 \nonumber \\
    &t_{\mathrm{max}} = \frac{\mu I(t)}{t}aK^2 \nonumber \\
    &c_t \in [\frac{1}{2},1] \nonumber \\
    &\mu = \frac{1}{2} \nonumber, 
\end{align}
we have the Corollary D.1~\label{cor:D-1} that
\begin{align}
    \mathbb{P}(\Vert \mathbf{g}_t-\nabla L_S(\mathbf{w}_t)\Vert_2 > t) &\leq \mathrm{exp}(-c_t\mu I(t)) + \mathrm{exp}(-I(t)) \nonumber \\
    &\leq \mathrm{exp}(-\frac{1}{4} I(t)) + \mathrm{exp}(-I(t)) \nonumber \\
    &\leq 2\mathrm{exp}(-\frac{1}{4} I(t)), 
\end{align}
when $t\geq t_{\mathrm{max}}(\mu)$. Then,
\begin{align}
    \mathbb{P}(\Vert \mathbf{g}_t-\nabla L_S(\mathbf{w}_t)\Vert_2 > t) &\leq \mathrm{exp}(-\frac{t^2}{2v(t_{\mathrm{max}}(\mu),\mu)}) + m\mathrm{exp}(-\frac{t^2_{\mathrm{max}}(\mu)}{\eta v(t_{\mathrm{max}}(\mu),\mu)}) \nonumber \\
    &\leq 2\mathrm{exp}(-\frac{t^2}{2v(t_{\mathrm{max}}(\mu),\mu)}) \nonumber \\
    &\leq 2\mathrm{exp}(-\frac{t^2}{2aK^2}),
\end{align}
when $0\leq t\leq t_{\mathrm{max}}(\mu)$.

Therefore, when $0 \leq t \leq t_{\mathrm{max}}$, we have the follow-up truncated conclusions:

If $\theta = \frac{1}{2}$, $\forall \alpha >0$ and $a=2$, we have the following inequality with probability at least $1-\delta$
\begin{align}
    \Vert \mathbf{g}_t-\nabla L_S(\mathbf{w}_t)\Vert_2  &\leq 2K\log^{\frac{1}{2}}(2/\delta). \nonumber 
\end{align}

If $\theta \in (\frac{1}{2},1]$, let $a=(4\theta)^{2\theta}e^2$, we have the following inequality with probability at least $1-\delta$
\begin{align}
    \Vert \mathbf{g}_t-\nabla L_S(\mathbf{w}_t)\Vert_2  &\leq \sqrt{2}e(4\theta)^{\theta}K\log^{\frac{1}{2}}(2/\delta). \nonumber 
\end{align}

If $\theta > 1$, let $a=(2^{2\theta+1}+2)\Gamma(2\theta+1) + \frac{2^{3\theta}\Gamma(3\theta+1)}{3}$, we have the following inequality with probability at least $1-\delta$
\begin{align}
    \Vert \mathbf{g}_t-\nabla L_S(\mathbf{w}_t)\Vert_2  &\leq \sqrt{2(2^{2\theta+1}+2)\Gamma(2\theta+1) + \frac{2^{3\theta}\Gamma(3\theta+1)}{3}}K\log^{\frac{1}{2}}(2/\delta). \nonumber 
\end{align}

When $t \geq t_{\mathrm{max}}$, let $I(t)=(t/K)^{\frac{1}{\theta}},~\forall \theta \in (\frac{1}{2},1]$, with probability at least $1-\delta$, then we have
\begin{align}
  \Vert \mathbf{g}_t-\nabla L_S(\mathbf{w}_t)\Vert_2   &\leq  4^{\theta}K\log^{\theta}(2/\delta) \nonumber .
\end{align}

Apply the truncated Corollary~\ref{cor:D-1}.1 above, when $0\leq t \leq t_{\mathrm{max}}$, we have
\begin{align}
    \mathbb{E}_t[\Vert\mathbf{g}_t - \nabla L_S(\mathbf{w}_{t})\Vert_2] \leq \sqrt{2a}K
\end{align}
and with probability $1-\delta$,
\begin{align}
    \mathbb{E}_tb^2_t = \mathbb{P}(\Vert\mathbf{g}_t -\nabla L_S(\mathbf{w}_{t})\Vert_2>\frac{c}{2})\leq 2\mathrm{exp}(-(\frac{c}{2\sqrt{2a}K})^{2} )
\end{align}
where $a = 2$ if $\theta = 1/2$, $a = (4\theta)^{2\theta}e^2$ if $\theta \in (1/2,1]$ and $a = (2^{2\theta+1}+2)\Gamma(2\theta+1) + \frac{2^{3\theta}\Gamma(3\theta+1)}{3}$ if $\theta > 1$.

When $t \geq t_{\mathrm{max}}$, the inequalities
\begin{align}
    \mathbb{E}_t[\Vert\mathbf{g}_t - \nabla L_S(\mathbf{w}_{t})\Vert_2] \leq 4^{\theta}K
\end{align}
and 
\begin{align}
    \mathbb{E}_tb^2_t = \mathbb{P}(\Vert\mathbf{g}_t -\nabla L_S(\mathbf{w}_{t})\Vert_2>\frac{c}{2})\leq 2\mathrm{exp}(-\frac{1}{4}(\frac{c}{2K})^{\frac{1}{\theta}} )
\end{align}
hold with probability $1-\delta$, where $\theta \geq \frac{1}{2}$.

Thus, with probability $1-T\delta$, we get 
\begin{align}
    \sum^T_{t=1}\eta_{t}\langle \mathbb{E}_t[\overline{\mathbf{g}}_t] - \nabla L_S(\mathbf{w}_{t}), \nabla L_S(\mathbf{w}_{t}) \rangle \leq 2aK^2\sum^T_{t=1}\eta_{t}\mathrm{exp}(-(\frac{c}{2\sqrt{2a}K})^{2} ) + \frac{1}{2}\sum^T_{t=1}\eta_{t}\Vert \nabla L_S(\mathbf{w}_{t})\Vert^2_2,  
\end{align}
when $0\leq t \leq t_{\mathrm{max}}$.

With probability $1-T\delta$, we obtain
\begin{align}
     \sum^T_{t=1}\eta_{t}\langle \mathbb{E}_t[\overline{\mathbf{g}}_t] - \nabla L_S(\mathbf{w}_{t}), \nabla L_S(\mathbf{w}_{t}) \rangle \leq 4^{2\theta}K^2\sum^T_{t=1}\eta_{t}\mathrm{exp}(-\frac{1}{4}(\frac{c}{2K})^{\frac{1}{\theta}} ) + \frac{1}{2}\sum^T_{t=1}\eta_{t}\Vert \nabla L_S(\mathbf{w}_{t})\Vert^2_2, 
\end{align}
when $t \geq t_{\mathrm{max}}$.

By setting $\rho=\frac{1}{16}$, $T= \frac{n\epsilon}{\sqrt{d\log(1/\delta)}}$ and $\eta_t = \frac{1}{\sqrt{T}}$, with probability $1-4\delta-T\delta$, we have
\begin{align}
    &\frac{1}{4}\sum^T_{t=1}\eta_{t}\Vert\nabla L_S(\mathbf{w}_{t}) \Vert^2_2 \leq L_S(\mathbf{w}_{1}) - L_S(\mathbf{w}_{S}) + \frac{1}{2}\beta c^2  +  2\beta m_2 e\frac{d^{\frac{1}{2}}c^2B^2\log^{\frac{3}{2}}(2/\delta)}{n\epsilon} \nonumber \\
     &+2\beta \sqrt{em_2}\frac{d^{\frac{1}{4}}c^2B\log^{\frac{1}{2}}(2/\delta)}{\sqrt{n\epsilon}} + 2\sqrt{em_2}c^2B\log^{\frac{1}{2}}(2/\delta) + \frac{16d^{\frac{1}{4}}c^2\log^{\frac{5}{4}}{(1/\delta)}}{\sqrt{n\epsilon}} \nonumber \\
     &+  \mathrm{E.}10~\left\{ \begin{array}{ll} 
         \displaystyle 2aK^2\sum^T_{t=1}\eta_{t}\mathrm{exp}(-(\frac{c}{2\sqrt{2a}K})^{2} ) , & \text{if } 0\leq t \leq t_{\mathrm{max}}, \\ \\ 
         \displaystyle 4^{2\theta}K^2\sum^T_{t=1}\eta_{t}\mathrm{exp}(-\frac{1}{4}(\frac{c}{2K})^{\frac{1}{\theta}} ),   & \text{if } t \geq t_{\mathrm{max}}. \end{array} \right.   
\end{align}
Let term $\mathrm{E}.10 \leq \frac{1}{\sqrt{T}}$, and we have $c\geq 2\sqrt{2a}K\log^{\frac{1}{2}}({\sqrt{T}})$ if $0\leq t \leq t_{\mathrm{max}}$ and $c\geq 4^{\theta}2K\log^{\theta}({\sqrt{T}})$ if $t \geq t_{\mathrm{max}}$. 

If $0\leq t \leq t_{\mathrm{max}}$, by taking $c=2\sqrt{2a}K\log^{\frac{1}{2}}({\sqrt{T}})$ we achieve
\begin{align}
    \frac{1}{\sqrt{T}}\sum^T_{t=1}\eta_{t}\Vert\nabla L_S(\mathbf{w}_{t}) \Vert^2_2 &\leq \frac{4(L_S(\mathbf{w}_{1}) - L_S(\mathbf{w}_{S}) )}{\sqrt{T}} + \frac{2aK^2}{\sqrt{T}} \nonumber \\
    &+ \frac{8aK^2\log(\sqrt{T})\log(2/\delta)}{\sqrt{T}}\left( 2\beta + 8\beta m_2 eB^2 (\frac{d^{\frac{1}{4}}\log^{\frac{1}{4}}(2/\delta)}{\sqrt{n\epsilon}})^2 \right. \nonumber \\
    &\left. + 8\beta \sqrt{em_2}\frac{d^{\frac{1}{4}}B\log^{-\frac{1}{2}}(2/\delta)}{\sqrt{n\epsilon}}  + 8\sqrt{em_2}B\log^{-\frac{1}{2}}(2/\delta) + \frac{64d^{\frac{1}{4}}\log^{\frac{1}{4}}{(1/\delta)}}{\sqrt{n\epsilon}} \right) \nonumber \\
    &\leq \mathbb{O}\left(\frac{\log(\sqrt{T})\log(1/\delta)}{\sqrt{T}}\cdot \frac{d^{\frac{1}{4}}\log^{\frac{1}{4}}{(1/\delta)}}{\sqrt{n\epsilon}} \right) \nonumber \\
    &\leq \mathbb{O}\left(\frac{\log(\sqrt{T})d^{\frac{1}{4}}\log^{\frac{5}{4}}{(1/\delta)}}{\sqrt{n\epsilon}}\right).
\end{align}

If $t \geq t_{\mathrm{max}}$, by taking $c=4^{\theta}2K\log^{\theta}({\sqrt{T}})$ we achieve
\begin{align}
    \frac{1}{\sqrt{T}}\sum^T_{t=1}\eta_{t}\Vert\nabla L_S(\mathbf{w}_{t}) \Vert^2_2 &\leq \frac{4(L_S(\mathbf{w}_{1}) - L_S(\mathbf{w}_{S}) )}{\sqrt{T}} + \frac{2aK^2}{\sqrt{T}} \nonumber \\
    &+ \frac{4^{2\theta+1}\log^{2\theta}({\sqrt{T}})\log(2/\delta)}{\sqrt{T}}\left( 2\beta + 8\beta m_2 eB^2 (\frac{d^{\frac{1}{4}}\log^{\frac{1}{4}}(2/\delta)}{\sqrt{n\epsilon}})^2 \right. \nonumber \\
    &\left. + 8\beta \sqrt{em_2}\frac{d^{\frac{1}{4}}B\log^{-\frac{1}{2}}(2/\delta)}{\sqrt{n\epsilon}}  + 8\sqrt{em_2}B\log^{-\frac{1}{2}}(2/\delta) + \frac{64d^{\frac{1}{4}}\log^{\frac{1}{4}}{(1/\delta)}}{\sqrt{n\epsilon}} \right) \nonumber \\
    &\leq \mathbb{O}\left(\frac{\log^{2\theta}(\sqrt{T})\log(1/\delta)}{\sqrt{T}}\cdot \frac{d^{\frac{1}{4}}\log^{\frac{1}{4}}{(1/\delta)}}{\sqrt{n\epsilon}} \right) \nonumber \\
    &\leq \mathbb{O}\left(\frac{\log^{2\theta}(\sqrt{T})d^{\frac{1}{4}}\log^{\frac{5}{4}}{(1/\delta)}}{\sqrt{n\epsilon}}\right).
\end{align}

Secondly, we pay extra attention to the bound in the case $\nabla L_S(\mathbf{w}_{t}) \geq c/2$.
\begin{align}
    L_S(\mathbf{w}_{t+1}) - L_S(\mathbf{w}_{t}) &\leq \langle\mathbf{w}_{t+1}-\mathbf{w}_{t}, \nabla L_S(\mathbf{w}_{t})\rangle + \frac{1}{2}\beta\Vert\mathbf{w}_{t+1}-\mathbf{w}_{t}\Vert^2_2 \nonumber \\
    &\leq \underbrace{-\eta_t\langle \overline{\mathbf{g}}_t + \mathbf{\zeta}_{t}, \nabla L_S(\mathbf{w}_{t}) \rangle}_{\mathrm{E}.9} + \frac{1}{2}\beta\eta^2_t\Vert \overline{\mathbf{g}}_t + \mathbf{\zeta}_{t} \Vert^2_2. 
\end{align}

We revisit term E.9 in the case and also set $s^{+}_t = \mathbb{I}_{\Vert\mathbf{g}_t\Vert_2\geq c}$ and $s^{-}_t = \mathbb{I}_{\Vert\mathbf{g}_t\Vert_2\leq c}$.
\begin{align}
    -\eta_t\langle \overline{\mathbf{g}}_t + \mathbf{\zeta}_{t}, \nabla L_S(\mathbf{w}_{t}) \rangle 
    =-\eta_t\langle \frac{c\mathbf{g}_t}{\Vert\mathbf{g}_t \Vert_2}s^{+}_t + \mathbf{g}_ts^{-}_t, \nabla L_S(\mathbf{w}_{t}) \rangle -\eta_t\langle \mathbf{\zeta}_{t}, \nabla L_S(\mathbf{w}_{t}) \rangle .
\end{align}

For term $-\sum^T_{t=1}\eta_t\langle \mathbf{g}_ts^{-}_t, \nabla L_S(\mathbf{w}_{t}) \rangle$, we obtain
\begin{align}
    -\sum^T_{t=1}\eta_t\langle \mathbf{g}_ts^{-}_t, \nabla L_S(\mathbf{w}_{t}) \rangle &= -\sum^T_{t=1}\eta_t s^{-}_t (\langle \mathbf{g}_t - \nabla L_S(\mathbf{w}_{t}), \nabla L_S(\mathbf{w}_{t}) \rangle + \Vert \nabla L_S(\mathbf{w}_{t}) \Vert^2_2) \nonumber \\
    &\leq -\sum^T_{t=1}\eta_t s^{-}_t \langle \mathbf{g}_t - \nabla L_S(\mathbf{w}_{t}), \nabla L_S(\mathbf{w}_{t}) \rangle - \sum^T_{t=1}\eta_t s^{-}_t\Vert \nabla L_S(\mathbf{w}_{t}) \Vert^2_2 \nonumber \\
    &\leq -\sum^T_{t=1}\eta_t s^{-}_t \langle \mathbf{g}_t - \nabla L_S(\mathbf{w}_{t}), \nabla L_S(\mathbf{w}_{t}) \rangle - \frac{c}{2}\sum^T_{t=1}\eta_t s^{-}_t\Vert \nabla L_S(\mathbf{w}_{t}) \Vert^2_2 \nonumber \\
    &\leq -\underbrace{\sum^T_{t=1}\eta_t s^{-}_t \langle \mathbf{g}_t - \nabla L_S(\mathbf{w}_{t}), \nabla L_S(\mathbf{w}_{t}) \rangle}_{\mathrm{E}.10} - \frac{c}{3}\sum^T_{t=1}\eta_t s^{-}_t\Vert \nabla L_S(\mathbf{w}_{t}) \Vert^2_2. \nonumber \\
\end{align}
Let consider the term E.10. Since $\mathbb{E}_t[\eta_t s^{-}_t \langle \mathbf{g}_t - \nabla L_S(\mathbf{w}_{t}), \nabla L_S(\mathbf{w}_{t}) \rangle]=0$, the sequence $(-\eta_t s^{-}_t \langle \mathbf{g}_t - \nabla L_S(\mathbf{w}_{t}), \nabla L_S(\mathbf{w}_{t}) \rangle, t\in \mathbb{N})$ is a martingale difference sequence. In addition, the term $\mathbf{g}_t - \nabla L_S(\mathbf{w}_{t})$ is a $subW(\theta,K)$ random variable, thus we apply sub-Weibull Freedman inequality with Lemma~\ref{lemma:3} and concentration inequality with Lemma~\ref{lemma:7} and~\ref{lemma:8} to get Corollary D.2~\label{cor:D-2} below:

From Lemma~\ref{lemma:3}, we define 
\begin{align}
    v(L,\mu) := \mathbb{E}\big[(X^L-\mathbb{E}[X])^2\mathbb{I}(X^L\leq\mathbb{E}[X])\big] + \mathbb{E}\big[(X^L-\mathbb{E}[X])^2\exp{\big(\mu(X^L-\mathbb{E}[X])\big)}\mathbb{I}(X^L>\mathbb{E}[X])\big], \nonumber
\end{align}
and make $\beta = kv(L,\mu)$, then we have $\sup_{\eta\in(0,1]}\{kv(L,\mu)\} = a\sum^k_{i=1}K^2_i$ based on Lemma~\ref{lemma:7}~ and~\ref{lemma:8} in~\cite{bakhshizadeh2023sharp} and obtain 
\begin{align}
    \mathbb{P}\left( \bigcup_{k\in[n]}\Big\{\sum^k_{i=1}\xi_i\geq kx~\mathrm{and}~\sum^k_{i=1}aK^2_{i-1}\leq \beta \Big\} \right) &\leq \mathrm{exp}(-\lambda kx + \frac{\lambda^2}{2}\beta) \nonumber \\
    &= \mathrm{exp}(-\lambda kx + kv(L,\mu)\frac{\lambda^2}{2}).
\end{align}
Subsequently, we define $x_{\mathrm{max}} :=  \frac{\mu I(kx)}{kx}a\sum^k_{i=1}K^2_i $ and have
\begin{enumerate}
    \item If $x\geq x_{\mathrm{max}}$, we choose $L=kx$ and $\lambda = \frac{\mu I(kx)}{kx}$, that is $\frac{x}{v(kx,\mu)} \geq \frac{x_{\mathrm{max}}}{v(kx,\mu)} = \frac{\mu I(kx)}{kx}$. Then the inequality achieves
    \begin{align}
        \mathbb{P}\left( \bigcup_{k\in[n]}\Big\{\sum^k_{i=1}\xi_i\geq kx~\mathrm{and}~\sum^k_{i=1}aK^2_{i-1}\leq \beta \Big\} \right) &\leq \mathrm{exp}(-\mu I(kx) + v(L,\mu)\frac{\mu^2I^2(kx)}{2kx^2}) \nonumber \\
        &\leq  \mathrm{exp}(-\mu I(kx)(1-v(L,\mu)\frac{\mu I(kx)}{2kx^2})) \nonumber \\
        &\leq \mathrm{exp}(-\mu c_xI(kx))\nonumber \\
        &\leq \mathrm{exp}(-\frac{1}{2}\mu I(kx)),
    \end{align}
    where $c_x = 1-\frac{\mu v(kx,\mu)I(kx)}{2kx^2}$ and the last inequality holds due to $c_x \geq \frac{1}{2}$.
    \item If $x\leq x_{\mathrm{max}}$, we choose $L=kx_{\mathrm{max}}$ and $\lambda = \frac{x}{v(L,\mu)} \leq \frac{x_{\mathrm{max}}}{v(L,\mu)}= \frac{\mu I(L)}{L}$. Then, we get
    \begin{align}
        \mathbb{P}\left( \bigcup_{k\in[n]}\Big\{\sum^k_{i=1}\xi_i\geq kx~\mathrm{and}~\sum^k_{i=1}aK^2_{i-1}\leq \beta \Big\} \right) &\leq  \mathrm{exp}(-\frac{kx^2}{v(L,\mu)} + \frac{kx^2}{2v(L,\mu)} ) \nonumber \\
        &\leq  \mathrm{exp}( - \frac{kx^2}{2v(L,\mu)} ) . 
    \end{align}
\end{enumerate}

Implementing the above inferences and propositions with
\begin{align}
    &\xi_t = \eta_t \langle \mathbf{g}_t - \nabla L_S(\mathbf{w}_{t}), \nabla L_S(\mathbf{w}_{t}) \rangle \nonumber \\
    &\Lambda :=  -\sum^{T}_{i=1}\eta_t s^{-}_t \langle \mathbf{g}_t - \nabla L_S(\mathbf{w}_{t}), \nabla L_S(\mathbf{w}_{t}) \rangle \nonumber \\
    &K_{t-1} = \eta_t K \Vert\nabla L_S(\mathbf{w}_{t}) \Vert_2 \nonumber \\
    &m_t = \eta_t K G \nonumber \\
    &k = T \nonumber \\
    &\mu = 1/2 \nonumber 
\end{align}
If $\theta = \frac{1}{2}$, $\forall \alpha >0$ and $a=2$, when $x \leq x_{\mathrm{max}}$ we have the following inequality with probability at least $1-\delta$
\begin{align}
  -\sum^T_{t=1}\eta_t s^{-}_t \langle \mathbf{g}_t - \nabla L_S(\mathbf{w}_{t}), \nabla L_S(\mathbf{w}_{t}) \rangle  &\leq \sqrt{2Tv(L,\mu)}\log^{\frac{1}{2}}(1/\delta) \nonumber \\
  &\leq \sqrt{2a\sum^T_{t=1}K^2_t}\log^{\frac{1}{2}}(1/\delta) \nonumber \\
   &\leq 2\sqrt{\sum^T_{t=1}\eta^2_tK^2\Vert\nabla L_S(\mathbf{w}_{t}) \Vert^2_2}\log^{\frac{1}{2}}(1/\delta) \nonumber \\
   &\leq 2KG\sqrt{\sum^T_{t=1}\eta^2_t}\log^{\frac{1}{2}}(1/\delta) ,
\end{align}
when $x\geq x_{\mathrm{max}}$, with $I(kx)=(kx/\sum^k_{i=1}K_i)^2$, we have 
\begin{align}
   -\sum^T_{t=1}\eta_t s^{-}_t \langle \mathbf{g}_t - \nabla L_S(\mathbf{w}_{t}), \nabla L_S(\mathbf{w}_{t}) \rangle  &\leq  4^{\frac{1}{2}}\frac{1}{T}\sum^T_{t=1}K_t\log^{\frac{1}{2}}(1/\delta) \nonumber \\
   &\leq 2\frac{KG}{T}\sum^T_{t=1}\eta_t\log^{\frac{1}{2}}(1/\delta) . 
\end{align}

If $\theta \in (\frac{1}{2},1]$, let $a=(4\theta)^{2\theta}e^2$, when $x \leq x_{\mathrm{max}}$ we have the following inequality with probability at least $1-\delta$
\begin{align}
  -\sum^T_{t=1}\eta_t s^{-}_t \langle \mathbf{g}_t - \nabla L_S(\mathbf{w}_{t}), \nabla L_S(\mathbf{w}_{t}) \rangle  &\leq \sqrt{2a\sum^T_{t=1}K^2_t}\log^{\frac{1}{2}}(1/\delta) \nonumber \\
  &\leq \sqrt{2}(4\theta)^{\theta}eKG\sqrt{\sum^T_{t=1}\eta^2_t}\log^{\frac{1}{2}}(1/\delta), 
\end{align}
when $x \geq x_{\mathrm{max}}$, let $I(kx)=(kx/\sum^k_{i=1}K_i)^{\frac{1}{\theta}},~\forall \theta \in (\frac{1}{2},1]$, then we have
\begin{align}
  -\sum^T_{t=1}\eta_t s^{-}_t \langle \mathbf{g}_t - \nabla L_S(\mathbf{w}_{t}), \nabla L_S(\mathbf{w}_{t}) \rangle  &\leq  \frac{4^{\theta}}{T}\sum^T_{t=1}K_t\log^{\frac{1}{2}}(1/\delta) \nonumber \\
   &\leq \frac{4^{\theta}KG}{T}\sum^T_{t=1}\eta_t\log^{\theta}(1/\delta) .
\end{align}

If $\theta > 1$, let $a=(2^{2\theta+1}+2)\Gamma(2\theta+1) + \frac{2^{3\theta}\Gamma(3\theta+1)}{3}$, when $x \leq x_{\mathrm{max}}$ we have the following inequality with probability at least $1-3\delta$
\begin{align}
  -&\sum^T_{t=1}\eta_t s^{-}_t \langle \mathbf{g}_t - \nabla L_S(\mathbf{w}_{t}), \nabla L_S(\mathbf{w}_{t}) \rangle  \leq \sqrt{2a\sum^T_{t=1}K^2_t}\log^{\frac{1}{2}}(1/\delta) \nonumber \\
  &\leq \sqrt{2(2^{2\theta+1}+2)\Gamma(2\theta+1) + \frac{2^{3\theta}\Gamma(3\theta+1)}{3}}KG\sqrt{\sum^T_{t=1}\eta^2_t}\log^{\frac{1}{2}}(1/\delta), 
\end{align}
when $x \geq x_{\mathrm{max}}$, let $I(kx)=(kx/\sum^k_{i=1}K_i)^{\frac{1}{\theta}},~\forall \theta > 1$, then we have
\begin{align}
  -\sum^T_{t=1}\eta_t s^{-}_t \langle \mathbf{g}_t - \nabla L_S(\mathbf{w}_{t}), \nabla L_S(\mathbf{w}_{t}) \rangle  &\leq  \frac{4^{\theta}}{T}\sum^T_{t=1}K_t\log^{\frac{1}{2}}(1/\delta) \nonumber \\
   &\leq \frac{4^{\theta}KG}{T}\sum^T_{t=1}\eta_t\log^{\theta}(1/\delta).
\end{align}

To continue the proof, employing Lemma~\ref{lemma:5} in term $-\eta_t\langle \frac{c\mathbf{g}_t}{\Vert\mathbf{g}_t \Vert_2}s^{+}_t, \nabla L_S(\mathbf{w}_{t}) \rangle$ and covering all $T$ iterations, we have
\begin{align}
    -\sum^T_{t=1}\eta_t\langle \frac{c\mathbf{g}_t}{\Vert\mathbf{g}_t \Vert_2}s^{+}_t, \nabla L_S(\mathbf{w}_{t}) \rangle &\leq - \frac{c\sum^T_{t=1}\eta_t s^{+}_t\Vert\nabla L_S(\mathbf{w}_t)\Vert_2}{3} + \frac{8c\sum^T_{t=1}\eta_t\Vert \mathbf{g}_t-\nabla L_S(\mathbf{w}_t)\Vert_2}{3} \nonumber \\
     &\leq - \frac{c\sum^T_{t=1}\eta_t(1-s^{-}_t)\Vert\nabla L_S(\mathbf{w}_t)\Vert_2}{3}  \nonumber \\
     &+ \frac{16\sum^T_{t=1}\eta_t\Vert \mathbf{g}_t-\nabla L_S(\mathbf{w}_t)\Vert_2\Vert\nabla L_S(\mathbf{w}_t) \Vert_2}{3} .
\end{align}
With the truncated Corollary~\ref{cor:D-1}.1, we have
\begin{enumerate}[leftmargin=*]
\setlength{\itemindent}{0pt}
    \item If $0 \leq t \leq t_{\mathrm{max}}$, with probability at least $1-3\delta$
    \begin{align}
         &-\sum^T_{t=1}\eta_t\langle \frac{c\mathbf{g}_t}{\Vert\mathbf{g}_t \Vert_2}s^{+}_t, \nabla L_S(\mathbf{w}_{t}) \rangle \leq - \frac{c\sum^T_{t=1}\eta_t(1-s^{-}_t)\Vert\nabla L_S(\mathbf{w}_t)\Vert_2}{3} \nonumber \\
         &+\frac{16\sum^T_{t=1}\eta_t\Vert\nabla L_S(\mathbf{w}_t) \Vert_2}{3} \left\{ \begin{array}{lll} 2K\log^{\frac{1}{2}}(2/\delta) , & \text{if } \theta = \frac{1}{2}, \\ \\ 
          \sqrt{2}e(4\theta)^{\theta}K\log^{\frac{1}{2}}(2/\delta), & \text{if } \theta \in (\frac{1}{2},1], \\ \\ 
         \displaystyle \sqrt{2(2^{2\theta+1}+2)\Gamma(2\theta+1) + \frac{2^{3\theta}\Gamma(3\theta+1)}{3}}K\log^{\frac{1}{2}}(2/\delta)  & \text{if } \theta > 1. \end{array} \right. . 
    \end{align}
    
 \item If $t \geq t_{\mathrm{max}}$ and $\theta \geq \frac{1}{2}$, with probability at least $1-3\delta$ 
    \begin{align}
         &-\sum^T_{t=1}\eta_t\langle \frac{c\mathbf{g}_t}{\Vert\mathbf{g}_t \Vert_2}s^{+}_t, \nabla L_S(\mathbf{w}_{t}) \rangle \leq - \frac{c\sum^T_{t=1}\eta_t(1-s^{-}_t)\Vert\nabla L_S(\mathbf{w}_t)\Vert_2}{3} \nonumber \\
         &+\frac{16\sum^T_{t=1}\eta_t\Vert\nabla L_S(\mathbf{w}_t) \Vert_2}{3} 4^{\theta}K\log^{\theta}(2/\delta). 
    \end{align}
\end{enumerate}

To simplify the proof, we unify the notation with $t_{\mathrm{max}} = x_{\mathrm{max}}$. Then, according to Lemma~\ref{lemma:1}, Corollary~\ref{cor:D-1}.1 and Corollary~\ref{cor:D-2}.2, combining the truncated results of $-\sum^T_{t=1}\eta_t\langle \mathbf{g}_ts^{-}_t, \nabla L_S(\mathbf{w}_{t}) \rangle $ and $ -\sum^T_{t=1}\eta_t\langle \frac{c\mathbf{g}_t}{\Vert\mathbf{g}_t \Vert_2}s^{+}_t, \nabla L_S(\mathbf{w}_{t}) \rangle$, we have the inequality:

\begin{enumerate}[leftmargin=*]
\setlength{\itemindent}{0pt}
    \item If $0 \leq x \leq x_{\mathrm{max}}$,  with probability at least $1 -3\delta- T\delta$ 
    \begin{align}
         &-\sum^T_{t=1}\eta_t\langle\overline{\mathbf{g}}_t, \nabla L_S(\mathbf{w}_{t}) \rangle \leq - \frac{c\sum^T_{t=1}\eta_t\Vert\nabla L_S(\mathbf{w}_t)\Vert_2}{3}  \nonumber \\
         &+ \left\{ \begin{array}{lll} 2KG\sqrt{\sum^T_{t=1}\eta^2_t}\log^{\frac{1}{2}}(1/\delta) , & \text{if } \theta = \frac{1}{2}, \\ \\ 
          \sqrt{2}(4\theta)^{\theta}eKG\sqrt{\sum^T_{t=1}\eta^2_t}\log^{\frac{1}{2}}(1/\delta), & \text{if } \theta \in (\frac{1}{2},1], \\ \\ 
         \displaystyle \sqrt{2(2^{2\theta+1}+2)\Gamma(2\theta+1) + \frac{2^{3\theta}\Gamma(3\theta+1)}{3}}KG\sqrt{\sum^T_{t=1}\eta^2_t}\log^{\frac{1}{2}}(1/\delta)  & \text{if } \theta > 1. \end{array} \right. \nonumber \\
         &+\frac{16\sum^T_{t=1}\eta_t\Vert\nabla L_S(\mathbf{w}_t) \Vert_2}{3} \left\{ \begin{array}{lll} 2K\log^{\frac{1}{2}}(2/\delta) , & \text{if } \theta = \frac{1}{2}, \\ \\ 
          \sqrt{2}e(4\theta)^{\theta}K\log^{\frac{1}{2}}(2/\delta), & \text{if } \theta \in (\frac{1}{2},1], \\ \\ 
         \displaystyle \sqrt{2(2^{2\theta+1}+2)\Gamma(2\theta+1) + \frac{2^{3\theta}\Gamma(3\theta+1)}{3}}K\log^{\frac{1}{2}}(2/\delta)  & \text{if } \theta > 1. \end{array} \right.
    \end{align}
    
 \item If $x \geq x_{\mathrm{max}}$ and $\theta \geq \frac{1}{2}$, with probability at least $1-3\delta-T\delta$
    \begin{align}
         &-\sum^T_{t=1}\eta_t\langle\overline{\mathbf{g}}_t, \nabla L_S(\mathbf{w}_{t}) \rangle \leq - \frac{c\sum^T_{t=1}\eta_t\Vert\nabla L_S(\mathbf{w}_t)\Vert_2}{3} +  \frac{4^{\theta}KG}{T}\sum^T_{t=1}\eta_t\log^{\theta}(1/\delta)  \nonumber \\
         &+\frac{16\sum^T_{t=1}\eta_t\Vert\nabla L_S(\mathbf{w}_t) \Vert_2}{3} 4^{\theta}K\log^{\theta}(2/\delta). 
    \end{align}
\end{enumerate}

So, integrating the results of formula.(84) and formula.(85) into formula.(69),  and applying Lemma~\ref{lemma:2} to $\sum^T_{t=1}\eta_t\langle \mathbf{\zeta}_{t}, \nabla L_S(\mathbf{w}_{t}) \rangle$ and $\frac{1}{2}\beta\sum^T_{t=1}\eta^2_t\Vert \overline{\mathbf{g}}_t + \mathbf{\zeta}_{t} \Vert^2_2$ because of $\zeta_t\sim \mathbb{N}(0,c^2\sigma^2_{\mathrm{dp}}\mathbb{I}_d)$, with$\Vert\overline{\mathbf{g}}_t\Vert_2 \leq c$ ,$\sigma^2_{\mathrm{dp}} = m_2\frac{TdB^2\log(1/\delta)}{n^2\epsilon^2}$ and probability $1 - 6\delta - T\delta$, if $0 \leq x\leq x_{\mathrm{max}}$, we have
\begin{align}
    &(\frac{c}{3} - \frac{16}{3}\sqrt{2a}K\log^{\frac{1}{2}}(2/\delta)-4\sqrt{e}\sigma_{\mathrm{dp}}\log^{\frac{1}{2}}(2/\delta)) \sum^T_{t=1}\eta_t\Vert\nabla L_S(\mathbf{w}_{t})\Vert_2 \leq L_S(\mathbf{w}_{1}) - L_S(\mathbf{w}_{S}) \nonumber \\
    & + (\frac{2\beta m_2 e dTc^2B^2\log^2(2/\delta)}{n^2\epsilon^2} + \frac{2\beta \sqrt{em_2dT}c^2B\log(2/\delta)}{n\epsilon} + \frac{1}{2}\beta c^2)\sum^T_{t=1}\eta^2_{t} \nonumber \\
    &+\sqrt{2a}KG\sqrt{\sum^T_{t=1}\eta^2_t}\log^{\frac{1}{2}}(1/\delta),
\end{align}
if $x\geq x_{\mathrm{max}}$, we have
\begin{align}
    &(\frac{c}{3} - \frac{16}{3}\sqrt{2a}K\log^{\theta}(2/\delta)-4\sqrt{e}\sigma_{\mathrm{dp}}\log^{\frac{1}{2}}(2/\delta)) \sum^T_{t=1}\eta_t\Vert\nabla L_S(\mathbf{w}_{t})\Vert_2 \leq L_S(\mathbf{w}_{1}) - L_S(\mathbf{w}_{S}) \nonumber \\
    & + (2\beta m_2 e d\frac{Tc^2B^2\log^2(2/\delta)}{n^2\epsilon^2} + 2\beta \sqrt{em_2Td}\frac{c^2B\log(2/\delta)}{n\epsilon} + \frac{1}{2}\beta c^2)\sum^T_{t=1}\eta^2_{t} \nonumber\\
    &+\sqrt{2a}KG\sqrt{\sum^T_{t=1}\eta^2_t}\log^{\theta}(1/\delta),
\end{align}
when $0\leq x \leq x_{\mathrm{max}}$, $a = 2$ if $\theta = 1/2$, $a = (4\theta)^{2\theta}e^2$ if $\theta \in (1/2,1]$ and $a = (2^{2\theta+1}+2)\Gamma(2\theta+1) + \frac{2^{3\theta}\Gamma(3\theta+1)}{3}$ if $\theta > 1$. While $x \geq x_{\mathrm{max}}$, $a = 2^{4\theta-1}$~$\forall \theta \geq \frac{1}{2}$.

Afterwards,
\begin{enumerate}[leftmargin=*]
\item In case $0 \leq x\leq x_{\mathrm{max}}$ and $\theta \geq \frac{1}{2}$:

If $16\sqrt{2a}K\log^{\frac{1}{2}}(2/\delta)\geq12\sqrt{e}\sigma_{\mathrm{dp}}\log^{\frac{1}{2}}(2/\delta)$, let $c \geq 33\sqrt{2a}K\log^{\frac{1}{2}}(2/\delta)$, $T= \frac{n\epsilon}{\sqrt{d\log(1/\delta)}}$ and $\eta_t = \frac{1}{\sqrt{T}}$, we obtain
\begin{align}
    &\sum^T_{t=1}\eta_t\Vert\nabla L_S(\mathbf{w}_{t})\Vert_2 \leq \frac{3}{\sqrt{2a}K\log^{\frac{1}{2}}(2/\delta)} \left( L_S(\mathbf{w}_{1}) - L_S(\mathbf{w}_{S}) +\sqrt{2a}KG\sqrt{\sum^T_{t=1}\eta^2_t}\log^{\frac{1}{2}}(1/\delta)  \nonumber \right. \\
    &\left. +  2\sum^T_{t=1}\eta^2_{t}\frac{\beta m_2 e dTc^2B^2\log^2(2/\delta)}{n^2\epsilon^2} + 2\sum^T_{t=1}\eta^2_{t}\frac{\beta \sqrt{em_2Td}c^2B\log(2/\delta)}{n\epsilon} + \frac{1}{2}\beta c^2\sum^T_{t=1}\eta^2_{t}  \right) \nonumber \\
    &\leq  \frac{3(L_S(\mathbf{w}_{1}) - L_S(\mathbf{w}_{S}) )}{\sqrt{2a}K\log^{\frac{1}{2}}(2/\delta)} + 3G\sqrt{\sum^T_{t=1}\eta^2_t} \nonumber \\
    &+ 3\sum^T_{t=1}\eta^2_{t}\left(\frac{32\sqrt{2a}}{9}\beta K\log^{\frac{1}{2}}(2/\delta) + 88\sqrt{2a}\beta K\log^{\frac{1}{2}}(2/\delta) +  \frac{33^2\sqrt{2a}\beta K\log^{\frac{1}{2}}(2/\delta)}{2}\right) ,
\end{align}
with $\frac{\sqrt{em_2Td}cB\log(2/\delta)}{n\epsilon} = \sqrt{e}\sigma_{\mathrm{dp}}\log^{\frac{1}{2}}(2/\delta).$

Therefore, with probability at least $1-6\delta-T\delta$, we have
\begin{align}
    \frac{1}{T}\sum^T_{t=1}\Vert\nabla L_S(\mathbf{w}_{t})\Vert_2 \leq \mathbb{O}\left(\frac{d^{\frac{1}{4}}\log^{\frac{3}{4}}(1/\delta)}{\sqrt{n\epsilon}}\right) , \nonumber
\end{align}
then, with probability $1-\delta$, we have
\begin{align}
    \frac{1}{T}\sum^T_{t=1}\Vert\nabla L_S(\mathbf{w}_{t})\Vert_2 \leq \mathbb{O}\left(\frac{d^{\frac{3}{4}}\log^{\frac{3}{4}}(T/\delta)}{\sqrt{n\epsilon}}\right) .
\end{align}

If $16\sqrt{2a}K\log^{\frac{1}{2}}(2/\delta)\leq 12\sqrt{e}\sigma_{\mathrm{dp}}\log^{\frac{1}{2}}(2/\delta)$, let $\frac{c}{3} \geq 9\sqrt{e}\sigma_{\mathrm{dp}}\log^{\frac{1}{2}}(2/\delta)$, that is, $c\geq 27\sqrt{e}\sigma_{\mathrm{dp}}\log^{\frac{1}{2}}(2/\delta)$, thus there exists $T= \max{(m_2eB^2\log(1/\delta), \frac{n\epsilon}{\sqrt{d\log(1/\delta)}})}$ and $\eta_t = \frac{1}{\sqrt{T}}$ that we obtain
\begin{align}
    &\sum^T_{t=1}\eta_t\Vert\nabla L_S(\mathbf{w}_{t})\Vert_2 \leq \frac{1}{\sqrt{e}\sigma_{\mathrm{dp}}\log^{\frac{1}{2}}(2/\delta)} \left( L_S(\mathbf{w}_{1}) - L_S(\mathbf{w}_{S}) +\sqrt{2a}KG\sqrt{\sum^T_{t=1}\eta^2_t}\log^{\frac{1}{2}}(1/\delta)  \nonumber \right. \\
    &\left. +  2\sum^T_{t=1}\eta^2_{t}\frac{\beta m_2 e dTc^2B^2\log^2(2/\delta)}{n^2\epsilon^2} + 2\sum^T_{t=1}\eta^2_{t}\frac{\beta \sqrt{em_2Td}c^2B\log(2/\delta)}{n\epsilon} + \frac{1}{2}\beta c^2\sum^T_{t=1}\eta^2_{t}  \right) \nonumber \\
    &\leq  \frac{L_S(\mathbf{w}_{1}) - L_S(\mathbf{w}_{S}) }{\sqrt{e}\sigma_{\mathrm{dp}}\log^{\frac{1}{2}}(2/\delta)} +  \frac{\sqrt{2a}KG}{\sqrt{e}\sigma_{\mathrm{dp}}}\sqrt{\sum^T_{t=1}\eta^2_t}  \nonumber \\
    &+ \sum^T_{t=1}\eta^2_{t}\left(2\sqrt{e}\beta \sigma_{\mathrm{dp}}\log^{\frac{1}{2}}(2/\delta) + 54\sqrt{e}\beta\sigma_{\mathrm{dp}}\log^{\frac{1}{2}}(2/\delta) + \frac{(27)^2\sqrt{e}}{2}  \beta\sigma_{\mathrm{dp}}\log^{\frac{1}{2}}(2/\delta)\right) .
\end{align}

Therefore, with $\sigma_{\mathrm{dp}} = \frac{\sqrt{m_2Td}B\log^{\frac{1}{2}}(1/\delta)}{n\epsilon} = \mathbb{O}(1)$ and probability $1-6\delta-T\delta$, we have
\begin{align}
    \frac{1}{T}\sum^T_{t=1}\Vert\nabla L_S(\mathbf{w}_{t})\Vert_2 \leq \mathbb{O}\left(\frac{d^{\frac{1}{4}}\log^{\frac{3}{4}}(1/\delta)}{\sqrt{n\epsilon}}\right) , \nonumber
\end{align}
then, with probability $1-\delta$, we have
\begin{align}
    \frac{1}{T}\sum^T_{t=1}\Vert\nabla L_S(\mathbf{w}_{t})\Vert_2 \leq \mathbb{O}\left(\frac{d^{\frac{1}{4}}\log^{\frac{3}{4}}(T/\delta)}{\sqrt{n\epsilon}}\right) .
\end{align}

\item In case $x\geq x_{\mathrm{max}}$:

If $\theta = \frac{1}{2}$ and $16\sqrt{2a}K\log^{\theta}(2/\delta)\geq12\sqrt{e}\sigma_{\mathrm{dp}}\log^{\frac{1}{2}}(2/\delta)$, let $c \geq 33\sqrt{2a}K\log^{\frac{1}{2}}(2/\delta)$, $T= \frac{n\epsilon}{\sqrt{d\log(1/\delta)}}$ and $\eta_t = \frac{1}{\sqrt{T}}$, we obtain
\begin{align}
    &\sum^T_{t=1}\eta_t\Vert\nabla L_S(\mathbf{w}_{t})\Vert_2 \leq \frac{3}{\sqrt{2a}K\log^{\frac{1}{2}}(2/\delta)} \left( L_S(\mathbf{w}_{1}) - L_S(\mathbf{w}_{S}) +\sqrt{2a}KG\sqrt{\sum^T_{t=1}\eta^2_t}\log^{\frac{1}{2}}(1/\delta)  \nonumber \right. \\
    &\left. +  2\sum^T_{t=1}\eta^2_{t}\frac{\beta m_2 e dTc^2B^2\log^2(2/\delta)}{n^2\epsilon^2} + 2\sum^T_{t=1}\eta^2_{t}\frac{\beta \sqrt{em_2Td}c^2B\log(2/\delta)}{n\epsilon} + \frac{1}{2}\beta c^2\sum^T_{t=1}\eta^2_{t}  \right) \nonumber \\
    &\leq  \frac{3(L_S(\mathbf{w}_{1}) - L_S(\mathbf{w}_{S}) )}{\sqrt{2a}K\log^{\frac{1}{2}}(2/\delta)} + 3G\sqrt{\sum^T_{t=1}\eta^2_t} \nonumber \\
    &+ 3\sum^T_{t=1}\eta^2_{t}\left(\frac{32\sqrt{2a}}{9}\beta K\log^{\frac{1}{2}}(2/\delta) + 88\sqrt{2a}\beta K\log^{\frac{1}{2}}(2/\delta) +  \frac{33^2\sqrt{2a}\beta K\log^{\frac{1}{2}}(2/\delta)}{2}\right) .
\end{align}
Therefore, with probability at least $1-6\delta-T\delta$, we have
\begin{align}
    \frac{1}{T}\sum^T_{t=1}\Vert\nabla L_S(\mathbf{w}_{t})\Vert_2 \leq \mathbb{O}\left(\frac{d^{\frac{1}{4}}\log^{\frac{3}{4}}(1/\delta)}{\sqrt{n\epsilon}}\right) , \nonumber
\end{align}
then, with probability $1-\delta$, we have
\begin{align}
    \frac{1}{T}\sum^T_{t=1}\Vert\nabla L_S(\mathbf{w}_{t})\Vert_2 \leq \mathbb{O}\left(\frac{d^{\frac{3}{4}}\log^{\frac{3}{4}}(T/\delta)}{\sqrt{n\epsilon}}\right) .
\end{align}

If $\theta = \frac{1}{2}$ and $16\sqrt{2a}K\log^{\theta}(2/\delta)\leq 12\sqrt{e}\sigma_{\mathrm{dp}}\log^{\frac{1}{2}}(2/\delta) $, we need $c\geq 27\sqrt{e}\sigma_{\mathrm{dp}}\log^{\frac{1}{2}}(2/\delta)$, thus there exists $T= \max{(m_2eB^2\log(1/\delta), \frac{n\epsilon}{\sqrt{d\log(1/\delta)}})}$ and $\eta_t = \frac{1}{\sqrt{T}}$ that we obtain
\begin{align}
    &\sum^T_{t=1}\eta_t\Vert\nabla L_S(\mathbf{w}_{t})\Vert_2 \leq \frac{1}{\sqrt{e}\sigma_{\mathrm{dp}}\log^{\frac{1}{2}}(2/\delta)} \left( L_S(\mathbf{w}_{1}) - L_S(\mathbf{w}_{S}) +\sqrt{2a}KG\sqrt{\sum^T_{t=1}\eta^2_t}\log^{\frac{1}{2}}(1/\delta)  \nonumber \right. \\
    &\left. +  2\sum^T_{t=1}\eta^2_{t}\frac{\beta m_2 e dTc^2B^2\log^2(2/\delta)}{n^2\epsilon^2} + 2\sum^T_{t=1}\eta^2_{t}\frac{\beta \sqrt{em_2Td}c^2B\log(2/\delta)}{n\epsilon} + \frac{1}{2}\beta c^2\sum^T_{t=1}\eta^2_{t}  \right) \nonumber \\
    &\leq  \frac{L_S(\mathbf{w}_{1}) - L_S(\mathbf{w}_{S}) }{\sqrt{e}\sigma_{\mathrm{dp}}\log^{\frac{1}{2}}(2/\delta)} +  \frac{\sqrt{2a}KG}{\sqrt{e}\sigma_{\mathrm{dp}}}  \nonumber \\
    &+ 2\sqrt{e}\beta \sigma_{\mathrm{dp}}\log^{\frac{1}{2}}(2/\delta) + 54\beta\sqrt{e}\sigma_{\mathrm{dp}}\log^{\frac{1}{2}}(2/\delta) + \frac{(27)^2\sqrt{e}\beta}{2}  \sigma_{\mathrm{dp}}\log^{\frac{1}{2}}(2/\delta) .
\end{align}

Therefore, with $\sigma_{\mathrm{dp}} = \frac{\sqrt{m_2Td}B\log^{\frac{1}{2}}(1/\delta)}{n\epsilon} = \mathbb{O}(1)$ and probability $1-6\delta-T\delta$, we have
\begin{align}
    \frac{1}{T}\sum^T_{t=1}\Vert\nabla L_S(\mathbf{w}_{t})\Vert_2 \leq \mathbb{O}\left(\frac{d^{\frac{1}{4}}\log^{\frac{3}{4}}(1/\delta)}{\sqrt{n\epsilon}}\right) , \nonumber
\end{align}
then, with probability $1-\delta$, we have
\begin{align}
    \frac{1}{T}\sum^T_{t=1}\Vert\nabla L_S(\mathbf{w}_{t})\Vert_2 \leq \mathbb{O}\left(\frac{d^{\frac{1}{4}}\log^{\frac{3}{4}}(T/\delta)}{\sqrt{n\epsilon}}\right) .
\end{align}

If $\theta > \frac{1}{2}$, then term $\log^{\theta}(2/\delta)$ dominates the inequality. Let $\frac{c}{3} \geq \frac{17}{3}\sqrt{2a}K\log^{\theta}(2/\delta)$, $T= \frac{n\epsilon}{\sqrt{d\log(1/\delta)}}$ and $\eta_t = \frac{1}{\sqrt{T}}$, we obtain
\begin{align}
     &\sum^T_{t=1}\eta_t\Vert\nabla L_S(\mathbf{w}_{t})\Vert_2 \leq \frac{3}{\sqrt{2a}K\log^{\theta}(2/\delta)} \left( L_S(\mathbf{w}_{1}) - L_S(\mathbf{w}_{S}) +\sqrt{2a}KG\sqrt{\sum^T_{t=1}\eta^2_t}\log^{\frac{1}{2}}(1/\delta)  \nonumber \right. \\
    &\left. +  2\sum^T_{t=1}\eta^2_{t}\frac{\beta m_2 e dTc^2B^2\log^2(2/\delta)}{n^2\epsilon^2} + 2\sum^T_{t=1}\eta^2_{t}\frac{\beta \sqrt{em_2Td}c^2B\log(2/\delta)}{n\epsilon} + \frac{1}{2}\beta c^2\sum^T_{t=1}\eta^2_{t}  \right) \nonumber \\
    &\leq  \frac{3(L_S(\mathbf{w}_{1}) - L_S(\mathbf{w}_{S}) )}{\sqrt{2a}K\log^{\theta}(2/\delta)} + 3G + \frac{16^2}{24}\beta K\log^{\theta}(2/\delta) + 136\beta K\log^{\theta}(2/\delta) + 3\beta(17)^2 K\log^{\theta}(2/\delta) \nonumber \\
     &\leq  \frac{3(L_S(\mathbf{w}_{1}) - L_S(\mathbf{w}_{S}) )}{\sqrt{2a}K\log^{\frac{1}{2}}(2/\delta)} + 3G\sqrt{\sum^T_{t=1}\eta^2_t} \nonumber \\
    &+ 3\sum^T_{t=1}\eta^2_{t}\left(\frac{32\sqrt{2a}}{9}\beta K\log^{\theta}(2/\delta) + \frac{136\sqrt{2a}}{3}\beta K\log^{\theta}(2/\delta) +  \frac{33^2\sqrt{2a}\beta K\log^{\theta}(2/\delta)}{2}\right) ,
\end{align}
with $16\sqrt{2a}K\log^{\theta}(2/\delta)\geq 12\sqrt{e}\sigma_{\mathrm{dp}}\log^{\frac{1}{2}}(1/\delta) $.

As a result, with probability $1-\delta$, we have
\begin{align}
    \frac{1}{T}\sum^T_{t=1}\Vert\nabla L_S(\mathbf{w}_{t})\Vert_2 \leq \mathbb{O}\left(\frac{\log^{\theta}(T/\delta)d^\frac{1}{4}\log^{\frac{1}{4}}(T/\delta)}{\sqrt{n\epsilon}}\right) .
\end{align}
\end{enumerate}

Consequently, integrate the above results on the condition that $\nabla L_S(\mathbf{w}_{t}) \geq c/2$.

In the case of $0 \leq x\leq x_{\mathrm{max}}$, we have 
\begin{align}
    \frac{1}{T}\sum^T_{t=1}\Vert\nabla L_S(\mathbf{w}_{t})\Vert_2 \leq \mathbb{O}\left(\frac{d^{\frac{1}{4}}\log^{\frac{3}{4}}(T/\delta)}{\sqrt{n\epsilon}}\right) ,
\end{align}
in the case of $x\geq x_{\mathrm{max}}$, we have 
\begin{align}
    \frac{1}{T}\sum^T_{t=1}\Vert\nabla L_S(\mathbf{w}_{t})\Vert_2 \leq \mathbb{O}\left(\frac{d^{\frac{1}{4}}\log^{\theta+\frac{1}{4}}(T/\delta)}{\sqrt{n\epsilon}}\right) ,
\end{align}
with probability $1-\delta$ and $\theta\geq \frac{1}{2}$.

In a word, covering the two cases, we ultimately come to the conclusion with probability $1-\delta$ and $\eta_t = \frac{1}{\sqrt{T}}$:

For the case $0\leq x\leq x_{\mathrm{max}}$,
\begin{align}
     \frac{1}{T}\sum^T_{t=1} \min\big\{\Vert\nabla L_S(\mathbf{w}_{t})\Vert_2, \Vert\nabla L_S(\mathbf{w}_{t})\Vert^2_2\big\} &\leq \mathbb{O}\left(\frac{d^{\frac{1}{4}}\log^{\frac{3}{4}}(T/\delta)}{(n\epsilon)^{\frac{1}{2}}}\right) + \mathbb{O}\left(\frac{d^{\frac{1}{4}}\log(\sqrt{T})\log^{\frac{5}{4}}(T/\delta) }{(n\epsilon)^\frac{1}{2}}\right) \nonumber \\
     &\leq \mathbb{O}\left(\frac{d^{\frac{1}{4}} \log^{\frac{1}{4}}(T/\delta) \big (\log^{\frac{1}{2}}(T/\delta) + \log(\sqrt{T})\log(T/\delta) \big ) }{(n\epsilon)^\frac{1}{2}}\right) \nonumber \\
      &\leq \mathbb{O}\left(\frac{d^{\frac{1}{4}} \log^{\frac{5}{4}}(T/\delta)\log(\sqrt{T}) }{(n\epsilon)^\frac{1}{2}}\right),
\end{align}
where $\hat{\log}(T/\delta) = \log^{\max(0,\theta-1)}(T/\delta)$. If $16\sqrt{2a}K\log^{\frac{1}{2}}(2/\delta)\leq 12\sqrt{e}\sigma_{\mathrm{dp}}\log^{\frac{1}{2}}(1/\delta) $, then $T= \max{\big(m_2eB^2\log(1/\delta), \frac{n\epsilon}{\sqrt{d\log(1/\delta)}}\big)}$ and $c=\max{\big(2\sqrt{2a}K\log^{\frac{1}{2}}({\sqrt{T}}),27\sqrt{e}\sigma_{\mathrm{dp}}\log^{\frac{1}{2}}(1/\delta)\big)}$. If $16\sqrt{2a}K\log^{\frac{1}{2}}(2/\delta)\geq 12\sqrt{e}\sigma_{\mathrm{dp}}\log^{\frac{1}{2}}(1/\delta) $, then $T= \frac{n\epsilon}{\sqrt{d\log(1/\delta)}}$ and $c=\max{\big(2\sqrt{2a}K\log^{\frac{1}{2}}({\sqrt{T}}),33\sqrt{2a}K\log^{\frac{1}{2}}(2/\delta)\big)}$. 

For the case $x\geq x_{\mathrm{max}}$,
\begin{align}
     \frac{1}{T}\sum^T_{t=1} \min\big\{\Vert\nabla L_S(\mathbf{w}_{t})\Vert_2, \Vert\nabla L_S(\mathbf{w}_{t})\Vert^2_2\big\} &\leq \mathbb{O}\left(\frac{d^{\frac{1}{4}}\log^{\theta+\frac{1}{4}}(T/\delta)}{(n\epsilon)^{\frac{1}{2}}}\right) + \mathbb{O}\left(\frac{d^{\frac{1}{4}}\log^{2\theta}(\sqrt{T})\log^{\frac{5}{4}}(T/\delta) }{(n\epsilon)^\frac{1}{2}}\right) \nonumber \\
     &\leq \mathbb{O}\left(\frac{d^{\frac{1}{4}} \log^{\frac{1}{4}}(T/\delta) \big (\log^{\theta}(T/\delta) + \log^{2\theta}(\sqrt{T})\log(T/\delta) \big ) }{(n\epsilon)^\frac{1}{2}}\right) \nonumber \\
     &\leq \mathbb{O}\left(\frac{d^{\frac{1}{4}} \log^{\frac{5}{4}}(T/\delta)\hat{\log}(T/\delta)\log^{2\theta}(\sqrt{T}) }{(n\epsilon)^\frac{1}{2}}\right),
\end{align}
where $\hat{\log}(T/\delta) = \log^{\max(0,\theta-1)}(T/\delta)$. If $\theta = \frac{1}{2}$ and $16\sqrt{2a}K\log^{\theta}(2/\delta)\leq 12\sqrt{e}\sigma_{\mathrm{dp}}\log^{\frac{1}{2}}(1/\delta) $, then $T= \max{\big(m_2eB^2\log(1/\delta), \frac{n\epsilon}{\sqrt{d\log(1/\delta)}}\big)}$ and $c=\max{\big(4^{\theta}2K\log^{\theta}({\sqrt{T}}),27\sqrt{e}\sigma_{\mathrm{dp}}\log^{\frac{1}{2}}(1/\delta)\big)}$. If $\theta = \frac{1}{2}$ and $16\sqrt{2a}K\log^{\theta}(2/\delta)\geq 12\sqrt{e}\sigma_{\mathrm{dp}}\log^{\frac{1}{2}}(1/\delta) $, then $T= \frac{n\epsilon}{\sqrt{d\log(1/\delta)}}$ and $c=\max{\big(4^{\theta}2K\log^{\theta}({\sqrt{T}}),33\sqrt{2a}K\log^{\frac{1}{2}}(2/\delta)\big)}$. If $\theta > \frac{1}{2}$, then  $T= \frac{n\epsilon}{\sqrt{d\log(1/\delta)}}$ and $c=\max{\big(4^{\theta}2K\log^{\theta}({\sqrt{T}}),17K\log^{\theta}(2/\delta)\big)}$.
\end{proof}
The proof is completed.
\clearpage   
\section{Uniform Bound for Discriminative Clipping DPSGD with Subspace Identification}
\begin{theorem}[\textbf{Uniform Bound for Discriminative Clipping DPSGD with Subspace Identification}]
Under Assumption A.1 and A.2, combining Theorem 2 and Theorem 3, for any $\delta^{\prime}\in(0,1)$, with probability $1-\delta^{\prime}$, we have
\begin{align}
    \frac{1}{T}\sum^T_{t=1} \min\big\{\Vert\nabla L_S(\mathbf{w}_{t})\Vert_2, \Vert\nabla L_S(\mathbf{w}_{t})\Vert^2_2\big\} &\leq p*\mathbb{O}\left(\frac{d^{\frac{1}{4}} \log^{\frac{5}{4}}(T/\delta)\hat{\log}(T/\delta)\log^{2\theta}(\sqrt{T}) }{(n\epsilon)^\frac{1}{2}}\right) \nonumber \\
    &+ (1-p)*\mathbb{O}\left(\frac{d^{\frac{1}{4}} \log^{\frac{5}{4}}(T/\delta)\log(\sqrt{T}) }{(n\epsilon)^\frac{1}{2}}\right), \nonumber
\end{align}
where $\delta^{\prime}=\delta^{\prime}_m + \delta$, $\hat{\log}(T/\delta) = \log^{\max(0,\theta-1)}(T/\delta)$ and $p$ is ratio of heavy-tailed samples.

\end{theorem}

\begin{proof}

We will combine the subspace skewing error with the theory of Discriminative Clipping DPSGD in this section to align with our algorithm outline. We have already discussed the error of traces in previous chapters and considered the condition of additional noise that satisfies DP, obtaining an upper bound on the error that depends on the factor $\mathbb{O}(\frac{1}{k})$. This conclusion means that, under the high probability guarantee of $1-\delta^{\prime}_{m}$, we can accurately identify the trace of the per sample gradient with minimal error, and classify light bodies and heavy tails based on this. 

Specifically, based on statistical characteristics, approximately 5\% -10\% of the data will fall into the tail part. Thus, we select the top-$p$ samples in the trace ranking as the tailed samples, where $p \in [0.05, 0.1]$. Furthermore, based on the relationship between trace and variance, trace $\lambda_p$ can be seen as the threshold $x_\mathrm{max}$ in truncated theories, which corresponds to the theoretical sample variance with empirical results. So, in truncated clipping DPSGD, we will accurately partition the sample into the heavy-tailed convergence bound with a high probability of $(1-\delta^{\prime}_{m}) * p$, and exactly induce the sample to the bound of light bodies with a high probability of $(1-\delta^{\prime}_{m}) * (1-p)$, while there is a discrimination error with probability $\delta^{\prime}_{m}$. Accordingly, we have
\begin{align}
    &\mathcal{C}_{\mathrm{u}}(c_1,c_2) := \frac{1}{T}\sum^T_{t=1} \min\big\{\Vert\nabla L_S(\mathbf{w}_{t})\Vert_2, \Vert\nabla L_S(\mathbf{w}_{t})\Vert^2_2\big\}  \nonumber \\
    &=(1-\delta^{\prime}_{m})*p*\mathcal{C}_{\mathrm{tail}}(c_1) + (1-\delta^{\prime}_{m})*(1-p)*\mathcal{C}_{\mathrm{body}}(c_2) + \delta^{\prime}_{m}*|\mathcal{C}_{\mathrm{tail}}(c_1)-\mathcal{C}_{\mathrm{body}}(c_2)| .
\end{align}
where $\mathcal{C}_{\mathrm{tail}}(c_1)$ means the convergence bound of $\frac{1}{T}\sum^T_{t=1} \min\big\{\Vert\nabla L_S(\mathbf{w}_{t})\Vert_2, \Vert\nabla L_S(\mathbf{w}_{t})\Vert^2_2\big\}$ when $\lambda_{\mathrm{tr}} \geq \lambda_{p}$, i.e. $\mathbb{O}\left(\frac{d^{\frac{1}{4}} \log^{\frac{5}{4}}(T/\delta)\hat{\log}(1/\delta)\log^{2\theta}(\sqrt{T}) }{(n\epsilon)^\frac{1}{2}}\right)$, $\mathcal{C}_{\mathrm{body}}(c_2)$ denotes the bound of $\frac{1}{T}\sum^T_{t=1} \min\big\{\Vert\nabla L_S(\mathbf{w}_{t})\Vert_2, \Vert\nabla L_S(\mathbf{w}_{t})\Vert^2_2\big\}$ when $0 \leq \lambda_{\mathrm{tr}} \leq \lambda_{p}$ i.e. $\mathbb{O}\left(\frac{d^{\frac{1}{4}} \log^{\frac{5}{4}}(T/\delta)\log(\sqrt{T}) }{(n\epsilon)^\frac{1}{2}}\right)$, with $c_1=4^{\theta}2K\log^{\theta}({\sqrt{T}})$ and $c_2=2\sqrt{2a}K\log^{\frac{1}{2}}({\sqrt{T}})$.

If $\theta = \frac{1}{2}$, then $\mathcal{C}_{\mathrm{tail}}(c_1) = \mathcal{C}_{\mathrm{body}}(c_2)$ and $\delta^{\prime}_{m} \rightarrow 0$, thus we have
\begin{align}
    \mathcal{C}_{\mathrm{u}}(c_1,c_2) = \mathcal{C}_{\mathrm{tail}}(c_1) = \mathbb{O}\left(\frac{d^{\frac{1}{4}} \log^{\frac{5}{4}}(T/\delta)\log(\sqrt{T}) }{(n\epsilon)^\frac{1}{2}}\right).
\end{align}

If $\theta > \frac{1}{2}$, then $\mathcal{C}_{\mathrm{tail}}(c_1)\geq\mathcal{C}_{\mathrm{body}}(c_2)$, and we need to proof that $\mathcal{C}_{\mathrm{tail}}(c_1) \geq \mathcal{C}_{u}(c_1,c_2) $, i.e.
\begin{align}
    \mathcal{C}_{\mathrm{tail}}(c_1) &\geq \mathcal{C}_{\mathrm{u}}(c_1,c_2)  \nonumber \\
    &\geq (1-\delta^{\prime}_{m})*p*\mathcal{C}_{\mathrm{tail}}(c_1) + (1-\delta^{\prime}_{m})*(1-p)*\mathcal{C}_{\mathrm{body}}(c_2) + \delta^{\prime}_{m}*|\mathcal{C}_{\mathrm{tail}}(c_1)-\mathcal{C}_{\mathrm{body}}(c_2)| . \nonumber 
\end{align}
By transposition, we have
\begin{align}
    (1-\delta^{\prime}_{m})(1-p)*\mathcal{C}_{\mathrm{tail}}(c_1) + \delta^{\prime}_{m}*\mathcal{C}_{\mathrm{body}}(c_2)
    &\geq (1-\delta^{\prime}_{m})*(1-p)*\mathcal{C}_{\mathrm{body}}(c_2) . \nonumber
\end{align}
Then, we have
\begin{align}
    \mathcal{C}_{\mathrm{tail}}(c_1) 
    &\geq \mathcal{C}_{\mathrm{body}}(c_2) - \frac{\delta^{\prime}_{m}}{(1-\delta^{\prime}_{m})*(1-p)}\mathcal{C}_{\mathrm{body}}(c_2), 
\end{align}
due to $\frac{\delta^{\prime}_{m}}{(1-\delta^{\prime}_{m})*(1-p)} \geq 0$, it is proved that $\mathcal{C}_{\mathrm{tail}}(c_1) \geq \mathcal{C}_{u}(c_1,c_2)$.

From another perspective, for $\mathcal{C}_{\mathrm{u}}(c_1,c_2)$, with probability $1-\delta^{\prime}_m$, we have
\begin{align}
    \mathcal{C}_{\mathrm{u}}(c_1,c_2) = p*\mathcal{C}_{\mathrm{tail}}(c_1) + *(1-p)*\mathcal{C}_{\mathrm{body}}(c_2) .
\end{align}
In other words, for the formula.(102), we define $\delta^{\prime} = \delta^{\prime}_m + \delta$. Then, with probability $1-\delta^{\prime}$, we have
\begin{align}
    \frac{1}{T}\sum^T_{t=1} \min\big\{\Vert\nabla L_S(\mathbf{w}_{t})\Vert_2, \Vert\nabla L_S(\mathbf{w}_{t})\Vert^2_2\big\} &\leq p*\mathbb{O}\left(\frac{d^{\frac{1}{4}} \log^{\frac{5}{4}}(T/\delta)\hat{\log}(T/\delta)\log^{2\theta}(\sqrt{T}) }{(n\epsilon)^\frac{1}{2}}\right) \nonumber \\
    &+ (1-p)*\mathbb{O}\left(\frac{d^{\frac{1}{4}} \log^{\frac{5}{4}}(T/\delta)\log(\sqrt{T}) }{(n\epsilon)^\frac{1}{2}}\right)
\end{align}
where $\hat{\log}(T/\delta) = \log^{\max(0,\theta-1)}(T/\delta)$.

\end{proof}

The proof is completed.

\clearpage
\section{Supplemental Experiments}
\vspace{-5pt}
\subsection{Implementation Details and Codebase}
\vspace{-5pt}
All experiments are conducted on a server with an Intel(R) Xeon(R) E5-2640 v4 CPU at 2.40GHz and a NVIDIA Tesla P40 GPU running on Ubuntu. By default, we uniformly set subspace dimension $k=200$, $\epsilon=\epsilon_{\mathrm{tr}}+\epsilon_{\mathrm{dp}}$ with $\epsilon_{\mathrm{tr}} = \epsilon_{\mathrm{dp}}$, $p=10\%$ and sub-Weibull index $\theta=2$ for any datasets. In particular, we use the LDAM~\cite{cao2019learning} loss function for heavy-tailed tasks.
\begin{enumerate}
    \item \textbf{MNIST}: MNIST has ten categories, 60,000 training samples and 10.000 testing samples. We construct a two-layer CNN network and replace the BatchNorm of the convolutional layer with GroupNorm. We set 40 epochs, 128 batchsize, 0.1 small clipping threshold, 1 large clipping threshold, and 1 learning rate.
    \item \textbf{FMNIST}: FMNIST has ten categories, 60,000 training samples and 10.000 testing samples. we use the same two-layer CNN architecture, and the other hyperparameters are the same as MNIST.
    \item \textbf{CIFAR10}: CIFAR10 has 50,000 training samples and 10,000 testing. We set 50 epoch, 256 batchsize, 0.01 small clipping threshold and 0.1 large clipping threshold with model SimCLRv2~\cite{tramer2020differentially} pre-trained by unlabeled ImageNet. We refer the code for pre-trained SimCLRv2 to \url{https://github.com/ftramer/Handcrafted-DP}.
    \item \textbf{CIFAR10-HT}: CIFAR10-HT contains 32$\times$32 pixel 12,406 training data and 10,000 testing data, and the proportion of 10 classes in training data is as follows: [0:5000, 1:2997, 2:1796, 3:1077, 4:645, 5:387, 6:232, 7:139, 8:83, 9:50]. We train CIFAR10-HT on model ResNeXt-29~\cite{xie2017aggregated} pre-trained by CIFAR100 with the same parameters as CIFAR10. We can see pre-trained ResNeXt in \url{https://github.com/ftramer/Handcrafted-DP} and CIFAR10-HT with LDAM-DRW loss function in \url{https://github.com/kaidic/LDAM-DRW}.
    \item \textbf{ImageNette}: ImageNette is a 10-subclass set of ImageNet and contains 9469 training examples and 3925 testing examples. We train on model ResNet-9~\cite{he2016deep} without pre-train and set 1000 batchsize, 0.15 small clipping threshold, 1.5 large clipping threshold and 0.0001 learning rate with 50 runs. 
    \item \textbf{ImageNette-HT}: We construct the heavy-tailed version of ImageNette by the method in~\cite{cao2019learning}. ImageNette-HT contains 2345 trainging data and 3925 testing data, which is difficult to train, and proportion of 10 classes in training data follows: [0:946, 1:567, 2:340, 3:204, 4:122, 5:73, 6:43, 7:26, 8:15, 9:9]. The other settings are the same as ImageNette. Our ResNet-9 refers to \url{https://github.com/cbenitez81/Resnet9/} with 2.5M network parameters.
\end{enumerate}
Moreover, we open our source code and implementation details for discriminative clipping on the following link: \url{https://anonymous.4open.science/r/DC-DPSGD-N-25C9/}.
\vspace{-5pt}
\subsection{Effects of Parameters on Test Accuracy}
\vspace{-5pt}
Due to space limitations, we place the remaining ablation study on MNIST, FMNIST, ImageNette and ImageNette-HT here. We acknowledge that since ImageNette-HT has only 2,345 training data, which is one-fifth of ImageNette, it is difficult to support the convergence of the model. In the future, we will improve this aspect in our work.
\begin{table}[htbp]
\vspace{-10pt}
\caption{Effects of parameters on test accuracy with MNIST and FMNIST.}
\label{Table:ablation_study_0}
\resizebox{\textwidth}{!}{
\begin{tabular}{c|cccc|ccc|ccc}
\toprule
\multirow{2}{*}{Dataset} & \multicolumn{4}{c|}{Subspace-$k$}                                                               & \multicolumn{3}{c|}{$\epsilon_{\mathrm{tr}}$ + $\epsilon_{\mathrm{dp}}$}                                  & \multicolumn{3}{c}{sub-Weibull-$\theta$}                      \\ \cline{2-11} 
                         & \multicolumn{1}{c|}{None} & \multicolumn{1}{c|}{100} & \multicolumn{1}{c|}{150} & 200 & \multicolumn{1}{c|}{2+6} & \multicolumn{1}{c|}{4+4} & \multicolumn{1}{c|}{6+2} & \multicolumn{1}{c|}{1/2} & \multicolumn{1}{c|}{1} & 2 \\ \midrule
MNIST                  & \multicolumn{1}{c|}{98.16}            & \multicolumn{1}{c|}{98.48}    & \multicolumn{1}{c|}{98.66}    &98.72    & \multicolumn{1}{c|}{98.78}  & \multicolumn{1}{c|}{98.72}  & \multicolumn{1}{c|}{98.42}     & \multicolumn{1}{c|}{98.61}    & \multicolumn{1}{c|}{98.69}  &98.72  \\ \hline
FMNIST               & \multicolumn{1}{c|}{85.78}            & \multicolumn{1}{c|}{87.61}    & \multicolumn{1}{c|}{87.71}    &87.80    & \multicolumn{1}{c|}{87.70}  & \multicolumn{1}{c|}{87.80}  & \multicolumn{1}{c|}{87.26}      & \multicolumn{1}{c|}{87.40}    & \multicolumn{1}{c|}{87.55}  &87.80 \\ \bottomrule
\end{tabular}
}
\end{table}

\begin{table}[htbp]
\vspace{-20pt}
\caption{Effects of parameters on test accuracy with ImageNette and ImageNette-HT.}
\label{Table:ablation_study_1}
\resizebox{\textwidth}{!}{
\begin{tabular}{c|cccc|ccc|ccc}
\toprule
\multirow{2}{*}{Dataset} & \multicolumn{4}{c|}{Subspace-$k$}                                                               & \multicolumn{3}{c|}{$\epsilon_{\mathrm{tr}}$ + $\epsilon_{\mathrm{dp}}$}                                  & \multicolumn{3}{c}{sub-Weibull-$\theta$}                      \\ \cline{2-11} 
                         & \multicolumn{1}{c|}{None} & \multicolumn{1}{c|}{100} & \multicolumn{1}{c|}{150} & 200 & \multicolumn{1}{c|}{2+6} & \multicolumn{1}{c|}{4+4} & \multicolumn{1}{c|}{6+2} & \multicolumn{1}{c|}{1/2} & \multicolumn{1}{c|}{1} & 2 \\ \midrule
ImageNette                  & \multicolumn{1}{c|}{66.08}            & \multicolumn{1}{c|}{68.34}    & \multicolumn{1}{c|}{69.00}    &69.29    & \multicolumn{1}{c|}{68.54}  & \multicolumn{1}{c|}{69.29}  & \multicolumn{1}{c|}{68.12}     & \multicolumn{1}{c|}{67.91}    & \multicolumn{1}{c|}{68.87}  &69.29  \\ \hline
ImageNette-HT               & \multicolumn{1}{c|}{29.33}            & \multicolumn{1}{c|}{31.44}    & \multicolumn{1}{c|}{33.17}    &33.70    & \multicolumn{1}{c|}{34.25}  & \multicolumn{1}{c|}{33.70}  & \multicolumn{1}{c|}{31.13}      & \multicolumn{1}{c|}{33.05}    & \multicolumn{1}{c|}{33.37}  &33.70 \\ \bottomrule
\end{tabular}
}
\end{table}

\end{document}